\documentclass[accepted]{uai2026} 
                        
\usepackage[american]{babel}

\usepackage[utf8]{inputenc} 
\usepackage[T1]{fontenc}    
\usepackage{hyperref}       
\usepackage{url}            
\usepackage{booktabs}       
\usepackage{amsfonts}       
\usepackage{nicefrac}       
\usepackage{microtype}      
\usepackage{xcolor}         
\usepackage{amsmath,amssymb,amsthm}
\usepackage{cleveref}
\usepackage{tikz,pgfplots}
\usetikzlibrary{plotmarks,spy}
\pgfplotsset{compat=newest}
\usepackage{graphics, epsfig, color}
\usepgfplotslibrary{groupplots}
\usepackage{libertinust1math}
\usepackage{subcaption}
\usepackage[most]{tcolorbox}
\usepackage{enumitem}
\usepackage{natbib}
\tcbset{
    boxsep=0mm,
    boxrule=0pt,
    colframe=white,
    arc=0mm,
    left=0.5mm,
    right=0.5mm
}

\DeclareMathOperator{\tr}{tr}

\DeclareMathOperator{\diag}{diag}

\newcommand{\RR}{{\mathbb{R}}}
\newcommand{\CC}{{\mathbb{C}}}
\newcommand{\EE}{{\mathbb{E}}}
\newcommand{\NN}{{\mathcal{N}}}

\newcommand{\A}{\mathbf{A}}
\newcommand{\B}{\mathbf{B}}
\newcommand{\C}{\mathbf{C}}

\newcommand{\K}{\mathbf{K}}

\newcommand{\M}{\mathbf{M}}
\newcommand{\V}{\mathbf{V}}
\newcommand{\U}{\mathbf{U}}
\newcommand{\W}{\mathbf{W}}
\newcommand{\X}{\mathbf{X}}
\newcommand{\Z}{\mathbf{Z}}
\newcommand{\Q}{\mathbf{Q}}

\newcommand{\T}{\mathbf{T}}

\newcommand{\x}{\mathbf{x}}
\newcommand{\y}{\mathbf{y}}

\newcommand{\z}{\mathbf{z}}
\newcommand{\ee}{\mathbf{e}}

\newcommand{\w}{\mathbf{w}}
\newcommand{\bu}{\mathbf{u}}
\newcommand{\bv}{\mathbf{v}}
\newcommand{\zo}{\mathbf{0}}
\newcommand{\one}{\mathbf{1}}
\newcommand{\I}{\mathbf{I}}

\newcommand{\balpha}{\boldsymbol{\alpha}}
\newcommand{\bmu}{\boldsymbol{\mu}}

\newcommand{\bSigma}{\boldsymbol{\Sigma}}
\newcommand{\bLambda}{\boldsymbol{\Lambda}}
\newcommand{\bDelta}{\boldsymbol{\Delta}}

\definecolor{RED}{rgb}{0.7,0,0}
\definecolor{BLUE}{rgb}{0,0,0.69}
\definecolor{GREEN}{rgb}{0,0.6,0}
\definecolor{PURPLE}{rgb}{0.69,0,0.8}
\definecolor{ORANGE}{RGB}{255,103,0}
\definecolor{BROWN}{RGB}{100,20,45}

\newcommand{\RED}{\color[rgb]{0.70,0,0}}
\newcommand{\BLUE}{\color[rgb]{0,0,0.69}}
\newcommand{\GREEN}{\color[rgb]{0,0.6,0}}
\newcommand{\PURPLE}{\color[rgb]{0.69,0,0.8}}

\newtheorem{Definition}{Definition}
\newtheorem{Assumption}{Assumption}
\newtheorem{Theorem}{Theorem}

\newtheorem{Proposition}{Proposition}
\newtheorem{Lemma}{Lemma}
\newtheorem{Remark}{Remark}

\title{On the Interpolation Error of Nonlinear Attention versus Linear Regression}

\author[1]{Zhenyu Liao$^*$}
\author[1]{Jiaqing Liu$^*$}
\author[2]{Tianqi Hou}
\author[3]{Difan Zou}
\author[1]{Zenan Ling$^\dagger$}
\affil[1]{EIC, Huazhong University of Science and Technology}
\affil[2]{Huawei}
\affil[3]{The University of Hong Kong}

\begin{document}
\maketitle

\begin{abstract}
Attention has become the core building block of modern machine learning (ML) by efficiently capturing the long-range dependencies among input tokens.
Its inherently parallelizable structure allows for efficient performance scaling with the rapidly increasing size of both data and model parameters. 
Despite its central role, the theoretical understanding of Attention, especially in the nonlinear setting, is progressing at a more modest pace.

This paper provides a precise characterization of the \emph{interpolation error} for a \emph{nonlinear Attention}, in the high-dimensional regime where the number of input tokens $n$ and the embedding dimension $p$ are both large and comparable. 
Under a signal-plus-noise data model and for fixed Attention weights, we derive explicit (limiting) expressions for the mean‑squared interpolation error.
Leveraging recent advances in random matrix theory, we show that nonlinear Attention generally incurs a larger interpolation error than linear regression on random inputs. 
However, this gap vanishes, and can even be reversed, when the input contains a structured signal, particularly if the Attention weights align with the signal direction.
Our theoretical insights are supported by numerical experiments.
\end{abstract}

\section{Introduction}
\label{sec:Introduction}

\def\thefootnote{$*$}\footnotetext{Equal contribution.}
\def\thefootnote{$\dagger$}\footnotetext{Author to whom any correspondence should be addressed.}
\def\thefootnote{\arabic{footnote}}

Since its introduction, the Transformer architecture has become a cornerstone of modern machine learning (ML) and artificial intelligence (AI)~\citep{vaswani2017Attentiona}, powering large language models (LLMs) such as BERT~\citep{devlin2019BERT}, LLaMA~\citep{touvron2023LLaMA}, and the GPT series~\citep{openai2024GPT4}. 
Originally developed for sequence modeling tasks such as machine translation and language modeling, Transformers have demonstrated remarkable versatility and now achieve state-of-the-art performance across a wide range of applications, including those that are not inherently sequential~\citep{dosovitskiy2020Image}. 
At the heart of this empirical success lies the Attention mechanism, which enables flexible integration of information across positions and scales efficiently with both data and model size.
Despite its success, our theoretical understanding of Attention, especially in the nonlinear setting, remains limited, particularly in terms of how it learns statistical patterns from high-dimensional input tokens.

Recent years have seen an increasing use of high-dimensional statistics~\citep{vershynin2018high}, statistical physics~\citep{carleo2019Machine}, and random matrix theory (RMT)~\citep{couillet2022RMT4ML} to derive insights in the design and optimization of large-scale ML models. 
In contrast to worst-case generalization bounds that can sometimes be loose, high-dimensional analysis offers precise characterizations and is able to explain phenomena such as the neural tangent kernel~\citep{jacot2018neural}, double descent in generalization~\citep{mei2021generalization,liao2021random,nakkiran2020deep,hastie2022Surprises}, and benign overfitting~\citep{bartlett2020benign,bartlett2021Deep}, which now inform core ML design principles.
More recently, theoretical studies have investigated the connection between self-attention and max-margin–type model~\citep{tarzanagh2023transformers}, as well as convergence properties of Transformer-based models under gradient-based methods~\citep{deora2023optimization, vasudeva2024implicit}.
In \citep{cui2024phase}, the dot-product Attention model has been considered for uncorrelated inputs and low-rank Attention weights.
A brief review of related work is provided in \Cref{subsec:Related}, with a more detailed discussion in \Cref{sec:further_discussions_of_prior_efforts}.

Yet, a precise characterization of \emph{nonlinear Attention}, particularly on \emph{structural inputs}, remains largely elusive. 
The main technical challenges stem from the nonlinearity of the Attention operator and the complex interactions between input tokens and Attention weights via queries, keys, and values.
Prior theoretical efforts often rely on restrictive assumptions: focusing on in-context learning by reducing Attention to gradient descent on (generalized) linear model, which only holds under particular weight configuration~\citep{bai2023Transformers,lu2025Asymptotic}; assuming simplified Attention matrices (e.g., all-ones~\citep{noci2022Signal} or random Markov matrices~\citep{naderi2024Mind}), or drawing connections to Bayesian learning~\citep{tiberi2024Dissecting}, sequence multi-index models~\citep{troiani2025Fundamental}, or generalized Potts model~\citep{rende2024Mapping} in statistical mechanics.

In this paper, we take a different perspective by viewing the nonlinear Attention as a non-symmetric kernel matrix, and present a precise analysis of Attention \emph{interpolation error} on structural random inputs (\Cref{def:signal_plus_noise}). 
We analyze this problem under the high-dimensional regime where the input length $n$ and embedding dimension $p$ are both large and comparable.
Building upon recent advances in the eigenspectral analysis of nonlinear random kernel matrices, we derive precise expressions for the interpolation error (see \Cref{def:interpolation_error}) of \emph{nonlinear} Attention with fixed weights admitting a full-plus-low-rank decomposition (\Cref{ass:low-rank-weights}).

Our result shows that the interpolation error of nonlinear Attention is determined by a system of nonlinear equations involving the dimension ratio $p/n$, the alignment between input signal and Attention weights, and the nonlinearity (via its Hermite coefficients).
By focusing on this canonical setting, our analysis takes a step forward to unveil the theoretical origin of many visually striking features emerging in modern large-scale ML.

\subsection{Our Contribution}
\label{subsec:Contribution}

Our main novelty lies in viewing nonlinear Attention as a non-symmetric kernel matrix. 
This perspective enables a Hermite polynomial expansion to ``linearize'' the Attention matrix, which in turn allows for a through analysis of its interpolation error with random matrix techniques. 
As a byproduct, we develop a systematic framework for a novel class of random matrix models arising from Attention.
The main contribution of this paper are summarized as follows.

\begin{enumerate}[leftmargin=*]
   \item In \Cref{theo:high_interpolation}, we derive a precise characterization of the interpolation error (\Cref{def:interpolation_error}) for \emph{nonlinear} Attention, under a high-dimensional signal-plus-noise model (\Cref{def:signal_plus_noise}) for the input tokens.
   We show that the Attention interpolation error is governed by a system of nonlinear equations involving the dimension ratio $p/n$, the interaction between input signal and the Attention weights, and as well as the nonlinearity via its Hermite coefficients.
   \item In \Cref{sec:Analysis}, we compare the interpolation error of nonlinear Attention to that of linear regression (see \Cref{prop:high_interpolation_error_RR}). 
   While nonlinear Attention generally incurs higher error than linear regression for random inputs, this disadvantage disappears—and can even be reversed—for structured inputs, particularly when the Attentions weights are aligned to the input signal. 
   We further show that Attention lacking a linear component (i.e., with its first-order Hermite coefficient being zero) are \emph{unable} to fit random and/or structural inputs.
   \item From a technical perspective, we establish in \Cref{prop:DE_resovlent_noise} a novel Deterministic Equivalent (see \Cref{def:DE} for a formal definition) for the resolvent of a \emph{generalized} sample covariance matrix (SCM) of the form $\C \X \X^\top \C^\top$. 
   This extends classical SCM that has been extensively studied in the literature, by considering a population covariance $\C = \C(\X)$ function of the input $\X$. 
   This may be of independent interest beyond this paper. 
\end{enumerate}
\subsection{Related Work}
\label{subsec:Related}

Here we briefly review related work.
A more detailed discussion is provided in \Cref{sec:further_discussions_of_prior_efforts}.

\paragraph{Theoretical understanding of Transformer \& Attention.}

Theoretical studies of Transformers have sought to characterize their expressive power and in-context learning (ICL) capabilities. 
For example, it has been established that Transformers are universal sequence-to-sequence function approximators~\citep{yun2019Are}.
A growing body of work has focused on understanding the ICL behavior of Transformers and Attention, that is, their ability to adapt to new downstream tasks from a few example~\citep{dong2024Survey,xie2021Explanation,grag2022what,li2023Transformers,bai2023Transformers,oswald2023Transformers,wumany,zhang2024context,chen2024transformers,li2025robustness}. 
Recent works further provide training dynamics and generalization analyses of ICL for Attention-based Transformers~\citep{li2024nonlinear,huang2023context,he2025context,chen2024unveiling}.
Our work provides a complementary random matrix analysis of nonlinear Attention that explicitly captures the \emph{structured interaction} between structured input signals and Attention weights, offering a more flexible and data-dependent understanding of interpolation error.

\paragraph{Interpolation error in high-dimensional models.}
Interpolation refers to the regime where a model achieves zero training error. In classical statistical learning, it was associated with overfitting \citep{tikhonov1963peaks, stein1974peripheral}, consistent with the bias--variance tradeoff.
Recent high-dimensional analyses revise this view. In proportional asymptotics, interpolating estimators admit precise characterizations and may exhibit benign overfitting \citep{bartlett2020benign, hastie2022Surprises}. The double descent phenomenon shows that test error can decrease beyond the interpolation threshold \citep{belkin2019reconciling}, reflecting the implicit bias of minimum-norm and gradient-based solutions \citep{soudry2018implicit, neyshabur2017implicit}.
Interpolation is central to understanding overparameterized linear models, kernel methods, and random feature regressors \citep{fan2020spectra, mei2019generalization}. In this work, we study interpolation induced by nonlinear Attention representations.


\paragraph{Random matrix analyses of ML methods.}

Random matrix theory (RMT) has emerged as a powerful and flexible tool to understand the dynamics and generalization properties of large-scale ML models. 
It has been successfully applied to shallow~\citep{pennington2017nonlinear,liao2018spectrum,liao2018dynamics,louart2018random} and deep neural networks~\citep{benigni2019eigenvalue,fan2020spectra,pastur2020gauss}, and more recently to \emph{linear} Attention~\citep{lu2025Asymptotic}.
These analyses encompass both homogeneous (e.g., standard normal)~\citep{pennington2017nonlinear,mei2019generalization} and structured (e.g., mixture-type) input data~\citep{liao2018spectrum,ali2022random,mai2025Breakdown}.
To the best of our knowledge, the present work provides the first precise characterization of the \emph{interpolation error} of nonlinear Attention on structured input, extending RMT analysis to a broader and more realistic class of Attention-based models.


\emph{Notations.}
Scalars are denoted by lowercase letters, vectors by bold lowercase, and matrices by bold uppercase.
For a matrix $\X\in \RR^{p\times n}$, we write $\X^\top$ for its transpose, $\x_i \in \RR^p$ for $i$th column, and $\|\X\|$ for its spectral norm.
We use $\I_p$ for the identity matrix of size $p$.
For a vector $\x \in \RR^p$, its Euclidean norm is given by $\| \x \| = \sqrt{\x^\top \x}$.
For a random variable $x$, we denote its expectation by $\EE[x]$.



\section{Setup and Preliminaries}
\label{sec:Preliminaries}

In the paper, we focus on the following entry-wise Attention.

\begin{Definition}[\textbf{Entry-wise Attention}]\label{def:Attention}
Let $\X = [\x_1, \ldots, \x_n] \in \RR^{p \times n}$ be the embedding of an input sequence of tokens $\x_1, \ldots, \x_n \in \RR^p$ of length $n$. 
A (single-head) \emph{nonlinear Attention} output $\A_\X \in \RR^{ p \times n }$ with key, query, and value matrices $\W_K \in \RR^{d \times p}, \W_Q \in \RR^{d \times p}, \W_V \in \RR^{p \times p}$ and entry-wise nonlinearity $f \colon \RR \to \RR$, is defined as:
\begin{equation}\label{eq:def_A_X}
   \A_\X = \W_V \X f( \X^\top \W_K^\top \W_Q \X/\sqrt p)/\sqrt p\equiv\W_V \X \K_\X.
\end{equation}
\end{Definition}

\begin{Remark}[Softmax Attention]\label{rem:Softmax_Attention}\normalfont
Compared to the classical Softmax Attention, the entry-wise Attention in \Cref{def:Attention} is practically compelling due to its computational advantage~\citep{wortsman2023Replacing,ramapuram2024Theory}.
Remarkably, following the line of arguments in \citep{hayase2025Gaussian}, we show in \Cref{rem:softmax_Attention} in \Cref{sec:further} of the appendix that under \Cref{ass:low-rank-weights} and for input tokens drawn from the signal-plus-noise model in \Cref{def:signal_plus_noise}, taking $f$ in \eqref{eq:def_A_X} to be a truncated exponential function leads to asymptotically the same $\A_\X$ as that using standard Softmax Attention, up to a scaling factor.
\end{Remark}
Intuitively, the matrix $\K_\X \equiv f( \X^\top \W_K^\top \W_Q \X/\sqrt p)/\sqrt p \in \RR^{n \times n}$ defines an asymmetric kernel parameterized by $\W_Q, \W_K$, and captures the pairwise similarly of input tokens.
The output $\A_\X$ is then obtained by ``mixing'' the values $\W_V \X$ according to the obtained similarities in $\K_\X$.
Extension to the more widely-used multi-head Attention does not introduce any fundamentally technical difficulty in the analysis, as long as the number of heads remains constant in the large $n,p \to \infty$ limit, see \Cref{rem:extension_to_multi-head} in \Cref{sec:further} of the appendix for a detailed discussion.

We consider that the product of key and query matrices $\W_K^\top \W_Q$ in \Cref{def:Attention} writes as the sum of a full rank identity matrix and an asymmetric low-rank (in fact rank-one) matrix as follow.
\begin{Assumption}[Full-plus-low-rank decomposition of Attention weights]\label{ass:low-rank-weights}
The weight matrices $\W_K,\W_Q \in \RR^{d \times p}$ of the Attention model in \Cref{def:Attention} satisfy, for $\w_Q, \w_K \in \RR^p$,
\begin{equation}
   \W_K^\top \W_Q = \I_p + \w_K \w_Q^\top \in \RR^{p \times p}.
\end{equation}
\end{Assumption}
The full-plus-low-rank decomposition for $\W_K^\top \W_Q$ in \Cref{ass:low-rank-weights} is largely inspired by the  empirical success of Low-Rank Adaption (LoRA) in fine-tuning Transformer-based LLMs~\citep{hu2021LoRA}.  
Note that \Cref{ass:low-rank-weights} implies that $d \geq p$, though this condition is not essential and can be relaxed by considering block decomposition of $\W_K^\top \W_Q$ with one full-rank sub-block.
Also, while here we focus on the rank-one setting in \Cref{ass:low-rank-weights} for clarity, our analysis extends to arbitrary but fixed (compared to $n,p,d$) rank structure; see \Cref{rem:extension_beyond_rank_one} in \Cref{sec:further} for further discussion on this point.
We would like to emphasize here that by considering, in \Cref{ass:low-rank-weights}, the Attention weights to be \emph{deterministic}, allows us to assess scenarios where the Attention weights are aligned with, orthogonal to the data signal direction $\bmu$ (in \Cref{def:signal_plus_noise} below), or anywhere in between.
This is, in particular, formally different from most existing analyses, which typically incorporate pre-training \emph{explicitly} by considering a particular learning algorithm (e.g., gradient descent). 

For the sake of our theoretical analysis, we assume the following for the nonlinearity $f$ in \Cref{def:Attention}.
\begin{Assumption}[Nonlinear function $f$]\label{ass:nonlinear}
The function $f \colon \RR \to \RR$ in \Cref{def:Attention} satisfies: 
(1) $\lim_{t \to \infty} |f(t)| < \infty $, $|f(x)| \leq C_1 \exp(C_2 |x|)$ for some constants $C_1, C_2 > 0$; and (2) $f$ is centered with respect to standard Gaussian measure, that is, $\EE[f(\xi)] = 0$ and $a_1 \equiv \EE[\xi f(\xi)] \neq 0$, $\sqrt 2 a_2 \equiv \EE[\xi^2 f(\xi)] = 0$, and $\nu \equiv \EE[f^2(\xi)]$ for $\xi \sim \NN(0,1)$. 
\end{Assumption}

The first item of \Cref{ass:nonlinear} holds for bounded nonlinearity such as sigmoid, truncated exponential, or ReLU variants.
For the second item, note that under \Cref{ass:low-rank-weights} and for tokens $\x_i$ drawn from the signal-plus-noise model in \Cref{def:signal_plus_noise} below, it follows from the Central Limit Theorem that the non-diagonal entry of $[\X^\top \W_K^\top \W_Q \X]_{ij}/\sqrt p \to \NN(0,1)$ in law as $p \to \infty$ and for $i \neq j$, so that the nonlinear $f$ is applied on a random matrix with asymptotically Gaussian (but strongly correlated) entries, justifying the associated Hermite expansion.
We consider in \Cref{ass:nonlinear} that the zeroth-order Hermite coefficient $\EE[f(\xi)]$ of $f$ is zero:
This can be achieved by subtracting the \emph{same} constant from all non-diagonal entries of $\K_\X$ in \eqref{eq:def_A_X} and should \emph{not} alter the Attention interpolation behavior.

We consider input tokens independently drawn from the following signal-plus-noise model, in the high-dimensional regime where $n,p,d$ are all large and comparable.
\begin{Definition}[\textbf{Signal-plus-noise model}]\label{def:signal_plus_noise}
Each token-target pair $(\x_i,y_i) \in \RR^p \times \RR, i \in \{ 1,\ldots, n \}$ is independently drawn from the following sub-exponential signal-plus-noise model:
\begin{equation}
    \x_i = y_i \bmu + \z_i \in \RR^p,
\end{equation}
for $\bmu \in \RR^p$ a deterministic signal, and $\z_i \in \RR^p$ having i.i.d.\@ sub-exponential entries of zero mean and uni variance.
\end{Definition}
The model in \Cref{def:signal_plus_noise} posits that each input consists of a structured signal ($\bmu$) perturbed by i.i.d.\@ random noise, and is widely used in statistical learning under structured inputs. 


\begin{Assumption}[High-dimensional asymptotics]\label{ass:high-dim}
As $ n \to \infty$, we have (1) $p/n \to c \in (0,\infty)$, $d/n \in (0,\infty)$; and (2) the signal $\bmu \in \RR^p$ and weights $\w_Q, \w_K \in \RR^p$ in \Cref{ass:low-rank-weights} satisfy $\limsup_n \max \{ \| \bmu \|, \| \w_Q \|, \| \w_K \| \}  < \infty$.
\end{Assumption}

Under \Cref{def:signal_plus_noise} and \Cref{ass:high-dim}, the input token matrix can be written as $\X = \bmu \y^\top + \Z$, where $\y = [y_1, \ldots, y_n]^\top \in \RR^{n}$, and $\Z \in \RR^{p \times n}$ is a random matrix with i.i.d.\@ sub-exponential entries.
In this case, both the rank-one signal component $\bmu \y^\top$ and the noise matrix $\Z$ have spectral norms of order $O(\sqrt n)$. 
Consequently, they are set on even ground in the high-dimensional regime as $n, p \to \infty$ under \Cref{ass:high-dim}.

\begin{Remark}[Beyond the signal-plus-noise model in \Cref{def:signal_plus_noise}]\normalfont
Our analysis extends beyond the signal-plus-noise model in \Cref{def:signal_plus_noise} in a few aspects.
For instance, the proposed framework naturally accommodates multi-signal settings in which the number of signal directions exceeds two but remains finite as $n,p \to \infty$. 
See \Cref{rem:extension_beyond_binary_signal_plus_noise} in \Cref{sec:further} for a more detailed discussion.
It would also be of interest to study extensions to non-i.i.d.\@ input tokens. 
For example, one may consider temporally correlated inputs of the form $\X = \Z \C$, where $\C \in \RR^{n \times n}$ is a Toeplitz matrix modeling an auto-regressive dependence structure. 
\end{Remark}

In this paper, we aim to quantify the interpolation error of Attention in \Cref{def:Attention}, under the signal-plus-noise model in \Cref{def:signal_plus_noise}.
Precisely, we evaluate the nonlinear Attention as a feature extractor in downstream tasks using \emph{linear probing}. 
Let $\A_\X \in \RR^{p \times n}$ be the Attention output defined in \eqref{eq:def_A_X} of \Cref{def:Attention} for input matrix $\X = [\x_1, \ldots, \x_n] \in \RR^{p \times n}$, and let $\y = [y_1, \ldots, y_n]^\top \in \RR^n$ denote the associated targets.
We define the ridge-regularized \emph{linear probing} vector $\w \in \RR^p$ by minimizing the mean squared error (MSE) on the pair $(\X, \y)$:
\begin{equation}
    \begin{aligned}
   L(\w) &= \frac1n \left\| \y^\top - \w^\top \A_\X \right\|^2 + \gamma \| \W_V^\top \w \|^2 \\
   &= \frac1n \left\| \y^\top - \w_V^\top \X \K_\X \right\|^2 + \gamma \| \w_V \|^2 \equiv L(\w_V), \label{eq:loss_probing}
\end{aligned}
\end{equation}
where $\w_V = \W_V^\top \w \in \RR^p$, and $\gamma \geq 0$ is the regularization penalty.
For any $\gamma > 0$, the unique minimizer of \eqref{eq:loss_probing} admits the following closed-form expression:
\begin{equation}
    \begin{aligned}
   \hat{\w}_V &=  \left( \X \K_\X \K_\X^\top \X^\top+ n \gamma \I_p \right)^{-1} \X \K_\X \y \\
   &= \X \K_\X \left( \K_\X^\top \X^\top \X \K_\X + n \gamma \I_{n} \right)^{-1} \y. \label{eq:def_w_V}
\end{aligned}
\end{equation}
With the explicit linear probing solution $\hat{\w}_V$ given in \eqref{eq:def_w_V}, we now define the Attention \emph{interpolation error} as follow.
\begin{Definition}[\textbf{Interpolation error of Attention}]\label{def:interpolation_error}
For $(\X, \y) \in \RR^{p \times n} \times \RR^n$ drawn from the signal-plus-noise model in \Cref{def:signal_plus_noise}, the \emph{interpolation error} of nonlinear Attention in \Cref{def:Attention} is defined as the MSE of the optimal linear probe $\hat{\w}_V$ in \eqref{eq:def_w_V}:
\begin{align}
   E_{\rm A} &= \frac1n \left\| \y^\top - \hat{\w}_V^\top \X \K_\X \right\|^2 = - \frac{\gamma^2}n \frac{\partial \y^\top \Q(\gamma) \y}{\partial \gamma}, \label{eq:def_E}
\end{align}
where we denote the \emph{nonlinear resolvent matrix}
\begin{equation}\label{eq:def_Q}
   \Q(\gamma) = \left( \K_\X^\top \X^\top \X \K_\X/n + \gamma \I_n \right)^{-1}.
\end{equation}
\end{Definition}

By \Cref{eq:def_E}, assessing the interpolation error of nonlinear Attention reduces to the analysis of the quadratic form $ \y^\top \Q(\gamma) \y$ of the random nonlinear resolvent $\Q(\gamma)$ defined in \eqref{eq:def_Q}.
When the random input tokens $\X$ are drawn from the signal-plus-noise model in \Cref{def:signal_plus_noise}, this analysis presents the following technical challenges.
\begin{enumerate}[leftmargin=*]
   \item The resolvent matrix $\Q$ depends on the input $\X$ in a \emph{nonlinear} fashion: both through the entry-wise nonlinearity $f$ (see \Cref{def:Attention}) and the matrix inverse in \eqref{eq:def_Q}.
   \item The structure of $\Q$ is more complex than classical random matrix models (e.g., Wigner or Wishart matrices) studied in RMT~\citep{bai2010spectral} or high-dimensional statistics~\citep{vershynin2018high}. 
   Specially, the matrix $\K_\X^\top \X^\top \X \K_\X/n$ can be viewed as a nonlinear extension of the standard sample covariance (or Gram) matrix~\citep{marvcenko1967distribution,baik2006eigenvalues}, but with the key distinction that the population covariance that is \emph{dependent} of $\X$. 
\end{enumerate}
These challenges need be addressed to characterize the interpolation error of nonlinear Attention.
To this end, we introduce the notion of \emph{Deterministic Equivalent}, which provides a tractable surrogate for analyzing the high-dimensional behavior of the random resolvent $\Q(\gamma)$ defined in \eqref{eq:def_Q}.
\begin{Definition}[\textbf{Deterministic Equivalent},~{\citep{couillet2022RMT4ML}}]\label{def:DE}
Let $\Q \in \RR^{n \times n}$ be a sequence of random matrices. 
A sequence of deterministic matrices $\bar \Q $ (of the same size) is called a \emph{Deterministic Equivalent} for $\Q$, denoted $\Q \leftrightarrow \bar \Q$, if for all (sequences of) deterministic matrices $\A \in \RR^{n \times n} $ and vectors $\mathbf{a}, \mathbf{b} \in \RR^n$ of unit spectral and Euclidean norm, 
\begin{equation}
   \Q \leftrightarrow \bar \Q: \frac1n \tr \left( \A (\Q - \bar \Q) \right) \to 0, \quad \mathbf{a}^\top (\Q - \bar \Q) \mathbf{b} \to 0,
\end{equation}
in probability as $n \to \infty$.
\end{Definition}
We aim to deriving a Deterministic Equivalent for the nonlinear resolvent $\Q(\gamma)$ defined in \eqref{eq:def_Q}, which in turn enables high-dimensional characterization of the quadratic form $\y^\top \Q(\gamma ) \y/ n$ and the interpolation error $E_{\rm A}$ in \eqref{eq:def_E} of \Cref{def:interpolation_error}.
This will be our focus in the next section.

\section{Main Technical Results}
\label{sec:Results}

This section presents our main technical result.
We begin with \Cref{lem:linearization_K_X}, which establishes a high-dimensional linearization of the Attention matrix $\K_\X$ defined in \eqref{eq:def_A_X}.
Next, \Cref{prop:DE_resovlent_noise} provides a Deterministic Equivalent for the noise-only nonlinear Attention resolvent. 
Together, these results enables a precise characterization of the interpolation error $E_{\rm A}$ defined in \Cref{def:interpolation_error}, which we present in \Cref{theo:high_interpolation} at the end of this section.

To start with, we show that under the full-plus-low-rank decomposition of the Attention weights in \Cref{ass:low-rank-weights}, the Attention matrix $\K_\X$ in \eqref{eq:def_A_X} admits a more tractable ``linearization'' via Hermite polynomial expansion in the high-dimensional regime of \Cref{ass:high-dim}.
This is given in the following result and proven in \Cref{subsec:proof_lem_linearization_K_X}.

\begin{Lemma}[\textbf{High-dimensional linearization of Attention kernel}]\label{lem:linearization_K_X}
Let Assumptions~\ref{ass:low-rank-weights}--\ref{ass:high-dim} hold. 
Then, the Attention kernel $\K_\X = f( \X^\top \W_K^\top \W_Q \X/\sqrt p)/\sqrt p$ defined in \eqref{eq:def_A_X} satisfies 
\begin{equation}
  \| \K_\X - \tilde{\K}_\X \| = O(n^{-1/2}), \quad \tilde{\K}_\X = \K_N + \U_K \bSigma_\K \V_Q^\top,
\end{equation}
with probability approaching one as $n,p \to \infty$. 
Here, $a_1$ is the first Hermite coefficient of $f$ (see \Cref{ass:nonlinear}), $\K_N \equiv f(\Z^\top \Z/\sqrt p)/ \sqrt p - \diag(\cdot)$ is a \emph{symmetric} noise-only kernel matrix, $\U_K = [\y,~\Z^\top \bmu,~\Z^\top \w_K]/\sqrt p\in \RR^{n \times 3}$, $\V_Q = [\y,~\Z^\top \bmu,~\Z^\top \w_Q]/\sqrt p \in \RR^{n \times 3}$, and
\begin{equation}
  \bSigma_\K= a_1 \left[ \begin{smallmatrix} \| \bmu \|^2 + \bmu^\top \w_K \w_Q^\top \bmu & 1 & \bmu^\top \w_K \\ 1 & 0 & 0 \\ \bmu^\top \w_Q & 0 & 1  \end{smallmatrix}\right] \in \RR^{3 \times 3}.
\end{equation}
Moreover, we have that $\max\{ \| \K_N \|, \| \U_K \|, \| \bSigma_\K \|,  \| \V_Q \| \} = O(1)$ with high probability as $n,p \to \infty$.
\end{Lemma}

\Cref{lem:linearization_K_X} shows that the nonlinear kernel matrix $\K_\X$ can be decomposed as the sum of:
\begin{enumerate}[leftmargin=*]
   \item a symmetric noisy-only random kernel matrix $\K_N$;\footnote{The noise-only kernel matrix $\K_N$ is known in the literature as a \emph{random inner-product kernel matrix}~\citep{cheng2013random,fan2019spectral,kammounCovarianceDiscriminativePower2023}, with connections to single-hidden-layer (random) neural networks~\citep{pennington2017nonlinear,benigni2019eigenvalue}.} and
   \item a low-rank, asymmetric informative matrix (rank at most three), whose structure depends on the interaction between the signal $\bmu$ and Attention weights $\w_K, \w_Q$, and on the nonlinearity $f$ only via its first Hermite coefficient $a_1 = \EE[\xi f(\xi)], \xi \sim \NN(0,1)$.
\end{enumerate}

In the following result, we focus on the noise-only part of the Attention matrix and derive a Deterministic Equivalent for its resolvent, the proof of which is given in \Cref{subsec:proof_of_prop:DE_resovlent_noise}.

\begin{Proposition}[\textbf{Deterministic Equivalent for noise-only nonlinear Attention}]\label{prop:DE_resovlent_noise}
Let $\Z \in \RR^{p \times n}$ be a random matrix having i.i.d.\@ sub-exponential entries, and define the symmetric noise-only kernel matrix $\K_N = f(\Z^\top \Z/\sqrt p)/\sqrt p - \diag(\cdot)$ as in \Cref{lem:linearization_K_X}.
Then, as $n,p \to \infty$ with $p/n \to c \in (0,\infty)$ and $\gamma > 0$, the following Deterministic Equivalent holds
\begin{equation*}
   \left(\K_N \Z^\top \Z \K_N/n + \gamma \I_n \right)^{-1} \leftrightarrow m(\gamma)/c \cdot \I_n, 
\end{equation*}
where $m(\gamma)$ is the unique Stieltjes transform solution to the fixed-point equation
\begin{align*}
   m(\gamma) &= \left( \gamma/c + \nu/c +  a_1^2/c^2 - \bv^\top \mathbf{T}(\gamma) \bv \right)^{-1},
\end{align*}
with symmetric matrix $\mathbf{T}(\gamma) \in \RR^{6 \times 6}$ having entries polynomial in $m(\gamma)$, $\delta_1(\gamma)$, $\delta_2(\gamma)$, $\delta_3(\gamma)$, $\delta_4(\gamma)$ as defined in \eqref{eq:full_expression} of \Cref{subsec:proof_of_prop:DE_resovlent_noise} in the appendix and $\bv = \begin{bmatrix} a_1^2 (1+c)/c^2  & a_1/c & a_1/c & 0 & 0 & 1 \end{bmatrix}^\top \in \RR^6$.

Notably, the system of equations depends on the regularization $\gamma$, the dimension ratio $c = \lim p/n$, and the nonlinearity $f$ via its Hermite coefficients $a_1$ and $\nu$ in \Cref{ass:nonlinear}.
\end{Proposition}

Using~\Cref{lem:linearization_K_X}~and~\Cref{prop:DE_resovlent_noise}, we obtain the following characterization of the interpolation error $E_{\rm A}$ for nonlinear Attention.
The proof is given in \Cref{subsec:proof_of_theo:ICM} of the appendix.

\begin{Theorem}[\textbf{High-dimensional characterization of Attention interpolation error}]\label{theo:high_interpolation}
Let Assumptions~\ref{ass:low-rank-weights}--\ref{ass:high-dim} hold. 
Then, the interpolation error $E_{\rm A}$ defined in \eqref{eq:def_E} satisfies $E_{\rm A} - \bar E_{\rm A} \to 0$ in probability as $n,p \to \infty$ with $ p/n \to c \in (0,\infty)$, where
\begin{equation}
    \bar E_{\rm A} = - \gamma^2 c^2 \cdot \ee_7^\top \left( c \I_9 + \bDelta(\gamma) \bLambda \right)^{-1} \bDelta'(\gamma) \left( c \I_9 + \bLambda \bDelta(\gamma) \right)^{-1}  \ee_7.
\end{equation}
Here, $\ee_7 \in \RR^9$ is the canonical basis vector with $[\ee_i]_j = \delta_{ij}$, $\bLambda, \bDelta(\gamma) \in \RR^{9 \times 9}$ are deterministic symmetric matrices defined in \Cref{lem:further_approx} of \Cref{subsec:proof_of_theo:ICM}, and $\bDelta'(\gamma)$ is the derivative of $\bDelta(\gamma)$ with respect to $\gamma$.
\end{Theorem}

\section{Interpolation Error: Nonlinear Attention versus Linear Regression}
\label{sec:Analysis}

In this section, we discuss the implications of our technical results in \Cref{theo:high_interpolation}, by contrasting the interpolation behavior of nonlinear Attention with that of linear regression.

\subsection{Interpolation Error of Linear Regression}
\label{subsec:ridge_regression}

Consider a classical baseline where the input embedding matrix $\X$ is directly used for linear probing, instead of the nonlinear Attention output $\A_\X$ defined in \eqref{eq:def_A_X} of \Cref{def:Attention}. 
In this case, the probing vector $\w_{\rm LR} \in \RR^p$ is obtained by minimizing the following ridge-regularized MSE:
\begin{equation}\label{eq:loss_probing_LR}
   L_{\rm LR}(\w) = \frac1n \left\| \y^\top - \w^\top \X \right\|^2 + \gamma \| \w \|^2.
\end{equation}
This leads to the linear regression model defined below. 

\begin{Definition}[\textbf{Linear regression and its interpolation error}]\label{def:RR_interpolation_error}
For $(\X, \y) \in \RR^{p \times n} \times \RR^n$ drawn from the signal-plus-noise model in \Cref{def:signal_plus_noise}, the linear  regression solution $\hat{\w}_{\rm LR}$ to \eqref{eq:loss_probing_LR} is explicitly given, for $\gamma > 0$, by
\begin{equation}\label{eq:def_w_LR}
   \hat{\w}_{\rm LR} =  \left( \X \X^\top + n \gamma \I_p \right)^{-1} \X \y= \X \left( \X^\top \X + n \gamma \I_{n} \right)^{-1} \y.
\end{equation}
Its associated \emph{interpolation error} is defined as
\begin{equation*}\label{eq:def_E_RR}
   E_{\rm LR} = \frac1n \left\| \y^\top - \hat{\w}_{\rm LR}^\top \X \right\|^2 = - \frac{\gamma^2}n \frac{\partial \y^\top \left( \X^\top \X/n  + \gamma \I_n \right)^{-1} \y}{\partial \gamma}, 
\end{equation*}
which is also the derivative of the quadratic form of the \emph{linear resolvent} $\left( \X^\top \X/n  + \gamma \I_n \right)^{-1}$.
\end{Definition}

The following result characterizes the linear regression interpolation error $E_{\rm LR}$ in \eqref{eq:def_E_RR} of \Cref{def:RR_interpolation_error}.
The proof is included in \Cref{subsec:proof_of_prop:ICM_RR} of the appendix for completeness.

\begin{Proposition}[\textbf{High-dimensional characterization of interpolation error for linear regression}]\label{prop:high_interpolation_error_RR}
Let \Cref{ass:high-dim} hold.
Then, the interpolation error $E_{\rm LR}$ defined in \eqref{eq:def_E_RR} of the linear regression model in \Cref{def:RR_interpolation_error} satisfies $E_{\rm LR} - \bar E_{\rm LR} \to 0$ in probability as $n,p \to \infty$, with $\bar E_{\rm LR}$ given in \eqref{eq:def_bar_E_LR}, for $m_{\rm LR}(\gamma)$ the Stieltjes transform solution to the popular Mar\u{c}enko-Pastur equation~\citep{marvcenko1967distribution} in \eqref{eq:def_E_RR} and $m_{\rm LR}'(\gamma) $ its derivative with respect to $\gamma$.
\end{Proposition}

\begin{figure*}[htb]
\centering
\begin{equation}\label{eq:def_bar_E_LR}
   \bar E_{\rm LR} = - \frac{ c \gamma^2 m'(\gamma) + c - 1 + \| \bmu\|^2 \left(\gamma^2 m'(\gamma) + (1 - c - \gamma) (\gamma m(\gamma) - 1) \right) }{ \left( 1 +  \| \bmu \|^2 - \| \bmu \|^2 \gamma m_{\rm LR}(\gamma) \right)^2 }, \quad c\gamma m_{\rm LR}^2(\gamma) + \left( 1 - c + \gamma \right) m_{\rm LR}(\gamma) - 1 =0.
\end{equation}
\end{figure*}


In what follows, we leverage \Cref{prop:high_interpolation_error_RR} to assess how the interpolation error $E_{\rm LR}$ of linear regression is influenced by: the regularization strength $\gamma$, the dimension ratio $c = \lim p/n$, and the signal-to-noise ratio (SNR) $\| \bmu \|^2$.

\begin{Remark}[Effect of regularization strength for linear regression]\label{rem:regularization_RR}\normalfont
Under the settings and notations of \Cref{prop:high_interpolation_error_RR}, the interpolation error $E_{\rm LR}$ is an \emph{increasing} function of the regularization strength $\gamma$. 
In the ``ridgeless'' limit $\gamma \to 0$, the interpolation error vanishes $E_{\rm LR}  \to 0$ for $p > n$; whereas in the strongly regularized  limit $\gamma \to \infty$ we have $E_{\rm LR} \to 1$.
Interestingly, when $\gamma \to 0$ and $c = \lim p/n  \to 0$, the Stieltjes transform $m_{\rm LR}(\gamma)$ becomes singular, which is connected to the now well-known ``double descent'' phenomenon in test error curves~\citep{bartlett2020benign,mei2021generalization,liao2020random,hastie2022Surprises}.
\end{Remark}
\begin{Remark}[Effect of embedding dimension for linear regression]\label{rem:ratio_RR}\normalfont
The interpolation error $E_{\rm LR}$ of linear regression is a \emph{decreasing} function of the dimension ratio $c = \lim p/n$. 
For fixed $n$, increasing the embedding dimension $p$ thus improves interpolation. 
In the limit $c \to 0$ and for $\gamma =0 $, the interpolation error converges to $E_{\rm LR} \to 1/(1 + \| \bmu \|^2)$.
Moreover, in the under-parametrized setting with $p < n$ and $\gamma = 0$, the interpolation error $E_{\rm LR}$ scales approximately with the embedding dimension $p$ as $1-c = 1 - p/n$, in line with classical statistical learning theory~\citep{bach2024Learning}. 
\end{Remark}
\begin{Remark}[Effect of SNR for linear regression]\label{rem:SNR_RR}\normalfont
The interpolation error $E_{\rm LR}$ \emph{decreases} with the SNR $\| \bmu \|^2$. 
In the limit $\| \bmu \| \to \infty$, one has $E_{\rm LR} \to 0$. 
In particular, for $\gamma = 0$ and $p < n$, the error scales as $E_{\rm LR} \propto 1/(1 + \| \bmu \|^2)$, a trend clearly illustrated in \Cref{subfig:E_SNR}.
\end{Remark}
%
%
%

\subsection{Interpolation Error of Nonlinear Attention versus Linear Regression}

In this subsection, we compare the interpolation error of nonlinear Attention (in \Cref{theo:high_interpolation}) with that of linear regression (in \Cref{prop:high_interpolation_error_RR}).

In \Cref{fig:capacity_versus_regularization}, we report the empirical interpolation error $E_{\rm A}$ of nonlinear Attention together with its theoretical (limiting) counterpart $\bar E_{\rm A}$ from \Cref{theo:high_interpolation}, and compare both against linear regression under the same setting.

In \Cref{subfig:E_gamma} and \Cref{subfig:E_p}, we consider the null model in the absence of statistical signal ($\bmu = \zo$) and identity Attention weights ($\w_K = \w_Q = \zo$). 
In this setting, the interpolation error of nonlinear Attention exhibits the same qualitative trends as linear regression (discussed in \Cref{subsec:ridge_regression}): 
it increases with the regularization strength $\gamma$ and decreases with the embedding dimension $p$. 
Quantitatively, however, nonlinear Attention (with, e.g., $\tanh$ nonlinearity in \Cref{fig:capacity_versus_regularization}) incurs a \emph{higher} interpolation error than linear regression, but \emph{only in the absence of signal}.

In contrast, when structured input signals are present ($\bmu \neq \zo$) and particularly when the Attention weights $\w_K, \w_Q$ are \emph{aligned} with the signal direction, \Cref{subfig:E_SNR} shows that the interpolation error of Attention becomes visually indistinguishable from that of linear regression.
This indicates that the apparent disadvantage of Attention in interpolation disappears once it is tuned to the underlying input structure.

We further provide numerical results in \Cref{subfig:E4_1} and in \Cref{sec:SM_additional_nums} of the appendix, demonstrating that this disadvantage can even reverse: nonlinear Attention may achieve \emph{strictly lower} interpolation error than linear regression, particularly in the \emph{high SNR and/or limited sample regime}, see, e.g., a detailed comparison in \Cref{fig_capacity_of_SNR_n_bigger_than_p} in \Cref{sec:SM_additional_nums}.

\begin{figure*}[tb]
\centering
\input{fig_capacity_versus_regularization}
\caption{ {  Empirical interpolation error $E$ (\textbf{\RED red}) of nonlinear Attention versus its high-dimensional equivalent $\bar{E}$ (\textbf{\BLUE blue}) from \Cref{theo:high_interpolation}, and the theoretical interpolation error of linear regression (\textbf{\GREEN green}) from \Cref{prop:high_interpolation_error_RR}, with $f(t) = \tanh(t)$. 
\textbf{\Cref{subfig:E_gamma}}: As a function of regularization strength $\gamma$, under null model with $\bmu = \w_K = \w_Q = \zo$, $p = 4\,096$, and $n = 1\,024$.
\textbf{\Cref{subfig:E_p}}: As a function of embedding dimension $p$, under null model with $n = 4\,096$, $\gamma = 10^{-2}$.
\textbf{\Cref{subfig:E_SNR}}: As a function of SNR $\| \bmu \|^2$, with $p = 512, n = 2\,048,\gamma = 10^{-2}$, and $\w_K = \w_Q = \bmu$. 
}}
\label{fig:capacity_versus_regularization}
\end{figure*}


\begin{figure}[!t]
\centering
\begin{subfigure}[t]{0.23\textwidth}
  \begin{tikzpicture}[font=\footnotesize]
    \renewcommand{\axisdefaulttryminticks}{4} 
    \pgfplotsset{every major grid/.append style={densely dashed}}          
    \tikzstyle{every axis y label}+=[yshift=-10pt] 
    \tikzstyle{every axis x label}+=[yshift=5pt]
    \pgfplotsset{every axis legend/.style={cells={anchor=west},fill=white,at={(0.98,0.9)}, anchor=north east, font=\footnotesize}}
    \begin{axis}[
      width=1\linewidth,
      height=.8\linewidth,
      xmin=0.1,xmax=10,
      xmode=log,
      ymax=1,
        grid=major,
        ymajorgrids=false,
        scaled ticks=false,
        xlabel={ SNR $\| \bmu \|^2$ },
        ylabel={ $E$ }
        ]
        \addplot[smooth,PURPLE,line width=1pt] plot coordinates{
        (0.100000,0.730849)(0.117210,0.720123)(0.137382,0.707871)(0.161026,0.693955)(0.188739,0.678249)(0.221222,0.660652)(0.259294,0.641093)(0.303920,0.619542)(0.356225,0.596019)(0.417532,0.570598)(0.489390,0.543416)(0.573615,0.514670)(0.672336,0.484621)(0.788046,0.453583)(0.923671,0.421917)(1.082637,0.390012)(1.268961,0.358266)(1.487352,0.327073)(1.743329,0.296795)(2.043360,0.267754)(2.395027,0.240214)(2.807216,0.214378)(3.290345,0.190384)(3.856620,0.168309)(4.520354,0.148174)(5.298317,0.129955)(6.210169,0.113587)(7.278954,0.098976)(8.531679,0.086009)(10.000000,0.074560)
        };
        \addplot[smooth,BLUE,line width=1pt] plot coordinates{ 
        (0.100000,0.753825)(0.117210,0.743281)(0.137382,0.731129)(0.161026,0.717198)(0.188739,0.701327)(0.221222,0.683378)(0.259294,0.663258)(0.303920,0.640924)(0.356225,0.616405)(0.417532,0.589802)(0.489390,0.561296)(0.573615,0.531142)(0.672336,0.499656)(0.788046,0.467205)(0.923671,0.434186)(1.082637,0.401015)(1.268961,0.368104)(1.487352,0.335848)(1.743329,0.304606)(2.043360,0.274693)(2.395027,0.246365)(2.807216,0.219817)(3.290345,0.195181)(3.856620,0.172528)(4.520354,0.151876)(5.298317,0.133193)(6.210169,0.116411)(7.278954,0.101434)(8.531679,0.088143)(10.000000,0.076408)
        };
        \addplot[densely dashed,GREEN,line width=1.5pt] plot coordinates{ 
        (0.100000,0.770551)(0.117210,0.761873)(0.137382,0.751919)(0.161026,0.740541)(0.188739,0.727587)(0.221222,0.712907)(0.259294,0.696357)(0.303920,0.677810)(0.356225,0.657161)(0.417532,0.634347)(0.489390,0.609349)(0.573615,0.582213)(0.672336,0.553055)(0.788046,0.522071)(0.923671,0.489540)(1.082637,0.455820)(1.268961,0.421336)(1.487352,0.386565)(1.743329,0.352006)(2.043360,0.318158)(2.395027,0.285488)(2.807216,0.254401)(3.290345,0.225229)(3.856620,0.198210)(4.520354,0.173490)(5.298317,0.151126)(6.210169,0.131096)(7.278954,0.113316)(8.531679,0.097656)(10.000000,0.083954)
        };
        \end{axis}
  \end{tikzpicture}
  \caption{{ $E$ versus SNR }}
  \label{subfig:E4_1}
  \end{subfigure}
  ~
  \begin{subfigure}[t]{0.23\textwidth}
  \begin{tikzpicture}[font=\footnotesize]
    \renewcommand{\axisdefaulttryminticks}{4} 
    \pgfplotsset{every major grid/.append style={densely dashed}}          
    \tikzstyle{every axis y label}+=[yshift=-10pt] 
    \tikzstyle{every axis x label}+=[yshift=5pt]
    \pgfplotsset{every axis legend/.style={cells={anchor=west},fill=white,at={(0.98,0.9)}, anchor=north east, font=\footnotesize}}
    \begin{axis}[
      width=1\linewidth,
      height=.8\linewidth,
      xmin=0.01,xmax=1000,
      xmode=log,
        grid=major,
        ymajorgrids=false,
        scaled ticks=false,
        xlabel={ Regularization penalty $\gamma$ },
        ylabel={  }
        ]
        \addplot[densely dashed,cyan,line width=1pt] plot coordinates{
        (0.010000,0.090772)(0.014874,0.091385)(0.022122,0.091947)(0.032903,0.092593)(0.048939,0.093518)(0.072790,0.095046)(0.108264,0.097732)(0.161026,0.102523)(0.239503,0.110996)(0.356225,0.125609)(0.529832,0.149789)(0.788046,0.187552)(1.172102,0.242353)(1.743329,0.315318)(2.592944,0.403767)(3.856620,0.501225)(5.736153,0.599218)(8.531679,0.689867)(12.689610,0.767839)(18.873918,0.830941)(28.072162,0.879563)(41.753189,0.915624)(62.101694,0.941613)(92.367086,0.959954)(137.382380,0.972705)(204.335972,0.981477)(303.919538,0.987467)(452.035366,0.991537)(672.335754,0.994294)(1000.000000,0.996156)
        };
        \addplot[densely dotted,BLUE,line width=1pt] plot coordinates{
        (0.010000,0.063815)(0.014874,0.089794)(0.022122,0.120303)(0.032903,0.154403)(0.048939,0.190633)(0.072790,0.227313)(0.108264,0.263281)(0.161026,0.298537)(0.239503,0.334298)(0.356225,0.372505)(0.529832,0.415146)(0.788046,0.463741)(1.172102,0.518983)(1.743329,0.580330)(2.592944,0.645645)(3.856620,0.711344)(5.736153,0.773264)(8.531679,0.827881)(12.689610,0.873183)(18.873918,0.908822)(28.072162,0.935684)(41.753189,0.955275)(62.101694,0.969218)(92.367086,0.978969)(137.382380,0.985704)(204.335972,0.990317)(303.919538,0.993457)(452.035366,0.995586)(672.335754,0.997025)(1000.000000,0.997997)
        };
        \addplot[densely dotted,RED,line width=1pt] plot coordinates{
        (0.010000,0.002307)(0.014874,0.003546)(0.022122,0.005436)(0.032903,0.008406)(0.048939,0.013084)(0.072790,0.020254)(0.108264,0.030738)(0.161026,0.045266)(0.239503,0.064418)(0.356225,0.088592)(0.529832,0.117920)(0.788046,0.152137)(1.172102,0.190569)(1.743329,0.232394)(2.592944,0.277093)(3.856620,0.324845)(5.736153,0.376539)(8.531679,0.433254)(12.689610,0.495398)(18.873918,0.562025)(28.072162,0.630776)(41.753189,0.698402)(62.101694,0.761525)(92.367086,0.817339)(137.382380,0.864125)(204.335972,0.901448)(303.919538,0.929973)(452.035366,0.951026)(672.335754,0.966151)(1000.000000,0.976803)
        };
        \addplot[smooth,BLUE,line width=1pt] plot coordinates{
        (0.010000,0.081063)(0.014874,0.099511)(0.022122,0.122385)(0.032903,0.150776)(0.048939,0.185921)(0.072790,0.229024)(0.108264,0.280911)(0.161026,0.341518)(0.239503,0.409443)(0.356225,0.481888)(0.529832,0.555227)(0.788046,0.625955)(1.172102,0.691480)(1.743329,0.750353)(2.592944,0.801997)(3.856620,0.846303)(5.736153,0.883378)(8.531679,0.913507)(12.689610,0.937211)(18.873918,0.955267)(28.072162,0.968619)(41.753189,0.978249)(62.101694,0.985060)(92.367086,0.989805)(137.382380,0.993074)(204.335972,0.995311)(303.919538,0.996833)(452.035366,0.997864)(672.335754,0.998561)(1000.000000,0.999031)
        };
        \addplot[smooth,RED,line width=1pt] plot coordinates{
        (0.010000,0.039663)(0.014874,0.048441)(0.022122,0.059195)(0.032903,0.072387)(0.048939,0.088593)(0.072790,0.108525)(0.108264,0.133045)(0.161026,0.163149)(0.239503,0.199901)(0.356225,0.244258)(0.529832,0.296774)(0.788046,0.357219)(1.172102,0.424252)(1.743329,0.495391)(2.592944,0.567429)(3.856620,0.637164)(5.736153,0.702055)(8.531679,0.760492)(12.689610,0.811653)(18.873918,0.855239)(28.072162,0.891297)(41.753189,0.920185)(62.101694,0.942576)(92.367086,0.959393)(137.382380,0.971680)(204.335972,0.980458)(303.919538,0.986620)(452.035366,0.990891)(672.335754,0.993823)(1000.000000,0.995822)
        };
        \end{axis}
  \end{tikzpicture}
  \caption{{ $E$ versus $\gamma$ }}
  \label{subfig:E4_2}
  \end{subfigure}
\caption{ { 
Interpolation behavior under alternative nonlinearities and pretrained Attention weights. \textbf{\Cref{subfig:E4_1}}: Theoretical interpolation errors for $ f(t) =\tanh(t) $ (\textbf{\BLUE blue}) versus $ f(t) = \max(-5,\min(5,t)) $ (\textbf{\PURPLE purple}) and that of linear regression (\textbf{\GREEN green}) in the over-determined regime, as a function of SNR, with $p = 512, n = 2\,048,\gamma = 1$, and $\w_K = \w_Q = \bmu$. 
\textbf{\Cref{subfig:E4_2}}: Interpolation error of Softmax (\textcolor{cyan}{\textbf{cyan}}) and entry-wise tanh (\textbf{\BLUE blue}), truncated exponential ($ f(t) = \min(5, \exp(t)) $ in \textbf{\RED red}) Attention. 
Theoretical predictions under \Cref{ass:low-rank-weights} in \underline{solid} lines and predictions with key/query weights extracted from a pretrained GPT-2 model (the detailed setup is provided in \Cref{sec:SM_additional_nums}) in \underline{dotted} lines, as a function of $\gamma$, for $p = 2\,048, n = 512$, and $ \w_K = \w_Q = \bmu = \zo $. 
 }}
\label{Figure4}
\end{figure}



\Cref{Figure4} further compares nonlinear Attention and linear regression in two cases:
\begin{enumerate}[leftmargin=*]
   \item In \Cref{subfig:E4_1}, nonlinear Attention achieves lower interpolation error than linear regression when $ p/n < 1 $ and the SNR is low, particularly in the weakly regularized setting. 
   Additional experiments are reported in \Cref{fig:capacity_versus_SNR} and \Cref{fig_capacity_of_SNR_n_bigger_than_p} (\Cref{sec:SM_additional_nums} of the appendix). 
   \item In \Cref{subfig:E4_2}, we evaluate Softmax and entry-wise Attention using pretrained GPT-2 weights. 
   The empirical behavior remains consistent with our theoretical predictions: Attention follows the same qualitative trends as in our theoretical prediction, while attaining quantitatively lower interpolation error. 
   Further results are provided in \Cref{real_weight_capacity_versus_gamma} (\Cref{sec:SM_additional_nums} of the appendix).
\end{enumerate}
Overall, these results suggest that under structured inputs and aligned weights, nonlinear Attention can match or even surpass linear regression in interpolation performance, particularly in low-SNR and/or weak-sample regimes.

\begin{figure*}[!t]
\centering
\input{fig_capacity_versus_regularization_cos} 
\caption{ { 
Effect of linear component in Attention interpolation.
\textbf{\Cref{subfig:E_a1}}: Empirical (\textbf{\RED red}) and theoretical (\textcolor{cyan}{\textbf{cyan}}) interpolation error for $ f(t) = \max(-5,\min(5,rt + \sqrt{1 - r^2} (t^3 - 3t)/\sqrt6)) $ as a function of the Hermite coefficient $ a_1 \approx r$ for $ p = n = 4\,096, \gamma = 1$, and $\| \bmu \|^2 = 1$. 
\textbf{\Cref{subfig:E_cos_p}}: Empirical (\textbf{\RED red}) and theoretical (\textcolor{cyan}{\textbf{cyan}}) for $ f(t) = \cos(t) $, versus the theoretical error of $ f(t) = \tanh(t) $ (\textbf{\BLUE blue}) and the theoretical error of $ f(t) = \max(-5,\min(5,t)) $ (\textbf{\PURPLE purple}), as a function of the embedding dimension $p$, for sample size $n = 4\,096, \gamma = 1$, and $\| \bmu \|^2 = 1$. 
\textbf{\Cref{subfig:E_cos_SNR}}: Empirical (\textbf{\RED red}) and theoretical (\textcolor{cyan}{\textbf{cyan}}) for $ f(t) = \cos(t) $, versus the theoretical error of $ f(t) = \tanh(t) $ (\textbf{\BLUE blue}) and the theoretical error of $ f(t) = \max(-5,\min(5,t)) $ (\textbf{\PURPLE purple}), as a function of the SNR $\| \bmu \|^2$, for $p = 512, n = 2\,048,\gamma = 10^{-2}$, and $\w_K = \w_Q = \bmu$.
 }}
\label{fig:capacity_versus_regularization_cos}
\end{figure*}


\subsection{Importance of the Attention Linear Component in Interpolation}
\label{subsec:importance_of_the_attention_linear_component_in_interpolation}

\Cref{fig:capacity_versus_regularization_cos} examines the impact of the \emph{linear component} of the nonlinearity $f$ in Attention, quantified by its first Hermite coefficient
\begin{equation}
  a_1 = \EE[\xi f(\xi)], \xi \sim \NN(0,1),
\end{equation}
on interpolation performance.

In \Cref{subfig:E_a1}, we consider a family of nonlinearities
\begin{equation*}
  f_r(t) = \max \left\{ -5, \min(5, r \cdot {\rm He}_1(t) + \sqrt{1- r^2} \cdot {\rm He}_3(t)) \right\},
\end{equation*}
parameterized by $r \in [0,1]$, where ${\rm He}_1(t) = t$ and ${\rm He}_3(t) = (t^3 - 3t)/\sqrt 6$ are the first and third normalized Hermite polynomials.
The ``total energy''  of $f$ is fixed at $\nu = \EE_{\xi \sim\mathcal{N}(0,1)}[f^2(\xi)] \approx 1$, so that varying $r$ changes only the proportion of linear versus purely nonlinear (in fact cubic) components. 
We observe a monotone decrease of interpolation error as $a_1$ increases, indicating that the linear component directly governs interpolation efficiency.

To further illustrate this effect, \Cref{subfig:E_cos_p} and \Cref{subfig:E_cos_SNR} compare three nonlinearities:
$f(t) = \tanh(t)$ (with $a_1 = 0.6057$), 
bounded linear $f(t) = \max(-5,\min(5,t))$ (with $a_1 \approx 1$), and 
$f(t ) = \cos(t)$ (with $a_1 \approx 0$). 
As shown in \Cref{subfig:E_cos_p}, increasing the embedding dimension $p$ improves interpolation only when $a_1 \neq 0$.
When $a_1 \approx 0$, no significant dimensional gain is observed. 
Similarly, \Cref{subfig:E_cos_SNR} demonstrates that interpolation error decreases with SNR only for nonlinearities possessing a \emph{nonzero} linear component; cosine-based Attention exhibits almost no improvement.

These results provide evidence that the first Hermite coefficient $a_1$ acts as a key control parameter for interpolation: without a linear component, Attention \emph{cannot} effectively leverage increasing dimension or signal strength. 
This aligns with the characterization in \Cref{theo:high_interpolation}, where the interpolation error depends explicitly on $a_1$.


\begin{figure}[!t]
\centering
 \begin{subfigure}[t]{0.23\textwidth}
  \begin{tikzpicture}[font=\footnotesize]
    \renewcommand{\axisdefaulttryminticks}{4} 
    \pgfplotsset{every major grid/.append style={densely dashed}}          
    \tikzstyle{every axis y label}+=[yshift=-10pt] 
    \tikzstyle{every axis x label}+=[yshift=5pt]
    \pgfplotsset{every axis legend/.style={cells={anchor=west},fill=white,at={(0.98,0.9)}, anchor=north east, font=\footnotesize}}
    \begin{axis}[
      width=1\linewidth,
      height=.8\linewidth,
      xmin=0.1,xmax=10,
      xmode=log,
      ymax=1,
        grid=major,
        ymajorgrids=false,
        scaled ticks=false,
        xlabel={ SNR $\| \bmu \|^2$ },
        ylabel={ $E$ }
        ]
        \addplot[smooth,BLUE,line width=1pt] plot coordinates{ 
        (0.100000,0.689194)(0.117210,0.680450)(0.137382,0.670378)(0.161026,0.658817)(0.188739,0.645601)(0.221222,0.630569)(0.259294,0.613577)(0.303920,0.594503)(0.356225,0.573272)(0.417532,0.549871)(0.489390,0.524362)(0.573615,0.496902)(0.672336,0.467748)(0.788046,0.437252)(0.923671,0.405848)(1.082637,0.374022)(1.268961,0.342279)(1.487352,0.311109)(1.743329,0.280950)(2.043360,0.252173)(2.395027,0.225064)(2.807216,0.199823)(3.290345,0.176568)(3.856620,0.155347)(4.520354,0.136147)(5.298317,0.118907)(6.210169,0.103531)(7.278954,0.089899)(8.531679,0.077876)(10.000000,0.067320)
        };
        \addplot[densely dotted,BLUE,line width=1pt] plot coordinates{
        (0.100000,0.712041)(0.117210,0.706674)(0.137382,0.700357)(0.161026,0.692923)(0.188739,0.684179)(0.221222,0.673902)(0.259294,0.661844)(0.303920,0.647731)(0.356225,0.631274)(0.417532,0.612186)(0.489390,0.590207)(0.573615,0.565148)(0.672336,0.536936)(0.788046,0.505677)(0.923671,0.471700)(1.082637,0.435579)(1.268961,0.398115)(1.487352,0.360254)(1.743329,0.322976)(2.043360,0.287161)(2.395027,0.253503)(2.807216,0.222457)(3.290345,0.194257)(3.856620,0.168951)(4.520354,0.146459)(5.298317,0.126617)(6.210169,0.109220)(7.278954,0.094039)(8.531679,0.080844)(10.000000,0.069413)
        };
        \addplot[smooth,PURPLE,line width=1pt] plot coordinates{ 
        (0.100000,0.608346)(0.117210,0.600080)(0.137382,0.590606)(0.161026,0.579793)(0.188739,0.567510)(0.221222,0.553635)(0.259294,0.538063)(0.303920,0.520718)(0.356225,0.501563)(0.417532,0.480616)(0.489390,0.457956)(0.573615,0.433737)(0.672336,0.408184)(0.788046,0.381594)(0.923671,0.354318)(1.082637,0.326747)(1.268961,0.299283)(1.487352,0.272318)(1.743329,0.246208)(2.043360,0.221258)(2.395027,0.197712)(2.807216,0.175744)(3.290345,0.155464)(3.856620,0.136921)(4.520354,0.120112)(5.298317,0.104993)(6.210169,0.091488)(7.278954,0.079496)(8.531679,0.068907)(10.000000,0.059599)
        };
        \addplot[densely dotted,PURPLE,line width=1pt] plot coordinates{
        (0.100000,0.627090)(0.117210,0.621453)(0.137382,0.614857)(0.161026,0.607146)(0.188739,0.598145)(0.221222,0.587661)(0.259294,0.575489)(0.303920,0.561412)(0.356225,0.545223)(0.417532,0.526736)(0.489390,0.505817)(0.573615,0.482417)(0.672336,0.456604)(0.788046,0.428595)(0.923671,0.398773)(1.082637,0.367673)(1.268961,0.335944)(1.487352,0.304288)(1.743329,0.273380)(2.043360,0.243802)(2.395027,0.216006)(2.807216,0.190290)(3.290345,0.166816)(3.856620,0.145627)(4.520354,0.126680)(5.298317,0.109867)(6.210169,0.095044)(7.278954,0.082046)(8.531679,0.070697)(10.000000,0.060825)
        };
        \end{axis}
  \end{tikzpicture}
  \caption{{ $ p/n = 1 $ }}
  \label{subfig:SNR_E2}
  \end{subfigure}
  ~
  \begin{subfigure}[t]{0.23\textwidth}
  \begin{tikzpicture}[font=\footnotesize]
    \renewcommand{\axisdefaulttryminticks}{4} 
    \pgfplotsset{every major grid/.append style={densely dashed}}          
    \tikzstyle{every axis y label}+=[yshift=-10pt] 
    \tikzstyle{every axis x label}+=[yshift=5pt]
    \pgfplotsset{every axis legend/.style={cells={anchor=west},fill=white,at={(0.98,0.9)}, anchor=north east, font=\footnotesize}}
    \begin{axis}[
      width=1\linewidth,
      height=.8\linewidth,
      xmin=0.1,xmax=10,
      xmode=log,
      ymax=1,
        grid=major,
        ymajorgrids=false,
        scaled ticks=false,
        xlabel={ SNR $\| \bmu \|^2$ },
        ylabel={ }
        ]
        \addplot[smooth,BLUE,line width=1pt] plot coordinates{ 
        (0.100000,0.642323)(0.117210,0.638146)(0.137382,0.633220)(0.161026,0.627410)(0.188739,0.620561)(0.221222,0.612489)(0.259294,0.602986)(0.303920,0.591819)(0.356225,0.578734)(0.417532,0.563462)(0.489390,0.545740)(0.573615,0.525334)(0.672336,0.502075)(0.788046,0.475907)(0.923671,0.446936)(1.082637,0.415476)(1.268961,0.382073)(1.487352,0.347487)(1.743329,0.312629)(2.043360,0.278448)(2.395027,0.245809)(2.807216,0.215389)(3.290345,0.187616)(3.856620,0.162679)(4.520354,0.140570)(5.298317,0.121146)(6.210169,0.104195)(7.278954,0.089473)(8.531679,0.076731)(10.000000,0.065736)
        };
        \addplot[densely dotted,BLUE,line width=1pt] plot coordinates{
        (0.100000,0.656799)(0.117210,0.655129)(0.137382,0.653142)(0.161026,0.650773)(0.188739,0.647939)(0.221222,0.644544)(0.259294,0.640465)(0.303920,0.635551)(0.356225,0.629619)(0.417532,0.622441)(0.489390,0.613741)(0.573615,0.603187)(0.672336,0.590389)(0.788046,0.574900)(0.923671,0.556235)(1.082637,0.533911)(1.268961,0.507508)(1.487352,0.476775)(1.743329,0.441763)(2.043360,0.402955)(2.395027,0.361351)(2.807216,0.318438)(3.290345,0.275989)(3.856620,0.235752)(4.520354,0.199106)(5.298317,0.166852)(6.210169,0.139206)(7.278954,0.115931)(8.531679,0.096544)(10.000000,0.080476)
        };
        \addplot[smooth,PURPLE,line width=1pt] plot coordinates{ 
        (0.100000,0.482500)(0.117210,0.478822)(0.137382,0.474498)(0.161026,0.469416)(0.188739,0.463449)(0.221222,0.456455)(0.259294,0.448276)(0.303920,0.438742)(0.356225,0.427679)(0.417532,0.414919)(0.489390,0.400318)(0.573615,0.383774)(0.672336,0.365258)(0.788046,0.344840)(0.923671,0.322710)(1.082637,0.299192)(1.268961,0.274729)(1.487352,0.249857)(1.743329,0.225148)(2.043360,0.201148)(2.395027,0.178325)(2.807216,0.157030)(3.290345,0.137486)(3.856620,0.119798)(4.520354,0.103970)(5.298317,0.089939)(6.210169,0.077594)(7.278954,0.066798)(8.531679,0.057402)(10.000000,0.049258)
        };
        \addplot[densely dotted,PURPLE,line width=1pt] plot coordinates{
        (0.100000,0.495781)(0.117210,0.494359)(0.137382,0.492657)(0.161026,0.490618)(0.188739,0.488167)(0.221222,0.485214)(0.259294,0.481648)(0.303920,0.477330)(0.356225,0.472095)(0.417532,0.465738)(0.489390,0.458019)(0.573615,0.448657)(0.672336,0.437337)(0.788046,0.423726)(0.923671,0.407498)(1.082637,0.388387)(1.268961,0.366249)(1.487352,0.341145)(1.743329,0.313415)(2.043360,0.283712)(2.395027,0.252972)(2.807216,0.222302)(3.290345,0.192796)(3.856620,0.165366)(4.520354,0.140619)(5.298317,0.118833)(6.210169,0.100011)(7.278954,0.083968)(8.531679,0.070416)(10.000000,0.059037)
        };
        \end{axis}
  \end{tikzpicture}
  \caption{{ $ p/n = 4 $ }}
  \label{subfig:SNR_E3}
  \end{subfigure}
  ~
 
\caption{ { 
Theoretical interpolation errors of $\tanh$ (\textbf{\BLUE blue}) and truncated linear (with $ f(t) = \max(-5,\min(5,t)) $ in \textbf{\PURPLE purple}) Transformer, for key/query weights aligned with the signal direction in solid lines: $ \w_K = \w_Q = \bmu_{\rm base} \sim \NN(\zo, \I_p/p)$ and $\bmu \propto \bmu_{\rm base}$; versus the case where both weights orthogonal to the signal in \underline{dotted} lines: $\w_K \perp \bmu_{\rm base}, \w_Q \perp \bmu_{\rm base}$, $\w_K \perp \w_Q$ and $\bmu \propto \bmu_{\rm base}$; for regularization strength $\gamma = 1$.
} }
\label{SNR_function_dimension}
\end{figure}

\subsection{Role of Attention Weight Alignment in Interpolation}

Beyond the effect of linear component of nonlinearity discussed in \Cref{subsec:importance_of_the_attention_linear_component_in_interpolation}, \Cref{theo:high_interpolation} shows that interpolation performance is also governed by the alignment between the Attention weights and the input signal $\bmu$.
In particular, the low-rank informative structure identified in \Cref{lem:linearization_K_X} depends on the interaction among $\bmu$, $\w_K$, and $\w_Q$, indicating that alignment plays a fundamental role in interpolation.

\Cref{SNR_function_dimension} evaluates this effect by varying the alignment between Attention weights (the query and key vectors $\w_Q, \w_K$) and the input signal $\bmu$. 
A clear pattern emerges: when the Attention weights are aligned with $\bmu$, the resulting interpolation error is substantially reduced compared to orthogonal case. 
This improvement holds across both nonlinearities considered: $ f(t) =\tanh(t) $ and the truncated linear function $ f(t) = \max(-5,\min(5,t)) $, and persists over a broad range of SNR levels and dimension ratios $p/n$. 


The performance gain from the weight alignment is particularly pronounced in the over-determined $p/n<1$ setting.

\section{Conclusion and Perspectives}
\label{sec:Conclusion}

In this paper, we provide a precise high-dimensional characterization of the interpolation error for nonlinear Attention on structured inputs.
We show that, although nonlinear Attention typically incurs slighter higher interpolation error than linear regression for random inputs, this disadvantage vanishes—and can even be reversed—when the input possesses structure, particularly when the Attentions weights are aligned with the underlying input signal.


A natural extension of this work is to incorporate more realistic architectural components used in practical Transformers, such as skip connections or multi-head Attention.
Another interesting direction is to go beyond the i.i.d.\@ signal-plus-noise model in \Cref{def:signal_plus_noise}. 
In real-world scenarios such as natural language processing or time series analysis, the input (tokenized) sequences typically exhibit strong temporal correlations. 
For instance, the case of linear temporal correlation has been recently studied in \citep{moniri2024Asymptotics}, though limited to linear regression model.
It would be of interest to extend our nonlinear random matrix analysis to such structured input settings.




\bibliographystyle{plainnat}
\bibliography{liao}




\newpage
\onecolumn 
\appendix  

\begin{center}
  \textbf{\Large Supplementary Material of \\On the Interpolation Error of Nonlinear Attention versus Linear Regression}
\end{center}

The technical appendices of this paper are organized as follows.
An extended discussion of related work is given in \Cref{sec:further_discussions_of_prior_efforts}.
Some auxiliary results and discussions are placed in \Cref{sec:further}.
The detailed proofs of our technical results are given in \Cref{sec:math_proof}.
Additional numerical results are provided in \Cref{sec:SM_additional_nums}.

\section{Further Discussions of Prior Efforts}
\label{sec:further_discussions_of_prior_efforts}

\paragraph{Transformers and empirical scaling laws.} 

A growing body of work has established empirical scaling laws for Transformer models with respect to data size, model size, and computational budget. 
Early studies demonstrated power-law curves between generalization performance and model size for Transformer-based LLMs~\citep{kaplan2020scaling,henighan2020scaling}, with subsequent extensions to transfer and multitask learning~\citep{hernandez2021scaling,wei2022emergent}. 
Notably, it has been shown in~\citep{hoffmann2022training} that smaller models trained on more data can outperform larger ``undertrained'' models under fixed compute budget. 
Other studies have explored the effects of overparameterization, initialization, and depth-width trade-offs in the scaling laws of Transformer-based models~\citep{bahri2024scaling,zhai2022scaling,xiao2021early}. 
Emergent phenomena and scaling transitions such as double descent~\citep{nakkiran2020deep}, in-context induction~\citep{olsson2022icl}, and phase shifts in predictability~\citep{ganguli2022predictability} have also been empirically observed. 
Investigations on Vision Transformers and instruction-tuned models~\citep{dosovitskiy2020Image,chowdhery2023palm} further support the universality of scaling behaviors across different modalities.

Our work complements these empirical findings by providing a \emph{precise} theoretical characterization on the scaling law of interpolation error of nonlinear Attention as a function of the sample-to-dimension ratio ($n/p$) and the SNR of the input data. 

\paragraph{Efficient Transformer variants and low-rank adaptation.} 

The quadratic complexity of vanilla Attention with respect to sequence length has motivated a wide range of approximation-based methods to improve computational efficiency.
Performer has replaced the Softmax nonlinearity with kernel-based random projections to achieve near-linear complexity~\citep{krzysztof2021rethinking}; Linformer has projected keys and values into a low-dimensional subspace~\citep{wang2020linformer}; Nyströmformer approximates the Attention matrix using the Nyström method~\citep{xiong2021nystromformer}; and Reformer has combined locality-sensitive hashing with reversible layers for memory savings~\citep{kitaev2020reformer}. 
In parallel, a series of works have proposed low-rank adaptation techniques for efficient fine-tuning of Transformer-based LLMs.
LoRA has introduced trainable low-rank updates to frozen weights~\citep{hu2022lora}; QLoRA has extended this idea to quantized models with minimal performance degradation~\citep{dettmers2023qlora}; LoRA-FA has improved memory efficiency via factorized updates~\citep{zhang2023lora}; and UniPELT has unified multiple parameter-efficient tuning strategies into a single framework~\citep{mao2021unipelt}. 

Motivated by these low-rank structures in computing and/or fine-tuning Transformer-based models, we assume in \Cref{ass:low-rank-weights} a full-plus-low-rank decomposition of the Attention weights, and characterizes how such structure affects the interpolation error of nonlinear Attention.

\paragraph{Theoretical understanding of DNN, LLMs, and in-context learning.} 

Recent theoretical advances in the optimization and generalization of over-parameterized deep neural networks (DNNs) have laid the groundwork for understanding the training behavior of modern large language models (LLMs).
Despite the fact that LLMs typically operate in a regime where the number of model parameters far exceeds the number of training samples, first-order methods such as stochastic gradient descent have been shown to converge reliably and generalize effectively under specific conditions~\citep{li2018learning,allen2019convergence} for DNNs. 
Notably, the ``edge of stability'' phenomenon has emerged as a key concept, capturing the peculiar yet effective optimization dynamics commonly observed during the training of DNNs and LLMs~\citep{cohen2021gradient,arora2022Understanding,wang2022Analyzing}. 
Building on these insights, a growing body of work has investigated the mechanisms underlying in-context learning (ICL). 
Transformers have been shown to approximate gradient descent steps via Attention blocks~\citep{von2023transformers,mahankali2023one}, and even to implement general-purpose learning algorithms directly from contextual input~\citep{akyurek2022learning,garg2022can}. 
Connection has also been drawn on Transformer computation and functional gradient descent~\citep{cheng2023transformers}.
Alternative viewpoints interpret ICL as a form of implicit Bayesian inference~\citep{xie2021Explanation,falck2024context}, offering probabilistic frameworks to explain generalization from prompts. 
At the mechanistic level, recent work has identified ``induction heads'' within Transformer architectures than enable token-level pattern recognition and generalization~\citep{olsson2022context}. Beyond training dynamics, LayerNorm has also been shown to influence the expressivity of Attention-based models~\citep{brody2023expressivity}. 
From a model design perspective, entropy-guided variants of the Attention have been proposed to balance stability, computational efficiency, and privacy in large-scale models~\citep{jha2025entropy}.



\section{Auxiliary Results and Further Discussions}
\label{sec:further}

In this section, we provide further discussions on possible extensions of our theoretical results.
We discuss in \Cref{rem:softmax_Attention} the connection between the entry-wise Attention in \Cref{def:Attention} to the standard Softmax Attention, in \Cref{rem:extension_beyond_rank_one} the extension of \Cref{ass:low-rank-weights} beyond the rank-one setting, and in \Cref{rem:extension_beyond_binary_signal_plus_noise} the possibility to relax the signal-plus-noise model in \Cref{def:signal_plus_noise}.

\begin{Remark}[On Softmax Attention]\label{rem:softmax_Attention}\normalfont


As already mentioned in the discussion after \Cref{def:Attention}, while \Cref{def:Attention} corresponds to entry-wise Attention (such the sigmoid or ReLU Attention~\citep{wortsman2023Replacing,ramapuram2024Theory}) instead of commonly used Softmax Attention, under the setting of Assumptions~\ref{ass:low-rank-weights}~and~\ref{ass:high-dim} and for input token drawn from the signal-plus-noise model in \Cref{def:signal_plus_noise} taking the truncated exponential function 
\begin{equation}
  f(t) = \check \exp(t), \quad \check \exp(t) \equiv \min(\exp(t), C) ,
\end{equation}
for some $C>0$ sufficiently large independent of $n,p$, leads to asymptotically the same Attention matrix $\A_\X$ as that of Softmax nonlinearity, up to a scaling factor.

Precisely, we follow the line of arguments similar to \citep{hayase2025Gaussian} and denote
\begin{equation}
  S_{ij} = \x_i^\top \W_K^\top \W_Q \x_j/\sqrt p,
\end{equation}
so that the (truncated) Softmax Attention writes
\begin{equation}
  \A_\X = \W_V \X \cdot\mathrm{Softmax}(\X) \sqrt p,\quad\text{with}\quad \sqrt p \cdot \mathrm{Softmax}(\X) = \left(  \diag \left\{ \frac{\sum_{j=1}^n \check \exp(S_{ij}) }{p \cdot \EE[\check \exp(\xi)]/c} \right\} \right)^{-1} \frac{\check \exp(S_{ij})}{\sqrt p \cdot \EE[\check \exp(\xi)]/c},
\end{equation}
for some (auxiliary) random variable $\xi \sim \NN(0,1)$.

Now, note from the proof of \Cref{lem:linearization_K_X} below in \Cref{subsec:proof_lem_linearization_K_X} that the $(i,j)$ entry of $\X^\top \W_K^\top \W_Q \X$ is given, for $i \neq j$, by
\begin{align*}
       \sqrt p S_{ij} = &\x_i^\top \W_K^\top \W_Q \x_j = \x_i^\top \x_j + \x_i^\top \w_K \w_Q^\top \x_j \\ 
       &= \underbrace{\z_i^\top \z_j}_{O(\sqrt p)} + \underbrace{y_i y_j \| \bmu \|^2 + (y_j \z_i + y_i \z_j)^\top \bmu + y_j \z_i^\top \w_K \w_Q^\top \bmu + y_i \bmu^\top \w_K \w_Q^\top \z_j + \z_i^\top \w_K \w_Q^\top \z_j + y_i y_j \bmu^\top \w_K \w_Q^\top \bmu}_{O(1)},
\end{align*}
with $\z_i^\top \z_j/\sqrt p \to \NN(0,1)$ in law as $p \to \infty$, and for $i=j$, by
\begin{align*}
    \sqrt p S_{ii} = \x_i^\top \W_K^\top \W_Q \x_i &= \| \x_i \|^2 + \x_i^\top \w_K \w_Q^\top \x_i \\ 
    &= \underbrace{\| \z_i \|^2}_{O(p)} + \underbrace{\| \bmu \| ^2 + 2 y_i \z_i^\top \bmu + y_i \z_i^\top \w_K \w_Q^\top \bmu + y_i \bmu^\top \w_K \w_Q^\top \z_i + \z_i^\top \w_K \w_Q^\top \z_i + \bmu^\top \w_K \w_Q^\top \bmu}_{O(1)}.
\end{align*}
As a consequence, we obtain
\begin{equation}
  \frac1p \sum_{j=1}^n \check \exp(S_{ij}) = \frac1p \check \exp(S_{ii}) + \frac{n}p \cdot \frac1n \sum_{j \neq i}^n \check \exp(S_{ij}) \to \frac1c \EE[ \exp(\xi) ],
\end{equation}
in probability as $n,p \to \infty$ at the same rate, so that 
\begin{equation}
  \left\| \left(  \diag \left\{ \frac{\sum_{j=1}^n \check \exp(S_{ij}) }{p \cdot \EE[\check \exp(\xi)]/c} \right\} \right)^{-1} - \I_n \right\| \to 0,
\end{equation}
in probability as $n,p \to \infty$.

Further note that $\frac{\check \exp(S_{ij})}{\sqrt p \cdot \EE[\check \exp(\xi)]/c}$ is nothing but the kernel matrix $\K_\X$ defined in \Cref{def:Attention} with truncated exponential nonlinearity $f(t) = \check \exp(t)$, up to a scaling factor that depends on $\EE[\check \exp(\xi)]/c$.
\end{Remark}

\begin{Remark}[Beyond rank-one setting in \Cref{ass:low-rank-weights}]\label{rem:extension_beyond_rank_one}\normalfont
While we consider in \Cref{ass:low-rank-weights} that the Attention weights admits a full-plus-low-rank decomposition of the form $\W_K^\top \W_Q = \I_p + \w_K \w_Q^\top$, with $\w_K \w_Q^\top$ being of rank one, it is possible to extend the analysis beyond the rank-one setting and consider a low-rank part of rank $K$ (with $K$ fixed as $n,p \to \infty$).
Notably, in that setting, the linearization result in \Cref{lem:linearization_K_X} must be modified so that the term $\U_K \bSigma_\K \V_Q^\top$ takes account of the rank-$K$ structure in the product $\W_K^\top \W_Q$.
\end{Remark}

\begin{Remark}[Extension beyond signal-plus-noise model in \Cref{def:signal_plus_noise}]\label{rem:extension_beyond_binary_signal_plus_noise}\normalfont
The signal-plus-noise model in \Cref{def:signal_plus_noise} can be extended by considering more sophisticated structure in data the signal, for example with $\X = \M \mathbf{J}^\top$ where $\M = [\bmu_1,~\ldots,~\bmu_K] \in \RR^{p \times K}$ is the matrix containing $K > 2$ signals (that, e.g., corresponds to $K$ tasks and/or mixture model having $K$ class), and $\mathbf{J} = [\mathbf{j}_1,~\ldots,~\mathbf{j}_K] \in \RR^{n \times K}$ is the canonical vector of task/class $\mathcal{C}_k \in \{1, \ldots K \}$, with $[\mathbf{j}_k]_i = 1$ if $\x_i$ belongs to task/class $\mathcal{C}_k$ and zero otherwise.
\end{Remark}

\begin{Remark}[Extension to multi-head Attention]
\label{rem:extension_to_multi-head}\normalfont
Compared to the single-head Attention module in \Cref{def:Attention}, the multi-head Attention architecture is more popularly used. 
As long as the number of heads remains constant and does not scale with the dimensions $n,p$, the linear combination-type $H$-head Attention of the form 
\begin{equation}
  \A_\X = \sum_{i=1}^H w_i \mathbf{H}_i(\X), \quad \mathbf{H}_i(\X) = \W_{V,i} \X f( \X^\top \W_{K,i}^\top \W_{Q,i} \X/\sqrt p)/\sqrt p,
\end{equation}
with weights $w_1, \ldots, w_H \in \RR$, do not introduce any fundamentally novel technical difficulty in the analysis, and we expect the main qualitative behavior and scaling laws characterized in this paper to remain valid in the multi-head setting. 
In particular, we believe that in the case of multi-task context (i.e., more that one directions to be captured from the input data, that are correlated with different targets, as discussed in \Cref{rem:extension_beyond_binary_signal_plus_noise}), multi-head Attention should be considered for the interpolation error to be efficient.
A complete treatment of the multi-head extension is left for future work.
\end{Remark}

\section{Mathematical Proofs}
\label{sec:math_proof}

In this section, we present the proofs of the technical results in this paper.
We first recall in \Cref{subsec:useful_lemmas} a few lemmas that will be used in the proofs. 
The proof of \Cref{lem:linearization_K_X} is given in \Cref{subsec:proof_lem_linearization_K_X}, the proof of \Cref{prop:DE_resovlent_noise} is given in \Cref{subsec:proof_of_prop:DE_resovlent_noise}, the proof of \Cref{theo:high_interpolation} is given in \Cref{subsec:proof_of_theo:ICM}, and finally the proof of \Cref{prop:high_interpolation_error_RR} in \Cref{subsec:proof_of_prop:ICM_RR}.

\subsection{Useful Lemmas}
\label{subsec:useful_lemmas}

\begin{Lemma}[Spectral norm of kernel random matrix,~\citep{fan2019spectral}]\label{lem:kernel_norm_control}
For a random matrix $\Z \in \RR^{p \times n}$ having i.i.d.\@ sub-gaussian entries that are symmetric in law, of zero mean and unit variance, and function $f \colon \RR \to \RR$ such that $|f(x)| \leq C_1 \exp(C_2 |x|)$ for some constants $C_1, C_2 > 0$, the random kernel matrix
\begin{equation}
   \K = f(\Z^\top \Z/\sqrt p)/\sqrt p - \diag (\cdot) \in \RR^{n \times n}, 
\end{equation}
satisfies, with high probability as $n,p \to \infty$ at the same pace, that
    \begin{enumerate}
        \item $\| \K \| = O(1)$ if $\EE_{\xi \sim \NN(0,1)}[f(\xi)] = 0$; and
        \item $\| \K \| = O(\sqrt p)$ with $\| \K - \EE[f(\xi)] \one_n \one_n^\top/ \sqrt p \| = O(1)$ otherwise.
    \end{enumerate}
\end{Lemma}

\begin{Lemma}[Matrix norm controls]\label{lem:operator_norm_control}
We have the following facts on the operator norm of matrices and Hadamard product between matrices.
\begin{enumerate}
   \item For $\A \in \RR^{n \times n}$, we have $\| \A \|_{\max} \leq \| \A  \| \leq n \| \A \|_{\max}$ with $\| \A \|_{\max} \equiv \max_{i,j} |A_{ij}|$.
   \item For $\A, \B \in \RR^{N \times n}$, we have $\| \A \odot \B \| \le \max(\sqrt n, \sqrt N) \| \A \|_{\max} \cdot \| \B \|$.
   \item If $\A \in \RR^{N \times n}$ is of rank one with $\A = \bu \bv^\top$, $\bu \in \RR^N, \bv \in \RR^n$, we have $\A \odot \B = \diag(\bu) \B \diag(\bv)$ so that
   \begin{equation}
      \| \A \odot \B \| \leq \| \bu \|_{\infty} \cdot \| \bv \|_{\infty} \cdot \| \B \|, 
   \end{equation}
   see \citep[Fact~13]{ba2022Highdimensional}. More generally, if $\A$ is of rank $K$ with $\A = \sum_{k=1}^K \bu_k \w_K^\top$, we similarly have
   \begin{align*}
      \| \A \odot \B \| &= \|  (\sum_{k=1}^K \bu_k \w_k^\top) \odot \B \| = \| \sum_{k=1}^K ( \bu_k \w_k^\top) \odot \B \|  \\ 
      &\leq \sum_{k=1}^K  \| ( \bu_k \w_k^\top) \odot \B \| \leq  \sum_{k=1}^K \| \bu_k \|_{\infty} \cdot \| \w_k\|_{\infty}  \cdot \| \B \|.
   \end{align*}
\end{enumerate}
\end{Lemma}

\subsection{Proof of Lemma~\ref{lem:linearization_K_X} }
\label{subsec:proof_lem_linearization_K_X}

Here, we present the proof of \Cref{lem:linearization_K_X} by ``linearizing'' the nonlinear kernel matrix
\begin{equation}
   \K_\X = f( \X^\top \W_K^\top \W_Q \X/\sqrt p)/\sqrt p \in \RR^{n \times n},
\end{equation}
defined in \eqref{eq:def_A_X} of \Cref{def:Attention}.

To start, note that for the binary mixture model in \Cref{def:signal_plus_noise} and under Assumptions~\ref{ass:low-rank-weights}~and~\ref{ass:high-dim}, we have $\x_i = y_i \bmu + \z_i$ and $\W_K^\top \W_Q = \I_p + \w_K \w_Q^\top$, so that for $i \neq j$,
\begin{align*}
       &\x_i^\top \W_K^\top \W_Q \x_j = \x_i^\top \x_j + \x_i^\top \w_K \w_Q^\top \x_j \\ 
       &= \underbrace{\z_i^\top \z_j}_{O(\sqrt p)} + \underbrace{y_i y_j \| \bmu \|^2 + (y_j \z_i + y_i \z_j)^\top \bmu + y_j \z_i^\top \w_K \w_Q^\top \bmu + y_i \bmu^\top \w_K \w_Q^\top \z_j + \z_i^\top \w_K \w_Q^\top \z_j + y_i y_j \bmu^\top \w_K \w_Q^\top \bmu}_{O(1)},
\end{align*}
for $y_i, y_j \in \{ \pm 1 \}$ and independent $\z_i, \z_j$ having i.i.d.\@ sub-gaussian entries of zero mean and unit variance, where we used the fact that $\max \{ \| \bmu \|, \| \w_K \|, \| \w_Q \| \} = O(1)$ under \Cref{ass:high-dim}.
Similarly, for $i=j$,
\begin{align*}
    \x_i^\top \W_K^\top \W_Q \x_i &= \| \x_i \|^2 + \x_i^\top \w_K \w_Q^\top \x_i \\ 
    &= \underbrace{\| \z_i \|^2}_{O(p)} + \underbrace{\| \bmu \| ^2 + 2 y_i \z_i^\top \bmu + y_i \z_i^\top \w_K \w_Q^\top \bmu + y_i \bmu^\top \w_K \w_Q^\top \z_i + \z_i^\top \w_K \w_Q^\top \z_i + \bmu^\top \w_K \w_Q^\top \bmu}_{O(1)},
\end{align*}
where we used the fact that for any deterministic vector $\w \in \RR^p$ of bounded norm, one has $\z_i^\top \w  \to \NN(0, \| \w \|^2)$ as $p \to \infty$
As a consequence, we can Taylor-expand the smooth function $f$ in $\K_\X$ defined in \eqref{eq:def_A_X} of \Cref{def:Attention}. 
We first treat its non-diagonal entry $(i,j)$, for $i \neq j$, as
\begin{align*}
      \sqrt p[\K_\X]_{ij} &= f (\z_i^\top \z_j/\sqrt p) + f' (\z_i^\top \z_j/\sqrt p) ( y_i y_j \| \bmu \| ^2+ (y_j \z_i + y_i \z_j)^\top \bmu + y_j \z_i^\top \w_K \w_Q^\top \bmu + y_i \bmu^\top \w_K \w_Q^\top \z_j \\
        & + \z_i^\top \w_K \w_Q^\top \z_j + y_i y_j \bmu^\top \w_K \w_Q^\top \bmu )/\sqrt p + O(p^{-1}),
\end{align*}
and for its diagonal entries as
\begin{align*}
    \sqrt p [\K_\X]_{ii} &= f(\| \z_i \|^2 /\sqrt p) + f'(\| \z_i \|^2 /\sqrt p)(\| \bmu \| ^2 + 2y_i \z_i^\top \bmu + y_i \z_i^\top \w_K \w_Q^\top \bmu + y_i \bmu^\top \w_K \w_Q^\top \z_i + \z_i^\top \w_K \w_Q^\top \z_i + \bmu^\top \w_K \w_Q^\top \bmu)/\sqrt p \\ 
    &+ O(p^{-1}).
\end{align*}
Note that under \Cref{ass:nonlinear}, one has $\lim_{t\to \infty} f(t) < \infty$ so that as $n,p \to \infty$,
\begin{equation*}
   \sqrt p [\K_\X]_{ii} = O(1).
\end{equation*}

This leads to the following spectral norm approximation of $\K_\X$ as 
\begin{align*}
   \sqrt p\K_\X &= \underbrace {f (\Z^\top \Z/\sqrt p) - \diag(\cdot)}_{O_{\| \cdot \|}(\sqrt p)} + \underbrace{f' (\Z^\top \Z/\sqrt p) \odot ( \| \bmu \|^2 \y \y^\top + \y \bmu^\top \Z + \Z^\top \bmu \y^\top + \bmu^\top \w_Q \cdot \Z^\top \w_K \y^\top + \bmu^\top \w_K \cdot \y \w_Q^\top \Z ) /\sqrt p}_{O_{\| \cdot \|}(\sqrt p)} \\
 & + \underbrace{ \bmu^\top \w_K \w_Q^\top \bmu \cdot f' (\Z^\top \Z/\sqrt p) \odot (\y \y^\top) /\sqrt p }_{O_{ \|\cdot \| }(\sqrt p)} + \underbrace{ \Z^\top \w_K \w_Q^\top \Z \odot f' (\Z^\top \Z/\sqrt p)/\sqrt p}_{O_{\| \cdot \|}(\sqrt p)} -\diag(\cdot) + O_{\| \cdot \|}(1) \\ 
 &= \underbrace {f (\Z^\top \Z/\sqrt p) - \diag(\cdot)}_{O_{\| \cdot \|}(\sqrt p)} + \underbrace{ a_1 ( ( \| \bmu \|^2 + \bmu^\top \w_K \w_Q^\top \bmu ) \y \y^\top + \y \bmu^\top \Z + \Z^\top \bmu \y^\top + \bmu^\top \w_Q \cdot \Z^\top \w_K \y^\top + \bmu^\top \w_K \cdot \y \w_Q^\top \Z ) /\sqrt p}_{O_{\| \cdot \|}(\sqrt p)} \\
 & + \underbrace{ a_1 \Z^\top \w_K \w_Q^\top \Z /\sqrt p }_{O_{\| \cdot \|}(\sqrt p)} -\diag(\cdot) + O_{\| \cdot \|}(1),
\end{align*}
where we used the fact that under \Cref{ass:nonlinear} for $\EE[f(\xi)] = 0$ and $\EE[f'(\xi)] = a_1 \neq 0$, it follows from \Cref{lem:kernel_norm_control} that $f (\Z^\top \Z/\sqrt p) - \diag(\cdot) = O_{\| \cdot \|} (\sqrt p)$ and $f' (\Z^\top \Z/\sqrt p) = \EE[f'(\xi)] \one_n \one_n^\top + O_{\| \cdot \|} (\sqrt p)$, and then Item~3 of \Cref{lem:operator_norm_control}.

Let $\U_K = [\y,~\Z^\top \bmu,~\Z^\top \w_K]/\sqrt p\in \RR^{n \times 3}$, $\V_Q = [\y,~\Z^\top \bmu,~\Z^\top \w_Q]/\sqrt p \in \RR^{n \times 3}$, and
\begin{equation}
   \bSigma_\K = a_1 \begin{bmatrix} \| \bmu \|^2 + \bmu^\top \w_K \w_Q^\top \bmu & 1 & \bmu^\top \w_K \\ 1 & 0 & 0 \\ \bmu^\top \w_Q & 0 & 1  \end{bmatrix} \in \RR^{3 \times 3}.
\end{equation}
Putting everything in matrix form, we conclude the proof of \Cref{lem:linearization_K_X}.

\subsection{ Proof of Proposition~\ref{prop:DE_resovlent_noise} }
\label{subsec:proof_of_prop:DE_resovlent_noise}

For the sake of presentation, we provide here the derivation of the Deterministic Equivalent for the resolvent\footnote{Note that this is not the same $\Q(\gamma)$ as in \eqref{def:interpolation_error} of \Cref{def:interpolation_error}.
It is used here for the sake of notational convenience and only within the proof of \Cref{prop:DE_resovlent_noise}.}
\begin{equation}\label{eq:def_Q_SM}
   \Q(\gamma) = \left( \frac{1}{p}\K \Z^\top \Z \K + \frac{\gamma}{c} \I_n \right)^{-1}, \quad \gamma > 0,
\end{equation}
where we denote, with a slight abuse of notation that $\K = \K_N = f(\Z^\top \Z/\sqrt p)/ \sqrt p - \diag(\cdot)$ for the noise-only kernel matrix $\K_N$ defined in \Cref{lem:linearization_K_X}.
The result in \Cref{prop:DE_resovlent_noise} can be obtained with a simple scaling.

Consider the following normalized traces involving $\Q (\gamma)$ defined in \eqref{eq:def_Q_SM}:
\begin{equation*}
   \frac1n \tr \Q(\gamma), \quad \frac{1}{p} \tr(\Q(\gamma) \K), \quad \frac{1}{p} \tr(\Q(\gamma) \K \cdot \Z^\top \Z/p) \quad  \frac{1}{p} \tr(\K \Q(\gamma) \K), \quad \frac{1}{p} \tr( \Z^\top \Z/p \cdot \K \Q(\gamma) \K \cdot \Z^\top \Z/p),
\end{equation*}
for which we shall subsequently prove that
\begin{equation}\label{eq:def_m_and_deltas_SM}
\begin{aligned}
   &\frac1n \tr \Q(\gamma) - m(\gamma) \to 0, \quad  \frac{1}{p} \tr(\Q(\gamma) \K) - \delta_1(\gamma) \to 0, \quad \frac{1}{p} \tr(\Q(\gamma) \K \Z^\top \Z/p) - \delta_2(\gamma) \to 0, \\ 
    &\frac{1}{p} \tr(\K \Q(\gamma) \K) - \delta_3(\gamma) \to 0, \quad \frac{1}{p} \tr( \Z^\top \Z/p \cdot \K \Q(\gamma) \K \cdot \Z^\top \Z/p) - \delta_4(\gamma) \to 0,
\end{aligned}
\end{equation}
in probability as $n,p \to \infty$, where $m(\gamma)$ and $\delta_1(\gamma), \delta_2(\gamma), \delta_3(\gamma), \delta_4(\gamma)$ are Stieltjes transforms satisfying the following self-consistent system of equations
\begin{equation}\label{eq:full_expression}
   \left\{
    \begin{array}{ll}
    m(\gamma) &= \left(\frac{\gamma}{c} + \frac{\nu}{c} +  \frac{a_1^2}{c^2} - \bv^\top  \mathbf{T}(\gamma) \bv \right)^{-1}\\ 
    c \delta_1(\gamma) &=  - m(\gamma) \bv^\top \mathbf{T}(\gamma) \bv_1 \\
    c \delta_2(\gamma) 
    &= \bv_2^\top \mathbf{T}(\gamma) \bv_1 + c \delta_1(\gamma) \left(1 - \bv_2^\top \mathbf{T}(\gamma) \bv \right) \\ 
    c \delta_3(\gamma) 
    &= \bv_1^\top \mathbf{T}(\gamma) \bv_1 + \frac{ c^2 \delta_1^2(\gamma) }{m(\gamma)}  \\ 
    c \delta_4(\gamma) &= \bv_4^\top \mathbf{T}(\gamma) \bv_4 + m(\gamma) \left(\bv_4^\top \mathbf{T}(\gamma) \bv - \frac{a_1}{c} \right)^2
    \end{array}
  \right.
\end{equation}
where we denote
\begin{equation}\label{eq:def_T_SM}
   \mathbf{T}(\gamma) \equiv \bDelta_0(\gamma) (\I_6 + \bLambda_0 \bDelta_0(\gamma) )^{-1} \in \RR^{6 \times 6},
\end{equation}
that is symmetric, for
\begin{equation}\label{eq:def_bDelta_0_SM}
   \bDelta_0(\gamma) \equiv \begin{bmatrix}
        \frac{m(\gamma)}{c} & \frac{a_1}{c} m(\gamma) & \delta_1(\gamma) & a_1 \delta_1(\gamma) & \delta_2(\gamma) & a_1 \delta_2(\gamma) \\
        \frac{a_1}{c} m(\gamma)& \frac{\nu}{c} m(\gamma) & a_1 \delta_1(\gamma) & \nu \delta_1(\gamma) & a_1 \delta_2(\gamma) & \nu \delta_2(\gamma) \\
        \delta_1(\gamma) & a_1 \delta_1(\gamma) & \delta_3(\gamma) & a_1 \delta_3(\gamma) & \frac{1}{c}(1 - \frac{\gamma}{c} m(\gamma)) & \frac{a_1}{c}(1 - \frac{\gamma}{c} m(\gamma)) \\
        a_1 \delta_1(\gamma) & \nu \delta_1(\gamma) & a_1 \delta_3(\gamma) & \nu \delta_3(\gamma) & \frac{a_1}{c}(1 - \frac{\gamma}{c} m(\gamma)) & \frac{\nu}{c}(1 - \frac{\gamma}{c} m(\gamma)) \\
        \delta_2(\gamma) & a_1 \delta_2(\gamma) & \frac{1}{c}(1 - \frac{\gamma}{c} m(\gamma)) & \frac{a_1}{c}(1 - \frac{\gamma}{c} m(\gamma)) & \delta_4(\gamma) & a_1 \delta_4(\gamma) \\ 
        a_1 \delta_2(\gamma) & \nu \delta_2(\gamma) & \frac{a_1}{c}(1 - \frac{\gamma}{c} m(\gamma)) & \frac{\nu}{c}(1 - \frac{\gamma}{c} m(\gamma)) & a_1 \delta_4(\gamma) & \nu \delta_4(\gamma)
    \end{bmatrix} \in \RR^{6 \times 6},
\end{equation}
and 
\begin{equation}\label{eq:def_bLambda_0_SM}
    \bLambda_0 = \begin{bmatrix}
        \frac{a_1^2}{c^2} (c+1) &  a_1/c & a_1/c & 0 & a_1 & 0 \\
         a_1/c & 1 & 1 & 0 & 0 & 0 \\
         a_1/c & 1 & 1 & 0 & 0 & 0 \\
        0 & 0 & 0 & 0 & 0 & 0 \\
        a_1 & 0 & 0 & 0 & 0 & 0 \\
        0 & 0 & 0 & 0 & 0 & 0
    \end{bmatrix} \in \RR^{6 \times 6},
\end{equation}
that are both symmetric, as well as 
\begin{equation}\label{eq:def_vs}
\left\{
   \begin{array}{ll}
   \bv = \begin{bmatrix} \frac{a_1^2}{c^2} (1+c) & \frac{a_1}c & \frac{a_1}c & 0 & 0 & 1 \end{bmatrix}^\top \in \RR^6 \\ 
   \bv_1 = \begin{bmatrix}  0 & 1 & 0 & 0 & 0 & 0 \end{bmatrix}^\top \in \RR^6\\ 
   \bv_2 = \begin{bmatrix}  1 & 0 & 0 & 0 & 0 & 0 \end{bmatrix}^\top \in \RR^6 \\ 
   \bv_4 = \begin{bmatrix} \frac{a_1}c & 1 & 1 & 0 & 0 & 0 \end{bmatrix}^\top \in \RR^6.
   \end{array}
  \right.
\end{equation}

\subsubsection{Preliminaries}

First, let us introduce some notations and preliminary results that will be used in the proof of \Cref{prop:DE_resovlent_noise}.

Following \citep[Section~4.3.3]{couillet2022RMT4ML}, we can decompose, up to permutation, the nonlinear kernel matrix $ \K$ as
\begin{equation}\label{eq:def_K_i}
   \K = \begin{bmatrix} \K_{-i} & f(\Z_{-i}^\top \z_i/\sqrt p)/\sqrt p \\ f(\z_i^\top \Z_{-i}/\sqrt p)/\sqrt p & 0 \end{bmatrix} = \begin{bmatrix} \K_{-i} & f(\balpha_{-i})/\sqrt p \\ f(\balpha_{-i})^\top/\sqrt p & 0 \end{bmatrix},
\end{equation}
where we denote $\K_{-i} \equiv f(\Z_{-i}^\top \Z_{-i}/\sqrt p)/\sqrt p - \diag(\cdot) \in \RR^{(n-1) \times (n-1)}$,
\begin{equation}\label{eq:balpha_SM}
   \balpha_{-i} \equiv \Z_{-i}^\top \z_i/\sqrt p \in \RR^{n -1},
\end{equation}
for $\Z_{-i} \in \RR^{p \times (n-1)}$ the sub-gaussian random matrix $\Z$ with its $i$th column removed, and $ \z_i \in \RR^p$ is the $i$th column of $ \Z $.
Note that by CLT and in the large $p$ limit, the random vector $\balpha_{-i}$ is asymptotically Gaussian $\NN(\zo, \I_{n-1})$.

Denote the shortcut $ \check \Z = \frac{1}{p} \Z^\top \Z $, we can similarly decompose $\check \Z$ as
\begin{equation}\label{eq:def_check_Z_i}
    \check \Z = \frac{1}{p} \Z^\top \Z = 
    \begin{bmatrix}
        \Z_{-i}^\top \Z_{-i} / p & \balpha_{-i} / \sqrt{p} \\
        \balpha_{-i}^\top / \sqrt{p} & 1 \\
    \end{bmatrix} + O_{\| \cdot \|} (p^{-1/2}) =\begin{bmatrix}
        \check \Z_{-i} & \balpha_{-i} / \sqrt{p} \\
        \balpha_{-i}^\top / \sqrt{p} & 1 \\
    \end{bmatrix} + O_{\| \cdot \|} (p^{-1/2}),
\end{equation}
where the $O_{\| \cdot \|} (p^{-1/2})$ error term is due to the approximation $\| \z_i \|^2/p = 1 + O(p^{-1/2})$ with a CLT argument.

Note that, by the decomposition $f(x) = a_1 x + f_{>1}(x)$ for $a_1$ the first Hermite polynomial of $f$ as defined in \Cref{ass:nonlinear}, we have
\begin{equation}
   f(\balpha_{-i}) = a_1\balpha_{-i} + f_{>1}(\balpha_{-i}),
\end{equation}
and, for any $\A \in \RR^{(n-1) \times (n-1)}$ independent of $\balpha_{-i}$,
\begin{equation}
   \frac1p f(\balpha_{-i})^\top \A \balpha_{-i} = a_1 \cdot \frac1p \tr \A + o(1), \quad \frac1p f(\balpha_{-i})^\top \A f(\balpha_{-i})= \nu \cdot \frac1p \tr \A + o(1).
\end{equation}
In particular, for $\A = \I_{n-1}$, we get $\frac1p f(\balpha_{-i})^\top \A \balpha_{-i} = \frac{a_1}c + o(1)$ and $\frac1p f(\balpha_{-i})^\top \A f(\balpha_{-i})=  \frac{\nu}c + o(1)$, where $c= \lim p/n $ as in \Cref{ass:high-dim}.

Further denote 
\begin{equation}\label{eq:def_U_SM}
   \U_0 = \begin{bmatrix}
        \balpha_{-i} & f(\balpha_{-i}) & \K_{-i}^\perp \balpha_{-i} & \K_{-i}^\perp f(\balpha_{-i}) & \K_{-i}^\perp \check{\Z}_{-i}^\perp \balpha_{-i} & \K_{-i}^\perp \check{\Z}_{-i}^\perp f(\balpha_{-i})
    \end{bmatrix}/\sqrt{p} \in \RR^{ (n-1) \times 6},
\end{equation}
as well as
\begin{equation}\label{eq:def_K_i_and_check_Z_i_SM}
   \K_{-i}^\perp = \{ f( (\z_j^\perp)^\top \z_k^\perp/\sqrt p ) \}_{j,k \neq i} /\sqrt p - \diag(\cdot)\in \RR^{(n-1) \times (n-1)}, \quad \check \Z_{-i}^\perp \equiv \{ (\z_j^\perp)^\top \z_k^\perp/p \}_{j,k \neq i}\in \RR^{(n-1) \times (n-1)},
\end{equation}
where, for given $i$,
\begin{equation}\label{eq:def_z_perp_SM}
   \z_j^\perp \equiv \z_j - \frac{\z_i^\top \z_j }{ \| \z_i \|} \frac{\z_i}{\| \z_i \|}, 
\end{equation}
that is orthogonal to and asymptotically independent of $\z_i^\top \z_j/\| \z_i \| \approx [\balpha_{-i}]_j$ in \eqref{eq:balpha_SM}. 
For the two ``leave-one-out'' matrices $\K_{-i}$ and $\check \Z_{-i}$ defined in \eqref{eq:def_K_i}~and~\eqref{eq:def_check_Z_i}, we have the following result.

\begin{Lemma}[Approximations of $\K_{-i}$ and $\check \Z_{-i}$, {\citep[Section 4.3.3]{couillet2022RMT4ML}}]\label{lem:approx_K_i_and_check_Z_i}
For $\K_{-i}$ and $\check{\Z}_{-i}$ defined respectively in \eqref{eq:def_K_i}~and~\eqref{eq:def_check_Z_i}, we have the following approximations in spectral norm:
\begin{enumerate}
   \item $\K_{-i} = \K_{-i}^\perp + \frac{a_1}{p} \balpha_{-i} \balpha_{-i}^\top + o_{\| \cdot \|}(1) $; and
   \item $\check{\Z}_{-i} = \check{\Z}_{-i}^\perp + \frac{1}{p} \balpha_{-i} \balpha_{-i}^\top + o_{\| \cdot \|}(1) $;
\end{enumerate}
for $\K_{-i}^\perp$ and $\check{\Z}_{-i}^\perp$ defined in \eqref{eq:def_K_i_and_check_Z_i_SM} that is asymptotically independent of $\balpha_{-i}$ in \eqref{eq:balpha_SM}.
\end{Lemma}
With these preliminary results at hand, we are ready to derive a Deterministic Equivalent for $\Q(\gamma)$ defined in \eqref{eq:def_Q_SM}.

\subsubsection{Self-consistent equation for $m(z)$}
\label{ssub:self_consistent_equation_for_m}

Here we present the derivation for the Deterministic Equivalent (of the diagonal entries) of $\Q(z)$.
With the block decomposition of $\K$ and $\check \Z$ in \eqref{eq:def_K_i}~and~\eqref{eq:def_check_Z_i}, we obtain for $\Q(\gamma) \equiv \Q$ in \eqref{eq:def_Q_SM} (where we drop the argument $\gamma$) that
\begin{equation}
    \Q^{-1} \equiv \K \check \Z \K + \frac{\gamma}{c} \I_n = \begin{bmatrix}
       [\Q^{-1}]_{11} & [\Q^{-1}]_{12} \equiv [\Q^{-1}]_{21}^\top \\ 
       [\Q^{-1}]_{21} & [\Q^{-1}]_{22}
    \end{bmatrix} \label{eq:def_inv_Q_SM}
\end{equation}
with 
\begin{align*}
   [\Q^{-1}]_{11} &\equiv \K_{-i} \check \Z_{-i} \K_{-i} + \frac{1}{p} f(\balpha_{-i}) \balpha_{-i}^\top \K_{-i} +\frac{1}{p} \K_{-i} \balpha_{-i} f(\balpha_{-i})^\top + \frac{1}{p} f(\balpha_{-i}) f(\balpha_{-i})^\top + \frac{\gamma}{c} \I_{n-1}, \\ 
   [\Q^{-1}]_{21} &\equiv \frac{1}{\sqrt{p}} f(\balpha_{-i})^\top \check \Z_{-i} \K_{-i} +  \frac{a_1}{c} \frac{f(\balpha_{-i})^\top}{\sqrt{p}}, \\ 
   [\Q^{-1}]_{22} &\equiv \frac{1}{p} f(\balpha_{-i})^\top \check \Z_{-i} f(\balpha_{-i}) + \frac{\gamma}{c},
\end{align*}
for which we have, per \Cref{lem:approx_K_i_and_check_Z_i},
\begin{enumerate}
   \item $\K_{-i} = \K_{-i}^\perp + \frac{a_1}{p} \balpha_{-i} \balpha_{-i}^\top + o_{\| \cdot \|}(1) $; and
   \item $\check{\Z}_{-i} = \check{\Z}_{-i}^\perp + \frac{1}{p} \balpha_{-i} \balpha_{-i}^\top + o_{\| \cdot \|}(1) $;
\end{enumerate}
where $ \K_{-i}^\perp $ and $ \check{\Z}_{-i}^\perp $ as defined in \eqref{eq:def_K_i_and_check_Z_i_SM} that are both asymptotically independent of $ \balpha_{-i} $. 
This allows for the first approximation of $[\Q^{-1}]_{22}$ as
\begin{align}
   [\Q^{-1}]_{22} &= \frac{1}{p} f(\balpha_{-i})^\top \check{\Z}_{-i} f(\balpha_{-i}) + \frac{\gamma}{c} = \frac1{p} f(\balpha_{-i})^\top \left( \frac1p \balpha_{-i} \balpha_{-i}^\top + \check{\Z}_{-i}^{\perp} \right) f(\balpha_{-i}) + \frac{\gamma}{c} + o(1) \nonumber \\ 
   &=  \frac{a_1^2}{c^2} +  \frac{\nu}{c} + \frac{\gamma}{c} + o(1). \label{eq:approx_Q_inv_22}
\end{align}

Let 
\begin{equation}\label{eq:def_Q_i_perp}
   \Q_{-i}^\perp = \left( \K_{-i}^\perp \check{\Z}_{-i}^\perp \K_{-i}^\perp + \frac{\gamma}{c} \I_{n-1} \right)^{-1} \in \RR^{ (n-1) \times (n-1) },
\end{equation}
for $ \K_{-i}^\perp, \check{\Z}_{-i}^\perp $ defined in \eqref{eq:def_K_i_and_check_Z_i_SM}, so that $\Q_{-i}^\perp $ is asymptotically independent of $\balpha_{-i}$ satisfying $ \tr (\Q - \Q_{-i}^\perp) = O(1)$. 
We have the following approximation.

\begin{Lemma}[Approximation of $\U_0^\top \Q_{-i}^\perp \U_0$]\label{lem:approx_U_Q_perp_U}
For $\Q_{-i}^\perp \in \RR^{(n-1) \times (n-1)}$ as in \eqref{eq:def_Q_i_perp} and $\U_0 \in \RR^{(n-1) \times 6}$ in \eqref{eq:def_U_SM}, we have
\begin{equation}
   \U_0^\top \Q_{-i}^\perp \U_0 = \bDelta_0(\gamma) + o_{\| \cdot \|}(1),
\end{equation}
with 
\begin{equation}
   \bDelta_0(z) \equiv \begin{bmatrix}
        \frac{m(\gamma)}{c} & \frac{a_1}{c} m(\gamma) & \delta_1(\gamma) & a_1 \delta_1(\gamma) & \delta_2(\gamma) & a_1 \delta_2(\gamma) \\
        \frac{a_1}{c} m(\gamma)& \frac{\nu}{c} m(\gamma) & a_1 \delta_1(\gamma) & \nu \delta_1(\gamma) & a_1 \delta_2(\gamma) & \nu \delta_2(\gamma) \\
        \delta_1(\gamma) & a_1 \delta_1(\gamma) & \delta_3(\gamma) & a_1 \delta_3(\gamma) & \frac{1}{c}(1 - \frac{\gamma}{c} m(\gamma)) & \frac{a_1}{c}(1 - \frac{\gamma}{c} m(\gamma)) \\
        a_1 \delta_1(\gamma) & \nu \delta_1(\gamma) & a_1 \delta_3(\gamma) & \nu \delta_3(\gamma) & \frac{a_1}{c}(1 - \frac{\gamma}{c} m(\gamma)) & \frac{\nu}{c}(1 - \frac{\gamma}{c} m(\gamma)) \\
        \delta_2(\gamma) & a_1 \delta_2(\gamma) & \frac{1}{c}(1 - \frac{\gamma}{c} m(\gamma)) & \frac{a_1}{c}(1 - \frac{\gamma}{c} m(\gamma)) & \delta_4(\gamma) & a_1 \delta_4(\gamma) \\ 
        a_1 \delta_2(\gamma) & \nu \delta_2(\gamma) & \frac{a_1}{c}(1 - \frac{\gamma}{c} m(\gamma)) & \frac{\nu}{c}(1 - \frac{\gamma}{c} m(\gamma)) & a_1 \delta_4(\gamma) & \nu \delta_4(\gamma)
    \end{bmatrix} \in \RR^{6 \times 6},
\end{equation}
as in \eqref{eq:def_bDelta_0_SM}, for $m(\gamma), \delta_1(\gamma), \delta_2(\gamma), \delta_3(\gamma), \delta_4(\gamma)$ as defined in \eqref{eq:def_m_and_deltas_SM} and \eqref{eq:full_expression}.
\end{Lemma}
\begin{proof}[Proof of \Cref{lem:approx_U_Q_perp_U}]
Since $\Q_{-i}^\perp$ is asymptotically independent of
\begin{equation}
   \U_0 = \begin{bmatrix}
        \balpha_{-i} & f(\balpha_{-i}) & \K_{-i}^\perp \balpha_{-i} & \K_{-i}^\perp f(\balpha_{-i}) & \K_{-i}^\perp \check{\Z}_{-i}^\perp \balpha_{-i} & \K_{-i}^\perp \check{\Z}_{-i}^\perp f(\balpha_{-i})
\end{bmatrix}/\sqrt{p} \in \RR^{ (n-1) \times 6},
\end{equation}
we obtain
\begin{align*}
   \U_0^\top \Q_{-i}^\perp \U_0 &= \begin{bmatrix}
        \frac1p \tr \Q & a_1 \frac1p \tr \Q & \frac1p \tr (\Q \K) & a_1 \frac1p \tr (\Q \K) & \frac1p \tr (\Q \K \check \Z) & a_1 \frac1p \tr (\Q \K \check \Z) \\
        a_1 \frac1p \tr \Q & \nu \frac1p \tr \Q & a_1 \frac1p \tr (\Q \K) & \nu \frac1p \tr (\Q \K) & a_1 \frac1p \tr (\Q \K \check \Z) & \nu \frac1p \tr (\Q \K \check \Z) \\
        \frac1p \tr (\Q \K) & a_1 \frac1p \tr (\Q \K) & \frac1p \tr (\K \Q \K) & a_1 \tr (\K \Q \K) & \frac1p \tr (\K \Q \K \check \Z) & a_1 \frac1p \tr (\K \Q \K \check \Z) \\
        a_1 \frac1p \tr (\Q \K) & \nu \frac1p \tr (\Q \K) & a_1 \frac1p \tr (\K \Q \K) & \nu \tr (\K \Q \K) & a_1 \frac1p \tr (\K \Q \K \check \Z) & \nu \frac1p \tr (\K \Q \K \check \Z) \\
        \frac1p \tr (\check \Z \K \Q) & a_1 \frac1p \tr (\check \Z \K \Q) & \frac1p \tr (\K \Q \K \check \Z) & a_1 \frac1p \tr (\K \Q \K \check \Z) & \frac1p \tr (\check \Z \K \Q \K \check \Z) & a_1 \frac1p \tr (\check \Z \K \Q \K \check \Z) \\
        a_1 \frac1p \tr (\check \Z \K \Q) & \nu \frac1p \tr (\check \Z \K \Q) & a_1 \frac1p \tr (\K \Q \K \check \Z) & \nu \frac1p \tr (\K \Q \K \check \Z) & a_1 \frac1p \tr (\check \Z \K \Q \K \check \Z) & \nu \frac1p \tr (\check \Z \K \Q \K \check \Z)
    \end{bmatrix}  + o_{\| \cdot \|}(1) \\ 
    &= \bDelta_0(\gamma) + o_{\| \cdot \|}(1),
\end{align*}
where we recall from \eqref{eq:def_m_and_deltas_SM} that
\begin{align*}
   m(\gamma) &= \frac1n \tr \Q(\gamma) + o(1) = \frac1n \tr \Q_{-i}^\perp(\gamma) + o(1) \\ 
   \delta_1(\gamma) &= \frac{1}{p} \tr(\Q(\gamma) \K) + o(1) = \frac1p \tr \left(\Q_{-i}^\perp(\gamma) \K_{-i}^\perp \right) + o(1) \\ 
   \delta_2(\gamma) &= \frac{1}{p} \tr(\Q(\gamma) \K \check \Z) + o(1) = \frac{1}{p} \tr(\Q_{-i}^\perp(\gamma) \K_{-i}^\perp \check \Z_{-i}^\perp) + o(1) \\ 
   \delta_3(\gamma) &= \frac{1}{p} \tr(\K \Q(\gamma) \K) + o(1) = \frac{1}{p} \tr(\K_{-i}^\perp \Q_{-i}^\perp(\gamma) \K_{-i}^\perp) + o(1) \\ 
   \delta_4(\gamma) &= \frac{1}{p} \tr(\check \Z \K \Q(\gamma) \K \check \Z) + o(1) = \frac{1}{p} \tr(\check \Z_{-i}^\perp \K_{-i}^\perp \Q_{-i}^\perp(\gamma) \K_{-i}^\perp \check \Z_{-i}^\perp) + o(1),
\end{align*}
and we use the fact that by \citep[Lemma~2.6]{silverstein1995empirical} and \citep[Lemma~2.9]{couillet2022RMT4ML}, when evaluating normalized traces forms as in \eqref{eq:def_m_and_deltas_SM} for $n,p$ large, we can ignore terms of finite rank inside the trace, by adding an error term $o(1)$ with high probability, as well as $\tr(\tilde \K \Q \tilde \K \check{\Z}) = \tr(\Q \tilde \K \check{\Z} \tilde \K) = \tr[\Q (\Q^{-1} - \gamma \I_n )] = \tr[\I_n -\gamma \Q] = \tr(\I_n) -\gamma \tr(\Q) = n - n \gamma m(\gamma) = n (1 - \gamma m(\gamma))$,
This concludes the proof of \Cref{lem:approx_U_Q_perp_U}.
\end{proof}

Our objective is to compute the $(i,i)$th diagonal entries of the inverse $\Q = (\K \check{\Z} \K + \gamma \I_n)^{-1} $. 
Using the block inversion lemma, we get
\begin{equation}\label{eq:Q_ii}
      [\Q]_{ii} = \left( [\Q^{-1}]_{22} - [\Q^{-1}]_{21} ([\Q^{-1}]_{11})^{-1} [\Q^{-1}]_{12} \right)^{-1},
\end{equation}
and 
\begin{equation}
   \Q = \begin{bmatrix}  ( [\Q^{-1}]_{11} - [\Q^{-1}]_{12} ([\Q^{-1}]_{22})^{-1} [\Q^{-1}]_{21} )^{-1} & - ([\Q^{-1}]_{11})^{-1} [\Q^{-1}]_{12} [\Q]_{ii} \\ - [\Q]_{ii} [\Q^{-1}]_{21} ([\Q^{-1}]_{11})^{-1} & [\Q]_{ii} \end{bmatrix}.
\end{equation}

We start with the inverse $([\Q^{-1}]_{11})^{-1}$, for which we have the following result.
\begin{Lemma}[Approximation of {$([\Q^{-1}]_{11})^{-1}$}]\label{lem:approx_inv_Q_11}
For $[\Q^{-1}]_{11} \in \RR^{(n-1) \times (n-1)}$ defined in \eqref{eq:def_inv_Q_SM}, we have
\begin{equation}
   ([\Q^{-1}]_{11})^{-1} = \Q_{-i}^\perp - \Q_{-i}^\perp \U_0 \left( \I_6 + \bLambda_0 \bDelta_0(\gamma) \right)^{-1} \bLambda_0 \U_0^\top \Q_{-i}^\perp + o_{\| \cdot \|}(1),
\end{equation}
where we recall $\bDelta_0(\gamma) \in \RR^{6 \times 6} $ as in \eqref{eq:def_bDelta_0_SM}, $\U_0 \in \RR^{(n-1) \times 6}$ as defined in \eqref{eq:def_U_SM}, and
\begin{equation}
    \bLambda_0 = \begin{bmatrix}
        \frac{a_1^2}{c^2} (c+1) &  a_1/c & a_1/c & 0 & a_1 & 0 \\
         a_1/c & 1 & 1 & 0 & 0 & 0 \\
         a_1/c & 1 & 1 & 0 & 0 & 0 \\
        0 & 0 & 0 & 0 & 0 & 0 \\
        a_1 & 0 & 0 & 0 & 0 & 0 \\
        0 & 0 & 0 & 0 & 0 & 0
    \end{bmatrix} \in \RR^{6 \times 6},
\end{equation}
as in \eqref{eq:def_bLambda_0_SM}.
\end{Lemma}
\begin{proof}[Proof of \Cref{lem:approx_inv_Q_11}]
Per its definition in \eqref{eq:def_inv_Q_SM}, we have 
\begin{align*}
       [\Q^{-1}]_{11} &\equiv \K_{-i} \check \Z_{-i} \K_{-i} + \frac{1}{p} f(\balpha_{-i}) \balpha_{-i}^\top \K_{-i} +\frac{1}{p} \K_{-i} \balpha_{-i} f(\balpha_{-i})^\top + \frac{1}{p} f(\balpha_{-i}) f(\balpha_{-i})^\top + \frac{\gamma}{c} \I_{n-1} \\ 
       &= \K_{-i}^\perp \check{\Z}_{-i}^\perp \K_{-i}^\perp + \U_0 \bLambda_0 \U_0^\top + \frac{\gamma}{c} \I_{n-1} + o_{\| \cdot \|}(1),
\end{align*}
for $\U_0, \bLambda$ defined in \eqref{eq:def_bLambda_0_SM}, $ \K_{-i}^\perp, \check{\Z}_{-i}^\perp $ defined in \eqref{eq:def_K_i_and_check_Z_i_SM} such that $\K_{-i} = \K_{-i}^\perp + \frac{a_1}{p} \balpha_{-i} \balpha_{-i}^\top + o_{\| \cdot \|}(1) $ and $\check{\Z}_{-i} = \check{\Z}_{-i}^\perp + \frac{1}{p} \balpha_{-i} \balpha_{-i}^\top + o_{\| \cdot \|}(1) $ by \Cref{lem:approx_K_i_and_check_Z_i}.
As such, by Woodbury identity,
\begin{align*}
   ([\Q^{-1}]_{11})^{-1} &= \left( \K_{-i}^\perp \check{\Z}_{-i}^\perp \K_{-i}^\perp + \frac{\gamma}{c} \I_{n-1} + \U_0 \bLambda \U_0^\top \right)^{-1} + o_{\| \cdot \|}(1) \\
   &= \Q_{-i}^\perp - \Q_{-i}^\perp \U_0 \left( \I_6 + \bLambda_0 \U_0^\top \Q_{-i}^\perp \U_0 \right)^{-1} \bLambda_0 \U_0^\top \Q_{-i}^\perp + o_{\| \cdot \|}(1),
\end{align*}
for $\Q_{-i}^\perp$ defined in \eqref{eq:def_Q_i_perp}.
Using \Cref{lem:approx_U_Q_perp_U} to approximate $\U_0^\top \Q_{-i}^\perp \U_0 = \bDelta_0(\gamma) + o_{\| \cdot \|}(1)$, we conclude the proof of \Cref{lem:approx_inv_Q_11}.
\end{proof}

With Lemmas~\ref{lem:approx_U_Q_perp_U}~and~\ref{lem:approx_inv_Q_11}, we get the following (block-wise) approximation for $\Q$.
\begin{Lemma}[Block approximation of $\Q$]\label{lem:approx_Q}
We have
\begin{equation}
   \Q = \begin{bmatrix}  
   \Q_{-i}^\perp - \Q_{-i}^\perp \U_0 ( \I_6 + \bLambda_1(\gamma) \bDelta_0(\gamma) )^{-1} \bLambda_1(\gamma) \U_0^\top \Q_{-i}^\perp & -m(\gamma) \Q_{-i}^\perp \U_0 \left( \I_6 + \bLambda_0 \bDelta_0(\gamma) \right)^{-1} \bv \\ 
   - m(\gamma) \bv^\top \left( \I_6 + \bDelta_0(\gamma) \bLambda_0 \right)^{-1} \U_0^\top \Q_{-i}^\perp & [\Q]_{ii} 
   \end{bmatrix} + o_{\| \cdot \|}(1),
\end{equation}
where we recall $\bDelta_0(\gamma) \in \CC^{6 \times 6} $ as in \eqref{eq:def_bDelta_0_SM}, $\bLambda_0$ as in \eqref{eq:def_bLambda_0_SM}, $\U_0 \in \RR^{(n-1) \times 6}$ as defined in \eqref{eq:def_U_SM}, and
\begin{equation}\label{eq:def_U_bLambda_1_SM}
    \bLambda_1(\gamma) = \bLambda_0 - \left(\frac{a_1^2}{c^2} +  \frac{\nu}{c} + \frac{\gamma}{c} \right)^{-1} \bv \bv^\top \in \RR^{6 \times 6}, 
    \quad \bv^\top = \begin{bmatrix}  \frac{a_1^2}{c^2} (1+c) & \frac{a_1}c & \frac{a_1}c & 0 & 0 & 1 \end{bmatrix} \in \RR^6,
\end{equation}
as in \eqref{eq:def_vs}.
We also have, by \eqref{eq:def_U_bLambda_1_SM} and Sherman--Morrison identity that 
\begin{align}
   &\bDelta_0(\gamma) ( \I_6 +  \bLambda_1(\gamma) \bDelta_0(\gamma) )^{-1} = \bDelta_0(\gamma) \left( \I_6 + \bLambda_0 \bDelta_0(\gamma)  - \left(\frac{a_1^2}{c^2} +  \frac{\nu}{c} + \frac{\gamma}{c} \right)^{-1}  \bv \bv^\top \bDelta_0(\gamma) \right)^{-1} \nonumber \\ 
   &= \bDelta_0(\gamma) (\I_6 + \bLambda_0 \bDelta_0(\gamma)  )^{-1} + \frac{ \bDelta_0(\gamma)  (\I_6 + \bLambda_0 \bDelta_0(\gamma)  )^{-1} \bv \bv^\top \bDelta_0(\gamma)  (\I_6 + \bLambda_0 \bDelta_0(\gamma) )^{-1} }{ \frac{a_1^2}{c^2} + \frac{\nu}{c} + \frac{\gamma}{c} -  \bv^\top \bDelta_0(\gamma)  (\I_6 +  \bLambda_0 \bDelta_0(\gamma)  )^{-1} \bv }, \label{eq:relation_bLambda_1_bLambda_0}
\end{align}
and
\begin{equation}
   \U_0^\top \left( \Q_{-i}^\perp - \Q_{-i}^\perp \U_0 ( \I_6 + \bLambda_1(\gamma) \bDelta_0(\gamma) )^{-1} \bLambda_1(\gamma) \U_0^\top \Q_{-i}^\perp \right) \U_0 = \bDelta_0(\gamma) ( \I_6 + \bLambda_1(\gamma) \bDelta_0(\gamma) )^{-1} + o_{\| \cdot \|}(1).
\end{equation}
\end{Lemma}
\begin{proof}[Proof of \Cref{lem:approx_Q}]
We first work on $[\Q^{-1}]_{21}$ by expanding the term $ \frac{1}{\sqrt{p}} f(\balpha_{-i})^\top \check{\Z}_{-i} \K_{-i} $ as
\begin{align*}
    \frac{f(\balpha_{-i})^\top}{\sqrt{p}} \check{\Z}_{-i} \K_{-i} &= \frac{f(\balpha_{-i})^\top}{\sqrt{p}} ( \check{\Z}_{-i}^\perp + \frac{1}{p} \balpha_{-i} \balpha_{-i}^\top ) ( \K_{-i}^\perp + \frac{a_1}{p} \balpha_{-i} \balpha_{-i}^\top ) + o_{\| \cdot \|}(1) \\
    &= \frac{f(\balpha_{-i})^\top}{\sqrt{p}} \check{\Z}_{-i}^\perp \K_{-i}^\perp + a_1 \frac{n}{p} \frac{1}{\sqrt{p}} \balpha_{-i}^\top \K_{-i}^\perp + a_1^2 \frac{n}{p} \left( 1 + \frac{n}{p} \right) \frac{\balpha_{-i}^\top}{\sqrt{p}} + o_{\| \cdot \|}(1),
\end{align*}
so that
\begin{equation}\label{eq:Q_inv_21_SM}
    [\Q^{-1}]_{21} = \frac{f(\balpha_{-i})^\top}{\sqrt{p}} \check{\Z}_{-i} \K_{-i} + a_1 \frac{n}{p} \frac{f(\balpha_{-i})^\top}{\sqrt{p}} + o_{\| \cdot \|}(1) = \bv^\top \U_0^\top + o_{\| \cdot \|}(1),
\end{equation}
with $\bv \in \RR^6$ defined in \eqref{eq:def_U_bLambda_1_SM}.

So that 
\begin{align*}
   [\Q]_{ii} [\Q^{-1}]_{21} ([\Q^{-1}]_{11})^{-1} &= m(\gamma) \bv^\top \U_0^\top \left( \Q_{-i}^\perp - \Q_{-i}^\perp \U_0 \left( \I_6 + \bLambda_0 \bDelta_0(\gamma) \right)^{-1} \bLambda_0 \U_0^\top \Q_{-i}^\perp \right) + o_{\| \cdot \|}(1) \\ 
   &= m(\gamma) \bv^\top \left( \I_6 + \bDelta_0(\gamma) \bLambda_0 \right)^{-1} \U_0^\top \Q_{-i}^\perp + o_{\| \cdot \|}(1).
\end{align*}
Then, with the approximation of the inverse $([\Q^{-1}]_{11})^{-1}$ in \Cref{lem:approx_inv_Q_11} and that of $[\Q^{-1}]_{21}$ above, we obtain
\begin{align*}
   ([\Q^{-1}]_{11} - [\Q^{-1}]_{12} ([\Q^{-1}]_{22})^{-1} [\Q^{-1}]_{21})^{-1} &= (  \K_{-i}^\perp \check{\Z}_{-i}^\perp \K_{-i}^\perp + \U_0 \bLambda_0 \U_0^\top + \frac{\gamma}{c} \I_{n-1} - ([\Q^{-1}]_{22})^{-1} \U_0 \bv \bv^\top \U_0^\top )^{-1}  + o_{\| \cdot \|}(1) \nonumber \\ 
   &= (  \K_{-i}^\perp \check{\Z}_{-i}^\perp \K_{-i}^\perp + \U_0 \bLambda_1 \U_0^\top + \frac{\gamma}{c} \I_{n-1} )^{-1}  + o_{\| \cdot \|}(1) \nonumber \\ 
   &= \Q_{-i}^\perp - \Q_{-i}^\perp \U_0 ( \I_6 + \bLambda_1(\gamma) \bDelta_0(\gamma) )^{-1} \bLambda_1(\gamma) \U_0^\top \Q_{-i}^\perp + o_{\| \cdot \|}(1),
\end{align*}
by \eqref{eq:approx_Q_inv_22} and Woodbury identity, for
\begin{equation}
    \bLambda_1(\gamma) = \bLambda_0 - ([\Q^{-1}]_{22})^{-1} \bv \bv^\top + o_{\| \cdot \|}(1) = \bLambda_0 - \left(\frac{a_1^2}{c^2} +  \frac{\nu}{c} + \frac{\gamma}{c} \right)^{-1} \bv \bv^\top + o_{\| \cdot \|}(1),
 \end{equation} 
as defined in \eqref{eq:def_U_bLambda_1_SM}.
This concludes the proof of \Cref{lem:approx_Q}.
\end{proof}

Following the same idea, we expand the quadratic form $[\Q^{-1}]_{21} ([\Q^{-1}]_{11})^{-1} [\Q^{-1}]_{12}$ in \eqref{eq:Q_ii} as
\begin{equation}
   [\Q^{-1}]_{21} ([\Q^{-1}]_{11})^{-1} [\Q^{-1}]_{12} = \bv^\top  \bDelta_0(\gamma) (\I_6 + \bLambda_0 \bDelta_0(\gamma) )^{-1} \bv + o(1),
\end{equation}
for $\bv \in \RR^6$ defined in \eqref{eq:def_U_bLambda_1_SM}. 
Plugging this approximation back to \eqref{eq:Q_ii} and ignoring the terms in $o(1)$, we obtain the following self-consistent equation on $m(\gamma)$,
\begin{equation}
   \frac{1}{m(\gamma)} = \frac{\gamma}{c} + \frac{\nu}{c} +  \frac{a_1^2}{c^2} - \bv^\top \mathbf{T}(\gamma) \bv, \quad \mathbf{T}(\gamma) = \bDelta_0(\gamma) (\I_6 + \bLambda_0 \bDelta_0(\gamma) )^{-1}.
\end{equation}

In the following, we determine the (self-consistent) equations for $\delta_1(\gamma), \delta_2(\gamma), \delta_3(\gamma)$ and $\delta_4(\gamma)$ in $\mathbf{T}(\gamma)$, so as to retrieve the final self-consistent equations in \eqref{eq:def_m_and_deltas_SM}.

\subsubsection{Establishing self-consistent equations for $\delta(\gamma)$s}
\label{ssub:self_consistent_equation_for_deltas}

Following the same idea above in \Cref{ssub:self_consistent_equation_for_m}, we now establish self-consistent equations for the intermediate variables $\delta_1(\gamma), \delta_2(\gamma), \delta_3(\gamma), \delta_4(\gamma)$ defined in \eqref{eq:def_m_and_deltas_SM}.

\paragraph{Self-consistent equation for $\delta_1(\gamma)$.}
We start with $\delta_1(\gamma) = \frac{1}{p} \tr(\Q(\gamma) \K) + o(1)$ by writing
\begin{align*}
\delta_1(\gamma)  &= \frac{1}{p} \tr(\Q \K) + o(1) = \frac1p \sum_{i=1}^n [\Q \K]_{ii} + o(1) = \frac1{c} [\Q \K]_{ii} + o(1) \\ 
&= - [\Q]_{ii} \frac1c [\Q^{-1}]_{21} ([\Q^{-1}]_{11})^{-1} f(\balpha_{-i})/\sqrt p + o(1) \\ 
&= - \frac{m(\gamma)}c \bv^\top \U_0^\top ([\Q^{-1}]_{11})^{-1} \U_0 \bv_1 + o(1) = - \frac{m(\gamma)}c \bv^\top  \bDelta_0(\gamma) (\I_6 + \bLambda_0 \bDelta_0(\gamma) )^{-1} \bv_1 + o(1) \\ 
&= - \frac{m(\gamma)}c \bv^\top \mathbf{T}(\gamma) \bv_1 + o(1),
\end{align*}
for 
\begin{equation}
   \bv_1^\top = \begin{bmatrix}  0 & 1 & 0 & 0 & 0 & 0 \end{bmatrix} \in \RR^6.
\end{equation}
where we used the fact that $f(\balpha_{-i})/\sqrt p = \U_0 \bv_1$, \eqref{eq:Q_inv_21_SM}, and \Cref{lem:approx_inv_Q_11}.
\paragraph{Self-consistent equation for $\delta_2(\gamma)$.}
We consider now $\delta_2(\gamma) = \frac{1}{p} \tr(\Q(\gamma) \K \check \Z ) + o(1)$ and write
\begin{align*}
   &\delta_2(\gamma) = \frac{1}{p} \tr(\Q \K \check \Z) + o(1) = \frac1{c} [\check \Z \Q \K]_{ii} + o(1) \\ 
   &=
   \frac1c 
   \begin{bmatrix} \balpha_{-i}^\top/\sqrt p & 1 \end{bmatrix}
   \begin{bmatrix}  
   \Q_{-i}^\perp - \Q_{-i}^\perp \U_0 ( \I_6 + \bLambda_1(\gamma) \bDelta_0(\gamma) )^{-1} \bLambda_1(\gamma) \U_0^\top \Q_{-i}^\perp & -m(\gamma) \Q_{-i}^\perp \U_0 \left( \I_6 + \bLambda_0 \bDelta_0(\gamma) \right)^{-1} \bv \\ 
   - m(\gamma) \bv^\top \left( \I_6 + \bDelta_0(\gamma) \bLambda_0 \right)^{-1} \U_0^\top \Q_{-i}^\perp & [\Q]_{ii} 
   \end{bmatrix} \begin{bmatrix} f(\balpha_{-i})/\sqrt p \\ 0 \end{bmatrix} + o(1) \\ 
   &= \frac1c \left( \bv_2^\top \bDelta_0(\gamma) ( \I_6 + \bLambda_1(\gamma) \bDelta_0(\gamma) )^{-1} \bv_1 - m(\gamma) \bv^\top \bDelta_0(\gamma) ( \I_6 + \bLambda_0 \bDelta_0(\gamma) )^{-1} \bv_1 \right) + o(1) \\ 
   &= \frac1c \left( (\bv_2 - m(\gamma) \bv)^\top \bDelta_0(\gamma) (\I_6 + \bLambda_0 \bDelta_0(\gamma))^{-1} \bv_1 + \frac{ \bv_2^\top \bDelta_0(\gamma) (\I_6 + \bLambda_0 \bDelta_0(\gamma)  )^{-1}  \bv \times \bv^\top \bDelta_0(\gamma) (\I_6 + \bLambda_0 \bDelta_0(\gamma))^{-1} \bv_1 }{ \frac{a_1^2}{c^2} + \frac{\nu}{c} + \gamma  -  \bv^\top \bDelta_0(\gamma) (\I_6 + \bLambda_0 \bDelta_0(\gamma) )^{-1} \bv }  \right) + o(1) \\
   &= \frac1c \left( (\bv_2 - m(\gamma) \bv)^\top \bDelta_0(\gamma) (\I_6 + \bLambda_0 \bDelta_0(\gamma))^{-1} \bv_1 + m(\gamma) \bv_2^\top \bDelta_0(\gamma) (\I_6 + \bLambda_0 \bDelta_0(\gamma)  )^{-1}  \bv \times \bv^\top \bDelta_0(\gamma) (\I_6 + \bLambda_0 \bDelta_0(\gamma))^{-1} \bv_1  \right) + o(1) \\
   &= \frac1c \left( \bv_2^\top \mathbf{T}(\gamma) \bv_1 + c \delta_1(\gamma) \left(1 - \bv_2^\top \mathbf{T}(\gamma) \bv \right) \right) + o(1),
\end{align*}
for $\bv_1^\top = \begin{bmatrix}  0 & 1 & 0 & 0 & 0 & 0 \end{bmatrix} \in \RR^6, \bv_2^\top = \begin{bmatrix}  1 & 0 & 0 & 0 & 0 & 0 \end{bmatrix} \in \RR^6$,
where we used the fact that $\balpha_{-i}/\sqrt p = \U_0 \bv_2$, $f(\balpha_{-i})/\sqrt p = \U_0 \bv_1$, \Cref{lem:approx_Q}, and the relation in \eqref{eq:relation_bLambda_1_bLambda_0}.
\paragraph{Self-consistent equation for $\delta_3(\gamma)$.}
We consider now $\delta_3(\gamma) = \frac{1}{p} \tr(\K \Q(\gamma) \K ) + o(1)$ and write
\begin{align*}
   \delta_3(\gamma) &= \frac{1}{p} \tr(\K \Q \K) + o(1) = \frac1{c} [\K \Q \K]_{ii} + o(1) \\
   &=
   \frac1c 
   \begin{bmatrix} f(\balpha_{-i})^\top/\sqrt p & 0 \end{bmatrix}
   \begin{bmatrix}  
   \Q_{-i}^\perp - \Q_{-i}^\perp \U_0 ( \I_6 + \bLambda_1(\gamma) \bDelta_0(\gamma) )^{-1} \bLambda_1(\gamma) \U_0^\top \Q_{-i}^\perp & -m(\gamma) \Q_{-i}^\perp \U_0 \left( \I_6 + \bLambda_0 \bDelta_0(\gamma) \right)^{-1} \bv \\ 
   - m(\gamma) \bv^\top \left( \I_6 + \bDelta_0(\gamma) \bLambda_0 \right)^{-1} \U_0^\top \Q_{-i}^\perp & [\Q]_{ii} 
   \end{bmatrix} \\ 
   &\times \begin{bmatrix} f(\balpha_{-i})/\sqrt p \\ 0 \end{bmatrix} + o(1) = \frac1c \bv_1^\top \bDelta_0(\gamma) ( \I_6 + \bLambda_1(\gamma) \bDelta_0(\gamma) )^{-1} \bv_1 + o(1) \\ 
   &= \frac1c \left( \bv_1^\top \bDelta_0(\gamma) (\I_6 + \bLambda_0 \bDelta_0(\gamma))^{-1} \bv_1 + \frac{ (\bv_1^\top \bDelta_0(\gamma) (\I_6 +  \bLambda_0 \bDelta_0(\gamma) )^{-1} \bv)^2 }{ \frac{a_1^2}{c^2} + \frac{\nu}{c} + \gamma  -  \bv^\top \bDelta_0(\gamma) (\I_6 + \bLambda_0 \bDelta_0(\gamma) )^{-1} \bv }  \right) + o(1) \\
   &= \frac1c \left( \bv_1^\top \bDelta_0(\gamma) (\I_6 + \bLambda_0 \bDelta_0(\gamma))^{-1} \bv_1 + m(\gamma) (\bv_1^\top \bDelta_0(\gamma) (\I_6 +  \bLambda_0 \bDelta_0(\gamma) )^{-1} \bv)^2 \right) + o(1) \\ 
   &= \frac1c \left( \bv_1^\top \mathbf{T}(\gamma) \bv_1 + m(\gamma) (\bv_1^\top \mathbf{T}(\gamma) \bv)^2 \right) + o(1) = \frac1c \left( \bv_1^\top \mathbf{T}(\gamma) \bv_1 + \frac{ c^2 \delta_1^2(\gamma) }{ m (\gamma) } \right) + o(1).
\end{align*}

\paragraph{Self-consistent equation for $\delta_4(\gamma)$.}
We consider now $\delta_4(\gamma) = \frac{1}{p} \tr(\check \Z \K \Q(\gamma) \K \check \Z) + o(1)$ and write
\begin{align*}
   &\delta_4(\gamma) = \frac{1}{p} \tr(\check \Z \K \Q \K \check \Z) + o(1) = \frac1{c} [\check \Z \K \Q \K \check \Z]_{ii} + o(1) \\
   &=
   \frac1c 
   \begin{bmatrix} (\K_{-i}^\perp\balpha_{-i} + \frac{a_1}c \balpha_{-i} + f(\balpha_{-i}) )^\top /\sqrt p & \frac{a_1}c \end{bmatrix} \\ 
   & \times
   \begin{bmatrix}  
   \Q_{-i}^\perp - \Q_{-i}^\perp \U_0 ( \I_6 + \bLambda_1(\gamma) \bDelta_0(\gamma) )^{-1} \bLambda_1(\gamma) \U_0^\top \Q_{-i}^\perp & -m(\gamma) \Q_{-i}^\perp \U_0 \left( \I_6 + \bLambda_0 \bDelta_0(\gamma) \right)^{-1} \bv \\ 
   - m(\gamma) \bv^\top \left( \I_6 + \bDelta_0(\gamma) \bLambda_0 \right)^{-1} \U_0^\top \Q_{-i}^\perp & [\Q]_{ii} 
   \end{bmatrix} \begin{bmatrix} * \\ * \end{bmatrix} + o(1) \\ 
   &= \frac1c \left( \bv_4^\top \bDelta_0(\gamma) ( \I_6 + \bLambda_1(\gamma) \bDelta_0(\gamma) )^{-1} \bv_4 - \frac{2 a_1 m(\gamma)}c \bv_4^\top \bDelta_0(\gamma) ( \I_6 + \bLambda_0(\gamma) \bDelta_0(\gamma) )^{-1}  \bv + \frac{a_1^2}{c^2} m(\gamma) \right) + o(1) \\ 
   &= \frac1c \left( \bv_4^\top \bDelta_0(\gamma) (\I_6 + \bLambda_0 \bDelta_0(\gamma))^{-1} \bv_4 + \frac{ (\bv_4^\top \bDelta_0(\gamma) (\I_6 +  \bLambda_0 \bDelta_0(\gamma) )^{-1} \bv)^2 }{ \frac{a_1^2}{c^2} + \frac{\nu}{c} + \gamma  -  \bv^\top \bDelta_0(\gamma) (\I_6 + \bLambda_0 \bDelta_0(\gamma) )^{-1} \bv } - \frac{2 a_1 m(\gamma)}c \bv_4^\top \bDelta_0(\gamma) ( \I_6 + \bLambda_0(\gamma) \bDelta_0(\gamma) )^{-1}  \bv \right)\\ 
   & + \frac{a_1^2}{c^2 \times c} m(\gamma) + o(1) \\
   &= \frac1c \left( \bv_4^\top \bDelta_0(\gamma) (\I_6 + \bLambda_0 \bDelta_0(\gamma))^{-1} \bv_4 + m(\gamma) \left(\bv_4^\top \bDelta_0(\gamma) (\I_6 +  \bLambda_0 \bDelta_0(\gamma) )^{-1} \bv - \frac{a_1}{c} \right)^2 \right) + o(1) \\
   &= \frac1c \left( \bv_4^\top \mathbf{T}(\gamma) \bv_4 + m(\gamma) \left(\bv_4^\top \mathbf{T}(\gamma) \bv - \frac{a_1}{c} \right)^2 \right) + o(1),
\end{align*}
for 
\begin{equation}
   \bv_4^\top = \begin{bmatrix} \frac{a_1}c & 1 & 1 & 0 & 0 & 0 \end{bmatrix} \in \RR^6.
\end{equation}
Putting these together, we obtain the system of equations as in \eqref{eq:full_expression}.

We thus conclude the proof of \Cref{prop:DE_resovlent_noise}.

\subsection{Proof of Theorem~\ref{theo:high_interpolation}}
\label{subsec:proof_of_theo:ICM}

Here, we provide detailed derivations of \Cref{theo:high_interpolation} on the Deterministic Equivalent of the interpolation error $E$ in \eqref{eq:def_E} \Cref{def:interpolation_error}. 
To do this, recall the following structured nonlinear resolvent
\begin{equation}
   \Q(\gamma) = \left( \frac1n \K_\X^\top \X^\top \X \K_\X + \gamma \I_n \right)^{-1},
\end{equation}
in \eqref{eq:def_Q} of \Cref{def:interpolation_error}.

First note that by \Cref{lem:linearization_K_X}, we have 
\begin{equation}
   \K_\X =  \K_N + \U_K \bSigma_\K \V_Q^\top + O_{\| \cdot \|}(n^{-1/2}),
\end{equation}
for $\bSigma_\K \in \RR^{3 \times 3}$ given by 
\begin{equation}\label{eq:def_Sigma_K}
\bSigma_\K= a_1 \left[ \begin{smallmatrix} \| \bmu \|^2 + \bmu^\top \w_K \w_Q^\top \bmu & 1 & \bmu^\top \w_K \\ 1 & 0 & 0 \\ \bmu^\top \w_Q & 0 & 1  \end{smallmatrix}\right] \in \RR^{3 \times 3},
\end{equation}
Similarly, under \Cref{ass:high-dim}, we have
\begin{equation}
    \frac1n \X^\top \X = \frac1n \Z^\top \Z +  \U_K \bSigma_\X \U_K^\top = c \check \Z + \U_K \bSigma_\X \U_K^\top, \quad \bSigma_\X \equiv c \Big[
\begin{smallmatrix}
   \| \bmu \|^2 & 1 & 0 \\ 
   1 & 0 & 0 \\ 
   0 & 0 & 0
\end{smallmatrix} \Big],
 \end{equation} 
that is of bounded norm with probability one as $n,p \to \infty$ at the same rate.
As such, we have 
\begin{align*}
   \Q(\gamma) 
   &= \left(  \frac1n \K_N \Z^\top \Z \K_N + \U \bSigma \U^\top + \gamma \I_n \right)^{-1} + O_{\| \cdot \|}(n^{-\frac12}) \\
   &= \left(  \frac1n \K_N \Z^\top \Z \K_N + \gamma \I_n \right)^{-1}  + O_{\| \cdot \|}(n^{-\frac12}) \\ 
   &- \left(  \frac1n \K_N \Z^\top \Z \K_N + \gamma \I_n \right)^{-1} \U \left( \bSigma^{-1} + \U^\top \left(  \frac1n \K_N \Z^\top \Z \K_N + \gamma \I_n \right)^{-1} \U \right)^{-1} \U^\top \left(  \frac1n \K_N \Z^\top \Z \K_N + \gamma \I_n \right)^{-1},
\end{align*}
by Woodbury identity, for 
\begin{equation}\label{eq:def_U}
   \U = \begin{bmatrix} \frac1n \K_N \Z^\top \Z \U_K & \K_N \U_K & \V_Q \end{bmatrix} \in \RR^{n \times 9}, 
\end{equation}
with $\U_K \in \RR^{n \times 3}$ and $\V_Q \in \RR^{n \times 3}$ defined in \Cref{lem:linearization_K_X}, and 
\begin{equation}\label{eq:def_Sigma}
   \bSigma = \begin{bmatrix}
   \mathbf{0}_3 & \mathbf{0}_3  & \bSigma_\K \\
   \mathbf{0}_3 & \bSigma_\X  & \bSigma_\X \U_K^\top \U_K \bSigma_\K \\
   \bSigma_\K^\top & \bSigma_\K^\top \U_K^\top \U_K \bSigma_\X  & \bSigma_\K^\top (\U_K^\top \frac1n \Z^\top \Z \U_K + \U_K^\top \U_K \bSigma_\X \U_K^\top \U_K )\bSigma_\K 
   \end{bmatrix} \in \RR^{9 \times 9},
\end{equation}

Our objective of interest is the the interpolation error $E$ defined in \eqref{eq:def_Q} of \Cref{def:interpolation_error} as 
\begin{equation}
   E = - \frac{\gamma^2}n \frac{\partial \y^\top \Q(\gamma) \y}{\partial \gamma}.
\end{equation}
Note that $\y/\sqrt p$ is the first column of $\V_Q$ and thus the seventh column of $\U$ defined in \eqref{eq:def_U}, so that 
\begin{align}
   &\frac1n \y^\top \Q(\gamma) \y = c \cdot \ee_7^\top \U^\top \Q(\gamma) \U \ee_7 \nonumber \\ 
   &= c \cdot \ee_7^\top \U^\top \left(  \frac1n \K_N \Z^\top \Z \K_N + \gamma \I_n \right)^{-1} \U \cdot \left( \I_9 + \bSigma \U^\top \left(  \frac1n \K_N \Z^\top \Z \K_N + \gamma \I_n \right)^{-1} \U \right)^{-1}  \ee_7 + O(n^{-\frac12}), \label{eq:yQy}
\end{align}
where $\ee_7 \in \RR^9$ is the canonical vector at location seven.

We have the following approximation for the above objective of interest.
\begin{Lemma}[Further approximations]\label{lem:further_approx}
For $\bSigma$ defined in \eqref{eq:def_Sigma} and $\U$ in \eqref{eq:def_U}, we have the following approximations in spectral norm holds with high probability as $n,p \to \infty$ with $p/n \to c \in (0, \infty)$,
\begin{align}
   &\bSigma = \bLambda + O_{\| \cdot \|}(n^{-\frac12}) \label{eq:def_Lambda} \\ 
   &\U^\top \left(  \frac1p \K_N \Z^\top \Z \K_N + \frac{\gamma}{c} \I_n \right)^{-1} \U = \bDelta(\gamma) + O_{\| \cdot \|}(n^{-\frac12}) \label{eq:def_Delta},
\end{align}
with $\bLambda = \begin{bmatrix}
   \mathbf{0}_3 & \mathbf{0}_3 & \bSigma_\K \\ 
   \mathbf{0}_3 & \bSigma_\X & [\bLambda]_{2,3} \\ 
   \bSigma_\K^\top & [\bLambda]_{2,3}^\top & [\bLambda]_{3,3} \\ 
\end{bmatrix} \in \RR^{9 \times 9}$ and $\bDelta(\gamma) = 
\begin{bmatrix}
   [\bDelta(\gamma)]_{1,1} & [\bDelta(\gamma)]_{1,2} & [\bDelta(\gamma)]_{1,3} \\ 
   [\bDelta(\gamma)]_{1,2}^\top & [\bDelta(\gamma)]_{2,2} & [\bDelta(\gamma)]_{2,3} \\ 
   [\bDelta(\gamma)]_{1,3}^\top & [\bDelta(\gamma)]_{2,3}^\top & [\bDelta(\gamma)]_{3,3} \\ 
\end{bmatrix} \in \CC^{9 \times 9} $ both three-by-three block symmetric matrices with corresponding blocks given by
\begin{align*}
   [\bLambda]_{2,3} &= a_1 \begin{bmatrix}
         (\| \bmu \|^2 + 1) T_1 & \| \bmu \|^2 & \bmu^\top \w_K (\| \bmu \|^2 + 1) \\ 
         T_1 & 1 & \bmu^\top \w_K \\ 
         0 & 0 & 0
      \end{bmatrix} \\
   [\bLambda]_{3,3} &= a_1^2 \begin{bmatrix}
         \frac{2 + c + \| \bmu \|^2}c T_1^2 + \frac{1+c}c T_1 + \frac{1+c}c \bmu^\top \w_Q \left( \bmu^\top \w_K  +  \bmu^\top \w_Q \| \w_K \|^2 \right)  & (*) & (*) \\ 
         \frac{1 + c + \| \bmu \|^2}c T_1 & 1 + \frac{\| \bmu \|^2}c & \frac{1+c+\| \bmu \|^2}c \bmu^\top \w_K \\ 
         \frac{2 + c + \| \bmu \|^2}c \bmu^\top \w_K T_1 + \frac{1+c}c (\bmu^\top \w_K + \bmu^\top \w_Q \| \w_K \|^2 ) & (*) & \frac{2+c+\| \bmu \|^2}c (\bmu^\top \w_K)^2  + \frac{1+c}c \| \w_K \|^2
      \end{bmatrix}
\end{align*}
for
\begin{equation}
   T_1 = \| \bmu \|^2 + \bmu^\top \w_K \bmu^\top \w_Q,
\end{equation}
and $\bSigma_\K \equiv a_1 \left[ \begin{smallmatrix} T_1 & 1 & \bmu^\top \w_K \\ 1 & 0 & 0 \\ \bmu^\top \w_Q & 0 & 1  \end{smallmatrix}\right] \in \RR^{3 \times 3}$ defined in \eqref{eq:def_Sigma_K} of \Cref{lem:linearization_K_X}, $\bSigma_\X \equiv c \Big[
\begin{smallmatrix}
   \| \bmu \|^2 & 1 & 0 \\ 
   1 & 0 & 0 \\ 
   0 & 0 & 0
\end{smallmatrix} \Big] \in \RR^{3 \times 3}$, as well as
\begin{align*}
   [\bDelta(\gamma)]_{1,1} &= \begin{bmatrix}
      c^2 \delta_4(\gamma) & 0 & 0 \\
      0 & c^2 \| \bmu \|^2 \delta_7(\gamma) & c^2 \bmu^\top \w_K \delta_7(\gamma) \\
      0 & (*) & c^2 \| \w_K \|^2 \delta_7(\gamma) &  \\
   \end{bmatrix} \in \RR^{3 \times 3} \\
   [\bDelta(\gamma)]_{1,2} &= \begin{bmatrix}
      1 - \frac{\gamma}{c} m(\gamma) & 0 & 0 \\
      0 & c \| \bmu \|^2 \delta_4(\gamma) & c \bmu^\top \w_K \delta_4(\gamma) \\
      0 & c \bmu^\top \w_K \delta_4(\gamma) & c \| \w_K \|^2 \delta_4(\gamma) \\
   \end{bmatrix} \in \RR^{3 \times 3} \\
   [\bDelta(\gamma)]_{1,3} &= \begin{bmatrix}
      c \delta_2(\gamma) & 0 & 0 \\
      0 & c \| \bmu \|^2 \delta_6(\gamma) & c \bmu^\top \w_Q \delta_6(\gamma) \\
      0 & c \bmu^\top \w_K \delta_6(\gamma) & c \w_K^\top \w_Q \delta_6(\gamma) \\
   \end{bmatrix} \in \RR^{3 \times 3} \\
   [\bDelta(\gamma)]_{2,2} &= \begin{bmatrix}
      \delta_3(\gamma) & 0 & 0 \\
      0 & \frac1c \| \bmu \|^2 (1 - \frac{\gamma}{c} m(\gamma)) & \frac1c \bmu^\top \w_K (1 - \frac{\gamma}{c} m(\gamma)) \\
      0 & (*) & \frac1c \| \w_K \|^2 (1 - \frac{\gamma}{c} m(\gamma)) \\
   \end{bmatrix} \in \RR^{3 \times 3} \\
   [\bDelta(\gamma)]_{2,3} &= \begin{bmatrix}
      \delta_1(\gamma) & 0 & 0 \\
      0 & \| \bmu \|^2 \delta_2(\gamma) & \bmu^\top \w_Q \delta_2(\gamma) \\
      0 & \bmu^\top \w_K \delta_2(\gamma) & \w_K^\top \w_Q \delta_2(\gamma) \\
   \end{bmatrix} \in \RR^{3 \times 3} \\
   [\bDelta(\gamma)]_{3,3} &= \begin{bmatrix}
        \frac1c m(\gamma) & 0 & 0 \\ 0 & \| \bmu \|^2 \delta_5(\gamma) & \bmu^\top \w_Q \delta_5(\gamma) \\ 0 & (*) & \| \w_Q \|^2 \delta_5(\gamma)
    \end{bmatrix} \in \RR^{3 \times 3},
\end{align*}
for $\delta_1(\gamma), \delta_2(\gamma), \delta_3(\gamma), \delta_4(\gamma)$ as defined in \eqref{eq:def_m_and_deltas_SM} of the proof of \Cref{prop:DE_resovlent_noise}, and 
\begin{equation}\label{eq:delta5_delta6_delta6}
   \left\{
    \begin{array}{ll}
    c\delta_5(\gamma) &= m(\gamma) \left(1 - \bv_2^\top \mathbf{T}(\gamma) \bv  \right) \\ 
    c\delta_6(\gamma) 
   & = \bv_4^\top \mathbf{T}(\gamma) \bv_2 + m(\gamma) (\bv_2^\top \mathbf{T}(\gamma) \bv -1 ) \left( \bv_4^\top \mathbf{T}(\gamma) \bv - \frac{a_1}c \right) \\ 
   c\delta_7(\gamma) 
   &= \bv_4^\top \mathbf{T}(\gamma) \bv_7 + m(\gamma) \left( \bv_4^\top \mathbf{T}(\gamma) \bv - \frac{a_1}c \right) \left( \bv_7^\top \mathbf{T}(\gamma) \bv  - \frac{a_1}c \left( 2 + \frac1c \right) \right)
    \end{array}
  \right.
\end{equation}
with
\begin{equation}
   \bv_7^\top = \begin{bmatrix} 2 \frac{a_1}{c} + \frac{a_1}{c^2} & \frac{1}{c} + 1 & \frac{1}{c} + 1 & 0 & 1 & 0 \end{bmatrix} \in \RR^6.
\end{equation}
\end{Lemma}
\begin{proof}[Proof of \Cref{lem:further_approx}]

We first work on the approximation of $\bSigma$ defined in \eqref{eq:def_Sigma}, for which we exploit the following concentration results:
\begin{align*}
   \U_K^\top \U_K & = \frac1c \begin{bmatrix}
   1 & 0 & 0 \\    
   0 & \| \bmu \|^2 & \bmu^\top \w_K \\ 
   0 & \bmu^\top \w_K & \| \w_K \|^2 
   \end{bmatrix} + O_{\| \cdot \|}(n^{-\frac12}), \\ 
   \frac1n \U_K^\top \Z^\top \Z \U_K &= \begin{bmatrix}
   1 & 0 & 0 \\    
   0 & \frac{1+c}c \| \bmu \|^2 & \frac{1+c}c \bmu^\top \w_K \\ 
   0 & \frac{1+c}c \bmu^\top \w_K & \frac{1+c}c \| \w_K \|^2 
   \end{bmatrix} + O_{\| \cdot \|}(n^{-\frac12}),
\end{align*}
where we used the Gaussian moments, we thus get
\begin{equation}
   \U_K^\top \frac1n \Z^\top \Z \U_K + \U_K^\top \U_K \bSigma_\X \U_K^\top \U_K = \begin{bmatrix}
      1 +  \| \bmu \|^2/c &  \| \bmu \|^2/c &  \bmu^\top \w_K/c \\
      \| \bmu \|^2/c & \frac{1+c}c \| \bmu \|^2 & \frac{1+c}c \bmu^\top \w_K \\
      \bmu^\top \w_K/c & \frac{1+c}c \bmu^\top \w_K & \frac{1+c}c \| \w_K \|^2 
   \end{bmatrix} + O_{\| \cdot \|}(n^{-\frac12}),
\end{equation}
and therefore $\bSigma = \bLambda + O_{\| \cdot \|}(n^{-\frac12})$ with 
\begin{equation}
   \bSigma = \begin{bmatrix}
      \zo_3 & \zo_3 & \bSigma_\K \\ 
      \zo_3 & \bSigma_\X & [\bLambda]_{2,3} \\ 
      \bSigma_\K^\top & [\bLambda]_{2,3}^\top & [\bLambda]_{3,3} \\ 
   \end{bmatrix},
\end{equation}
and
\begin{align*}
   [\bLambda]_{2,3} &= a_1 \begin{bmatrix}
         (\| \bmu \|^2 + 1) T_1 & \| \bmu \|^2 & \bmu^\top \w_K (\| \bmu \|^2 + 1) \\ 
         T_1 & 1 & \bmu^\top \w_K \\ 
         0 & 0 & 0
      \end{bmatrix} \\
   [\bLambda]_{3,3} &= a_1^2 \begin{bmatrix}
         \frac{2 + c + \| \bmu \|^2}c T_1^2 + \frac{1+c}c T_1 + \frac{1+c}c \bmu^\top \w_Q \left( \bmu^\top \w_K  +  \bmu^\top \w_Q \| \w_K \|^2 \right)  & (*) & (*) \\ 
         \frac{1 + c + \| \bmu \|^2}c T_1 & 1 + \frac{\| \bmu \|^2}c & \frac{1+c+\| \bmu \|^2}c \bmu^\top \w_K \\ 
         \frac{2 + c + \| \bmu \|^2}c \bmu^\top \w_K T_1 + \frac{1+c}c (\bmu^\top \w_K + \bmu^\top \w_Q \| \w_K \|^2 ) & (*) & \frac{2+c+\| \bmu \|^2}c (\bmu^\top \w_K)^2  + \frac{1+c}c \| \w_K \|^2
      \end{bmatrix}.
\end{align*}
where we denote the shortcut $T_1 = \| \bmu \|^2 + \bmu^\top \w_K \bmu^\top \w_Q$, 
This concludes the proof of the approximation of $\bSigma$ in \Cref{lem:further_approx}.

We then proceed to the approximation of $\U^\top \left(  \frac1n \K_N \Z^\top \Z \K_N + \gamma \I_n \right)^{-1} \U$.
Note that for $\U$ defined in \eqref{eq:def_U} and 
\begin{equation}\label{eq:def_Q0}
   \Q_0 \equiv \left(  \frac1p \K_N \Z^\top \Z \K_N + \frac{\gamma}{c} \I_n \right)^{-1},
\end{equation}
we have 
\begin{equation}
   \U^\top \Q_0 \U = \begin{bmatrix}
      \U_K^\top \frac1n \Z^\top \Z \K_N \Q_0 \frac1n \K_N \Z^\top \Z \U_K & \U_K^\top \frac1n \Z^\top \Z \K_N \Q_0 \K_N \U_K & \U_K^\top \frac1n \Z^\top \Z \K_N \Q_0 \V_Q \\ 
      \U_K^\top \K_N \Q_0 \frac1n \K_N \Z^\top \Z \U_K & \U_K^\top \K_N \Q_0 \K_N \U_K & \U_K^\top \K_N \Q_0 \V_Q \\ 
      \V_Q^\top \Q_0 \frac1n \K_N \Z^\top \Z \U_K & \V_Q^\top \Q_0 \K_N \U_K & \V_Q^\top \Q_0 \V_Q
   \end{bmatrix} \in \RR^{9 \times 9},
\end{equation}
which writes as a three-by-three block matrix, for $\U_K = [\y,~\Z^\top \bmu,~\Z^\top \w_K]/\sqrt p\in \RR^{n \times 3}$, $ \V_Q = [\y,~\Z^\top \bmu,~\Z^\top \w_Q]/\sqrt p \in \RR^{n \times 3}$ as in \Cref{lem:linearization_K_X}. 

In the following, we further evaluate the nine (in fact six by symmetry) blocks of $\U^\top \Q_0 \U$, in the limit of $n,p \to \infty$ with $ p/n \to c \in (0, \infty)$.
To that end, we need the following intermediate results.

\begin{Lemma}[Further Deterministic Equivalents]\label{lem:1}
Under the same settings and notations as in \Cref{prop:DE_resovlent_noise}, we have the following Deterministic Equivalent results (in the sense of \Cref{def:DE})
\begin{align*}
   \frac1{n^2} \Z^\top \Z \K_N \Q_0 \K_N \Z^\top \Z & \leftrightarrow c^3 \delta_4(\gamma) \cdot \I_n, \\
   \frac1n \Z^\top \Z \K_N \Q_0 \K_N &\leftrightarrow \left(c - \gamma m(\gamma) \right) \cdot \I_n, \\
   \frac1n \Z^\top \Z \K_N \Q_0  & \leftrightarrow c^2 \delta_2(\gamma) \cdot \I_n, \\
   \K_N \Q_0 \K_N  &\leftrightarrow c \delta_3(\gamma) \cdot \I_n, \\
   \K_N \Q_0 & \leftrightarrow c \delta_1(\gamma) \cdot \I_n , \\
   \frac1p \Z \frac1n \Z^\top \Z \K_N \Q_0 \frac1n \K_N \Z^\top \Z \Z^\top & \leftrightarrow  c^2 \delta_7(\gamma) \cdot \I_p, \\
   \frac1p \Z \frac1n \Z^\top \Z \K_N \Q_0 \K_N \Z^\top & \leftrightarrow c \delta_4(\gamma) \cdot \I_p, \\
   \frac1p \Z \frac1n \Z^\top \Z \K_N \Q_0 \Z^\top & \leftrightarrow c \delta_6(\gamma) \cdot \I_p, \\
   \frac1p \Z \K_N \Q_0 \K_N \Z^\top &\leftrightarrow \frac1c \left(1 - \frac{\gamma}{c} m (\gamma ) \right) \cdot \I_p , \\
   \frac1p \Z \K_N \Q_0 \Z^\top &\leftrightarrow \delta_2(\gamma) \cdot \I_p, \\
   \frac1p \Z \Q_0 \Z^\top &\leftrightarrow \delta_5(\gamma) \cdot \I_p.
\end{align*}
\end{Lemma}

\begin{proof}[Proof of \Cref{lem:1}]
Note that for $\Q_0$ defined in \eqref{eq:def_Q0}, we have, by the proof of \Cref{prop:DE_resovlent_noise} in \Cref{subsec:proof_of_prop:DE_resovlent_noise}., the following Deterministic Equivalent results.
\begin{align*}
   \frac1{n^2} \Z^\top \Z \K_N \Q_0 \K_N \Z^\top \Z & \leftrightarrow \frac1{n^2} \EE[\Z^\top \Z \K_N \Q_0 \K_N \Z^\top \Z] \leftrightarrow \frac{p^3}{n^3} \tr \frac{1}{p} \left( \frac{1}{p} \Z^\top \Z \K_N \Q_0 \K_N \frac{1}{p} \Z^\top \Z  \right) \cdot \I_n \leftrightarrow c^3 \delta_4(\gamma) \cdot \I_n, \\
   \frac1n \Z^\top \Z \K_N \Q_0 \K_N &\leftrightarrow \frac1n \EE[\Z^\top \Z \K_N \Q_0 \K_N] \leftrightarrow  \frac{p}{n} \frac1n \tr \left( \frac1p \Z^\top \Z \K_N \Q_0 \K_N \right) \cdot \I_n \leftrightarrow \left(c - \gamma m(\gamma) \right) \cdot \I_n, \\
   \frac1n \Z^\top \Z \K_N \Q_0  & \leftrightarrow \frac1n \EE[\Z^\top \Z \K_N \Q_0] \leftrightarrow \frac{p^2}{n^2} \frac1p \tr \left( \frac1p \Z^\top \Z \K_N \Q_0 \right)\cdot \I_n  \leftrightarrow c^2 \delta_2(\gamma) \cdot \I_n, \\
   \K_N \Q_0 \K_N  &\leftrightarrow \EE[\K_N \Q_0 \K_N ] \leftrightarrow \frac{p}{n} \frac1p \tr \left( \K_N \Q_0 \K_N \right) \cdot \I_n \leftrightarrow c \delta_3(\gamma) \cdot \I_n, \\
   \K_N \Q_0 &\leftrightarrow \EE[\K_N \Q_0] \leftrightarrow \frac{p}{n} \frac1p \tr \left( \K_N \Q_0 \right) \cdot \I_n \leftrightarrow c \delta_1(\gamma) \cdot \I_n , \\
   \Q_0 &\leftrightarrow m(\gamma) \cdot \I_n.
\end{align*}
Similarly, we have
\begin{align*}
   \frac1p \Z \frac1n \Z^\top \Z \K_N \Q_0 \frac1n \K_N \Z^\top \Z \Z^\top & \leftrightarrow \frac{p^2}{n^2} \frac{1}{p} \tr \left( \frac1p \Z^\top \Z \K_N \Q_0 \frac1p \K_N \Z^\top \Z \frac1p \Z^\top \Z \right) \cdot \I_p \leftrightarrow c^2 \delta_7(\gamma) \cdot \I_p, \\
   \frac1p \Z \frac1n \Z^\top \Z \K_N \Q_0 \K_N \Z^\top &\leftrightarrow \frac{p}{n} \frac1p \tr \left( \frac1p \Z^\top \Z \K_N \Q_0 \K_N \frac1p \Z^\top \Z \right) \cdot \I_p  \leftrightarrow c \delta_4(\gamma) \cdot \I_p, \\
   \frac1p \Z \frac1n \Z^\top \Z \K_N \Q_0 \Z^\top &\leftrightarrow \frac{p}{n} \frac1p \tr \left( \frac1p \Z^\top \Z \K_N \Q_0 \frac1p \Z^\top \Z \right) \cdot \I_p \leftrightarrow c \delta_6(\gamma) \cdot \I_p, \\
   \frac1p \Z \K_N \Q_0 \K_N \Z^\top &\leftrightarrow \frac1p \tr \left( \K_N \Q_0 \K_N \frac1p \Z^\top \Z \right) \cdot \I_p \leftrightarrow \frac1c \left(1 - \frac{\gamma}{c} m (\gamma ) \right) \cdot \I_p , \\
   \frac1p \Z \K_N \Q_0 \Z^\top &\leftrightarrow \frac1p \tr \left( \K_N \Q_0 \frac1p \Z^\top \Z \right) \cdot \I_p \leftrightarrow \delta_2(\gamma) \cdot \I_p, \\
   \frac1p \Z \Q_0 \Z^\top &\leftrightarrow \frac{1}{p} \tr \left( \Q_0 \frac1p \Z^\top \Z \right) \cdot \I_p \leftrightarrow \delta_5(\gamma) \cdot \I_p,
\end{align*}
for $\delta_5(\gamma),\delta_6(\gamma),\delta_7(\gamma)$ as defined in \eqref{eq:delta5_delta6_delta6}.

To complete the proof of \Cref{lem:1}, we establish, in the following as similar to \Cref{ssub:self_consistent_equation_for_deltas}, self-consistent equations for $\delta_5(\gamma),\delta_6(\gamma)$ and $\delta_7(\gamma)$.

\paragraph{Self-consistent equation for $\delta_5(\gamma)$.}

Consider $\delta_5(\gamma) = \frac{1}{p} \tr( \Q \check \Z) + o(1)$ and write
\begin{align*}
   \delta_5(\gamma) &= \frac{1}{p} \tr( \Q \check \Z) + o(1) = \frac1{c} [\Q \check \Z]_{ii} + o(1) \\ 
    &= \frac{1}{c} \begin{bmatrix}
        - m(\gamma) \bv^\top \left( \I_6 + \bDelta_0(\gamma) \bLambda_0 \right)^{-1} \U^\top \Q_{-i}^\perp & [\Q]_{ii}  
    \end{bmatrix} \begin{bmatrix}
        \balpha_{-i}/\sqrt{p} \\ 1
    \end{bmatrix} + o(1) \\ 
    &= \frac{1}{c} \left( - m(\gamma) \bv^\top \left( \I_6 + \bDelta_0(\gamma) \bLambda_0 \right)^{-1} \U^\top \Q_{-i}^\perp \U \bv_2 + m(\gamma) \right) + o(1) \\
    &= \frac{m(\gamma)}{c} \left( - \bv^\top \left( \I_6 + \bDelta_0(\gamma) \bLambda_0 \right)^{-1} \bDelta_0(\gamma) \bv_2 + 1 \right) + o(1) \\
    &= \frac{m(\gamma)}{c} \left( - \bv_2^\top \bDelta_0(\gamma) (\I_6 +  \bLambda_0 \bDelta_0(\gamma) )^{-1} \bv + 1 \right) + o(1) \\ 
    &= \frac{m(\gamma)}{c} \left( 1 - \bv_2^\top \mathbf{T}(\gamma) \bv \right) + o(1).
\end{align*}

\paragraph{Self-consistent equation for $\delta_6(\gamma)$.}

Consider now $\delta_6(\gamma) = \frac{1}{p} \tr(\check \Z \K \Q \check \Z) + o(1)$ and write
\begin{align*}
    \delta_6(\gamma) &= \frac{1}{p} \tr(\check \Z \K \Q \check \Z) + o(1) = \frac1{c} [\check \Z \K \Q \check \Z]_{ii} + o(1) \\
    &= \frac{1}{c} \begin{bmatrix}
        \bv_4^\top \U^\top & \frac{a_1}{c}
    \end{bmatrix} \begin{bmatrix}
        \Q_{-i}^\perp - \Q_{-i}^\perp \U ( \I_6 + \bLambda_1(\gamma) \bDelta_0(\gamma) )^{-1} \bLambda_1(\gamma) \U^\top \Q_{-i}^\perp & -m(\gamma) \Q_{-i}^\perp \U \left( \I_6 + \bLambda_0 \bDelta_0(\gamma) \right)^{-1} \bv \\ 
   - m(\gamma) \bv^\top \left( \I_6 + \bDelta_0(\gamma) \bLambda_0 \right)^{-1} \U^\top \Q_{-i}^\perp & [\Q]_{ii}
    \end{bmatrix} \begin{bmatrix}
        \U \bv_2 \\ 1
    \end{bmatrix} + o(1) \\
    &= \frac{1}{c} \left( \bv_4^\top \bDelta_0(\gamma) (\I_6 + \bLambda_1(\gamma) \bDelta_0(\gamma))^{-1} \bv_2 - m(\gamma) \bv_4^\top \bDelta_0(\gamma) (\I_6 + \bLambda_0 \bDelta_0(\gamma))^{-1} \bv \right. \\
    & \left. - \frac{a_1}{c} m(\gamma) \bv^\top (\I_6 + \bDelta_0(\gamma) \bLambda_0)^{-1} \bDelta_0(\gamma) \bv_2 + \frac{a_1}{c} m(\gamma) \right) + o(1) \\
    &= \frac{1}{c} \left( \bv_4^\top \bDelta_0(\gamma) (\I_6 + \bLambda_0 \bDelta_0(\gamma))^{-1} \bv_2 + m(\gamma) \bv_4^\top \bDelta_0(\gamma) (\I_6 + \bLambda_0 \bDelta_0(\gamma))^{-1} \bv \times \bv^\top \bDelta_0(\gamma) (\I_6 + \bLambda_0 \bDelta_0(\gamma))^{-1} \bv_2 \right.\\
    &\left.- m(\gamma) \left( \frac{a_1}{c} \bv_2 + \bv_4 \right)^\top \bDelta_0(\gamma) (\I_6 + \bLambda_0 \bDelta_0(\gamma))^{-1} \bv + \frac{a_1}{c} m(\gamma) \right) + o(1) \\ 
    &= \frac1c \left( \bv_4^\top \mathbf{T}(\gamma) \bv_2 + m(\gamma) (\bv_2^\top \mathbf{T}(\gamma) \bv -1 ) \left( \bv_4^\top \mathbf{T}(\gamma) \bv - \frac{a_1}c \right) \right) + o(1).
\end{align*}

\paragraph{Self-consistent equation for $\delta_7(\gamma)$.}

Consider $\delta_7(\gamma) = \frac{1}{p} \tr(\check \Z \K \Q \K \check \Z \check \Z) + o(1)$ and write
\begin{align*}
    \delta_7(\gamma) &= \frac{1}{p} \tr(\check \Z \K \Q \K \check \Z \check \Z) + o(1) = \frac1{c} [\check \Z \K \Q \K \check \Z \check \Z]_{ii} + o(1) \\
    &= \frac{1}{c} \begin{bmatrix}
        \bv_4^\top \U^\top & \frac{a_1}{c}
    \end{bmatrix} \begin{bmatrix}
        \Q_{-i}^\perp - \Q_{-i}^\perp \U ( \I_6 + \bLambda_1(\gamma) \bDelta(\gamma) )^{-1} \bLambda_1(\gamma) \U^\top \Q_{-i}^\perp & -m(\gamma) \Q_{-i}^\perp \U \left( \I_6 + \bLambda_0 \bDelta(\gamma) \right)^{-1} \bv \\ 
   - m(\gamma) \bv^\top \left( \I_6 + \bDelta(\gamma) \bLambda_0 \right)^{-1} \U^\top \Q_{-i}^\perp & [\Q]_{ii}
    \end{bmatrix} \begin{bmatrix}
        \U \bv_7 \\ 2 \frac{a_1}{c} + \frac{a_1}{c^2}
    \end{bmatrix} + o(1) \\
    &= \frac{1}{c} \left( \bv_4^\top \bDelta(\gamma) (\I_6 + \bLambda_1(\gamma) \bDelta(\gamma))^{-1} \bv_7 - \left( 2 \frac{a_1}{c} + \frac{a_1}{c^2} \right) m(\gamma) \bv_4^\top \bDelta(\gamma) (\I_6 + \bLambda_0 \bDelta(\gamma))^{-1} \bv \right.\\
    &\left. - \frac{a_1}{c} m(\gamma) \left( \bv^\top (\I_6 + \bDelta(\gamma) \bLambda_0)^{-1} \bDelta(\gamma) \bv_7 - \left( 2 \frac{a_1}{c} + \frac{a_1}{c^2} \right) \right) \right) + o(1) \\
    &= \frac{1}{c} \left( \bv_4^\top \bDelta (\I_6 + \bLambda_0 \bDelta(\gamma))^{-1} \bv_7 + m(\gamma) \bv_4^\top \bDelta (\I_6 + \bLambda_0 \bDelta(\gamma))^{-1} \bv \bv^\top \bDelta (\I_6 + \bLambda_0 \bDelta(\gamma))^{-1} \bv_7 \right.\\
    &\left. - \frac{a_1}{c} m(\gamma) \left( \left( \frac{1}{c} + 2 \right) \bv_4 + \bv_7 \right)^\top \bDelta(\gamma) (\I_6 + \bLambda_0 \bDelta(\gamma))^{-1} \bv + \frac{a_1^2}{c^2} \left( \frac{1}{c} + 2 \right) m(\gamma) \right) + o(1) .
\end{align*} 
With these self-consistent equations for $\delta_5(\gamma),\delta_6(\gamma), \delta_7(\gamma)$, we conclude the proof of \Cref{lem:1}.
\end{proof}

With \Cref{lem:1} at hand, we are now ready to evaluate the blocks of $\U^\top \Q_0 \U$ as follows.

\paragraph{Approximation of the $(1,1)$ block of $\U^\top \Q_0 \U $.}
\begin{align*}
    &\U_K^\top \frac1n \Z^\top \Z \K_N \Q_0(\gamma) \frac1n \K_N \Z^\top \Z \U_K = \frac{1}{p} \begin{bmatrix}
        \y^\top \\ \bmu^\top \Z \\ \w_K^\top \Z
    \end{bmatrix} \frac1n \Z^\top \Z \K_N \Q_0 \frac1n \K_N \Z^\top \Z \begin{bmatrix}
        \y & \Z^\top \bmu & \Z^\top \w_K
    \end{bmatrix} \\
    & = \frac{1}{p} \begin{bmatrix}
        \y^\top \frac1n \Z^\top \Z \K_N \Q_0 \frac1n \K_N \Z^\top \Z \y & \y^\top \frac1n \Z^\top \Z \K_N \Q_0 \frac1n \K_N \Z^\top \Z \Z^\top \bmu & \y^\top \frac1n \Z^\top \Z \K_N \Q_0 \frac1n \K_N \Z^\top \Z \Z^\top \w_K \\
        \bmu^\top \Z \frac1n \Z^\top \Z \K_N \Q_0 \frac1n \K_N \Z^\top \Z \y & \bmu^\top \Z \frac1n \Z^\top \Z \K_N \Q_0 \frac1n \K_N \Z^\top \Z \Z^\top \bmu & \bmu^\top \Z \frac1n \Z^\top \Z \K_N \Q_0 \frac1n \K_N \Z^\top \Z \Z^\top \w_K \\
        \w_K^\top \Z \frac1n \Z^\top \Z \K_N \Q_0 \frac1n \K_N \Z^\top \Z \y & \w_K^\top \Z \frac1n \Z^\top \Z \K_N \Q_0 \frac1n \K_N \Z^\top \Z \Z^\top \bmu & \w_K^\top \Z \frac1n \Z^\top \Z \K_N \Q_0 \frac1n \K_N \Z^\top \Z \Z^\top \w_K
    \end{bmatrix} \\
    & = c^2 \begin{bmatrix}
        \delta_4(\gamma) & 0 & 0 \\ 0 & \| \bmu \|^2 \delta_7(\gamma) & \bmu^\top \w_K \delta_7(\gamma) \\ 0 & \bmu^\top \w_K \delta_7(\gamma) & \| \w_K \|^2 \delta_7(\gamma) 
    \end{bmatrix} + O_{\| \cdot \|}(n^{-1/2}), 
\end{align*}
where we used \Cref{lem:1} for the approximation in the last line.

\paragraph{Approximation of the $(1,2)$ block of $\U^\top \Q_0 \U $.}
\begin{align*}
    & \U_K^\top \frac1n \Z^\top \Z \K_N \Q_0(\gamma) \K_N \U_K = \frac{1}{p} \begin{bmatrix}
        \y^\top \\ \bmu^\top \Z \\ \w_K^\top \Z
    \end{bmatrix} \frac1n \Z^\top \Z \K_N \Q_0 \K_N \begin{bmatrix}
        \y & \Z^\top \bmu & \Z^\top \w_K
    \end{bmatrix} \\
    &= \frac{1}{p} \begin{bmatrix}
        \y^\top \frac1n \Z^\top \Z \K_N \Q_0 \K_N \y & \y^\top \frac1n \Z^\top \Z \K_N \Q_0 \K_N \Z^\top \bmu & \y^\top \frac1n \Z^\top \Z \K_N \Q_0 \K_N \Z^\top \w_K \\
        \bmu^\top \Z \frac1n \Z^\top \Z \K_N \Q_0 \K_N \y & \bmu^\top \Z \frac1n \Z^\top \Z \K_N \Q_0 \K_N \Z^\top \bmu & \bmu^\top \Z \frac1n \Z^\top \Z \K_N \Q_0 \K_N \Z^\top \w_K \\
        \w_K^\top \Z \frac1n \Z^\top \Z \K_N \Q_0 \K_N \y & \w_K^\top \Z \frac1n \Z^\top \Z \K_N \Q_0 \K_N \Z^\top \bmu & \w_K^\top \Z \frac1n \Z^\top \Z \K_N \Q_0 \K_N \Z^\top \w_K 
    \end{bmatrix} \\
    &= \begin{bmatrix}
        1 - \frac{\gamma}{c} m(\gamma)  & 0 & 0 \\ 0 & c \| \bmu \|^2 \delta_4(\gamma) & c \bmu^\top \w_K \delta_4(\gamma) \\ 0 & c \bmu^\top \w_K \delta_4(\gamma) & c \| \w_K \|^2 \delta_4(\gamma) 
    \end{bmatrix} + O_{\| \cdot \|}(n^{-1/2}),
\end{align*}
where we used \Cref{lem:1} for the approximation in the last line.

\paragraph{Approximation of the $(1,3)$ block of $\U^\top \Q_0 \U $.}
\begin{align*}
    & \U_K^\top \frac1n \Z^\top \Z \K_N \Q_0(\gamma) \V_Q = \frac{1}{p} \begin{bmatrix}
        \y^\top \\ \bmu^\top \Z \\ \w_K^\top \Z
    \end{bmatrix} \frac1n \Z^\top \Z \K_N \Q_0 \begin{bmatrix}
        \y & \Z^\top \bmu & \Z^\top \w_Q
    \end{bmatrix} \\
    &= \frac{1}{p} \begin{bmatrix}
        \y^\top \frac1n \Z^\top \Z \K_N \Q_0 \y & \y^\top \frac1n \Z^\top \Z \K_N \Q_0 \Z^\top \bmu & \y^\top \frac1n \Z^\top \Z \K_N \Q_0 \Z^\top \w_Q \\
        \bmu^\top \Z \frac1n \Z^\top \Z \K_N \Q_0 \y & \bmu^\top \Z \frac1n \Z^\top \Z \K_N \Q_0 \Z^\top \bmu & \bmu^\top \Z \frac1n \Z^\top \Z \K_N \Q_0 \Z^\top \w_Q \\
        \w_K^\top \Z \frac1n \Z^\top \Z \K_N \Q_0 \y & \w_K^\top \Z \frac1n \Z^\top \Z \K_N \Q_0 \Z^\top \bmu & \w_K^\top \Z \frac1n \Z^\top \Z \K_N \Q_0 \Z^\top \w_Q 
    \end{bmatrix} \\
    &= c \begin{bmatrix}
        \delta_2(\gamma) & 0 & 0 \\ 0 & \| \bmu \|^2 \delta_6(\gamma) & \bmu^\top \w_Q \delta_6(\gamma) \\ 0 & \bmu^\top \w_K \delta_6(\gamma) & \w_K^\top \w_Q \delta_6(\gamma) 
    \end{bmatrix} + O_{\| \cdot \|}(n^{-1/2}),
\end{align*}
where we used \Cref{lem:1} for the approximation in the last line.

\paragraph{Approximation of the $(2,2)$ block of $\U^\top \Q_0 \U $.}
\begin{align*}
    &\U_K^\top \K_N \Q_0(\gamma) \K_N \U_K = \frac{1}{p} \begin{bmatrix}
        \y^\top \\ \bmu^\top \Z \\ \w_K^\top \Z
    \end{bmatrix} \K_N \Q_0 \K_N \begin{bmatrix}
        \y & \Z^\top \bmu & \Z^\top \w_K
    \end{bmatrix} \\
    &= \frac{1}{p} \begin{bmatrix}
        \y^\top \K_N \Q_0 \K_N \y & \y^\top \K_N \Q_0 \K_N \Z^\top \bmu & \y^\top \K_N \Q_0 \K_N \Z^\top \w_K \\
        \bmu^\top \Z \K_N \Q_0 \K_N \y & \bmu^\top \Z \K_N \Q_0 \K_N \Z^\top \bmu & \bmu^\top \Z \K_N \Q_0 \K_N \Z^\top \w_K \\
        \w_K^\top \Z \K_N \Q_0 \K_N \y & \w_K^\top \Z \K_N \Q_0 \K_N \Z^\top \bmu & \w_K^\top \Z \K_N \Q_0 \K_N \Z^\top \w_K 
    \end{bmatrix} \\
    &= \frac1c \begin{bmatrix}
        c \delta_3(\gamma) & 0 & 0 \\ 0 & \| \bmu \|^2 ( 1 - \frac{\gamma}{c} m(\gamma) ) & \bmu^\top \w_K ( 1 - \frac{\gamma}{c} m(\gamma) ) \\ 0 & \bmu^\top \w_K ( 1 - \frac{\gamma}{c} m(\gamma) ) & \| \w_K \|^2 ( 1 - \frac{\gamma}{c} m(\gamma) )
    \end{bmatrix} + O_{\| \cdot \|}(n^{-1/2}),
\end{align*}
where we used \Cref{lem:1} for the approximation in the last line.

\paragraph{Approximation of the $(2,3)$ block of $\U^\top \Q_0 \U $.}
\begin{align*}
    &\U_K^\top \K_N \Q_0(\gamma) \V_Q = \frac{1}{p} \begin{bmatrix}
        \y^\top \\ \bmu^\top \Z \\ \w_K^\top \Z
    \end{bmatrix} \K_N \Q_0 \begin{bmatrix}
        \y & \Z^\top \bmu & \Z^\top \w_Q
    \end{bmatrix} \\
    &= \frac{1}{p} \begin{bmatrix}
        \y^\top \K_N \Q_0 \y & \y^\top \K_N \Q_0 \Z^\top \bmu & \y^\top \K_N \Q_0 \Z^\top \w_Q \\
        \bmu^\top \Z \K_N \Q_0 \y & \bmu^\top \Z \K_N \Q_0 \Z^\top \bmu & \bmu^\top \Z \K_N \Q_0 \Z^\top \w_Q \\
        \w_K^\top \Z \K_N \Q_0 \y & \w_K^\top \Z \K_N \Q_0 \Z^\top \bmu & \w_K^\top \Z \K_N \Q_0 \Z^\top \w_Q 
    \end{bmatrix} \\ 
    &= \begin{bmatrix}
        \delta_1(\gamma) & 0 & 0 \\ 0 & \| \bmu \|^2 \delta_2(\gamma) & \bmu^\top \w_Q \delta_2(\gamma) \\ 0 & \bmu^\top \w_K \delta_2(\gamma) & \w_K^\top \w_Q \delta_2(\gamma) 
    \end{bmatrix} + O_{\| \cdot \|}(n^{-1/2}),
\end{align*}
where we again use \Cref{lem:1} for the approximation in the last line.

\paragraph{Approximation of the $(3,3)$ block of $\U^\top \Q_0 \U $.}
\begin{align*}
    &\V_Q^\top \Q_0(\gamma) \V_Q = \frac{1}{p} \begin{bmatrix}
        \y^\top \\ \bmu^\top \Z \\ \w_Q^\top \Z
    \end{bmatrix} \K_N \Q_0 \K_N \begin{bmatrix}
        \y & \Z^\top \bmu & \Z^\top \w_Q
    \end{bmatrix} \\
    &= \frac{1}{p} \begin{bmatrix}
        \y^\top \Q_0 \y & \y^\top \Q_0 \Z^\top \bmu & \y^\top \Q_0 \Z^\top \w_Q \\
        \bmu^\top \Z \Q_0 \y & \bmu^\top \Z \Q_0 \Z^\top \bmu & \bmu^\top \Z \Q_0 \Z^\top \w_Q \\
        \w_Q^\top \Z \Q_0 \y & \w_Q^\top \Z \Q_0 \Z^\top \bmu & \w_Q^\top \Z \Q_0 \Z^\top \w_Q 
    \end{bmatrix} \\ 
    &= \begin{bmatrix}
        \frac{n}{p} m(\gamma) & 0 & 0 \\ 0 & \| \bmu \|^2 \delta_5(\gamma) & \bmu^\top \w_Q \delta_5(\gamma) \\ 0 & \bmu^\top \w_Q \delta_5(\gamma) & \| \w_Q \|^2 \delta_5(\gamma)
    \end{bmatrix} + O_{\| \cdot \|}(n^{-1/2}),
\end{align*}
where we used \Cref{lem:1} for the approximation in the last line.
This concludes the proof of the approximation of the quadratic form $\U^\top \left(  \frac1n \K_N \Z^\top \Z \K_N + \gamma \I_n \right)^{-1} \U$ in \Cref{lem:further_approx}.
\end{proof}

With \Cref{lem:further_approx} at hand, it follows from \eqref{eq:yQy} that 
\begin{align*}
   \frac1n \y^\top \Q(\gamma) \y &= c \cdot \ee_7^\top \U^\top \left(  \frac1n \K_N \Z^\top \Z \K_N + \gamma \I_n \right)^{-1} \U \cdot \left( \I_9 + \bSigma \U^\top \left(  \frac1n \K_N \Z^\top \Z \K_N + \gamma \I_n \right)^{-1} \U \right)^{-1}  \ee_7 + O(n^{-\frac12}) \\ 
   &= c \cdot \ee_7^\top \bDelta(\gamma) \cdot \left( \I_9 + \bLambda \bDelta(\gamma) \right)^{-1}  \ee_7 + O(n^{-\frac12}).
\end{align*}
To assess the high-dimensional behavior of the interpolation error $E$ defined in \eqref{eq:def_Q} of \Cref{def:interpolation_error}, it thus remains to evaluate the following derivative (with respective to $\gamma$) as
\begin{align*}
    E&= - \frac{\gamma^2}{n} \frac{\partial y^\top \Q(\gamma) y}{\partial \gamma} = - \gamma^2 c^2 \cdot \ee_7^\top \left( c \I_9 + \bDelta(\gamma) \bLambda \right)^{-1} \bDelta'(\gamma) \left( c \I_9 + \bLambda \bDelta(\gamma) \right)^{-1}  \ee_7 + O(n^{-\frac12}),
\end{align*}
where we denote $\bDelta'(\gamma)$ the derivative (with respect to $\gamma$) of $\bDelta(\gamma)$ defined in \eqref{eq:def_Delta}. 

To evaluate $\bDelta'(\gamma)$, we need the following result on the derivatives of $m'(\gamma)$ and $\delta_(\gamma)$s.
\begin{Lemma}[Derivatives of the $\delta(\gamma)$s]\label{lem:derivatives_deltas}
Under the settings and notations of \Cref{theo:high_interpolation}, we have that $m'(\gamma),\delta_1'(\gamma),\delta_2'(\gamma),\delta_3'(\gamma),\delta_4'(\gamma)$ satisfy the following system of equations
\begin{equation}\label{eq:full_expression_derivatives}
   \left\{
    \begin{array}{ll}
    m'(\gamma) &= \left( \bv^\top \T'(\gamma) \bv - \frac{1}{c} \right) m^2(\gamma) \\ 
    c \delta_1'(\gamma) &=  -  m'(\gamma) \bv^\top \T(\gamma) \bv_1 - m(\gamma) \bv^\top \T'(\gamma) \bv_1 \\
    c \delta_2'(\gamma) &= \bv_2^\top \T'(\gamma) (\bv_1 - c \delta_1(\gamma) \bv) +  c \delta'_1(\gamma) (1 - \bv_2^\top \T(\gamma) \bv) \\ 
    c \delta_3'(\gamma) &= \bv_1^\top \T'(\gamma) \bv_1 + \frac{ c^2 \delta_1(\gamma) \left( 2 \delta_1'(\gamma) m(\gamma) - \delta_1(\gamma) m'(\gamma) \right) }{m^2(\gamma)} \\ 
    c \delta_4'(\gamma) &= \bv_4^\top \T'(\gamma) \bv_4 + m'(\gamma) \left( \bv_4^\top \T(\gamma) \bv - \frac{a_1}c \right)^2 + 2 m(\gamma) \left( \bv_4^\top \T(\gamma) \bv - \frac{a_1}c \right) \bv_4^\top \T'(\gamma) \bv
    \end{array}
  \right.
\end{equation}
for $\mathbf{T}(\gamma)$ and $\bDelta_0(\gamma)$ defined in \eqref{eq:def_T_SM} and \eqref{eq:def_bDelta_0_SM}, receptively, so that their derivatives (with respective to $\gamma$) satisfy
\begin{align*}
   \mathbf{T}'(\gamma) &= (\I_6 + \bDelta_0(\gamma) \bLambda_0 )^{-1} \bDelta_0'(\gamma) (\I_6 + \bLambda_0 \bDelta_0(\gamma) )^{-1},
\end{align*}
and
\begin{equation*}
   \bDelta_0'(\gamma) \equiv \begin{bmatrix}
        \frac{m'(\gamma)}{c} & \frac{a_1}{c} m'(\gamma) & \delta_1'(\gamma) & a_1 \delta_1'(\gamma) & \delta_2'(\gamma) & a_1 \delta_2'(\gamma) \\
        \frac{a_1}{c} m'(\gamma)& \frac{\nu}{c} m'(\gamma) & a_1 \delta_1'(\gamma) & \nu \delta_1'(\gamma) & a_1 \delta_2'(\gamma) & \nu \delta_2'(\gamma) \\
        \delta_1'(\gamma) & a_1 \delta_1'(\gamma) & \delta_3'(\gamma) & a_1 \delta_3'(\gamma) & -\frac1{c^2} (m (\gamma) + \gamma m'(\gamma) ) & -\frac{a_1}{c^2} (m (\gamma) + \gamma m'(\gamma) ) \\
        a_1 \delta_1'(\gamma) & \nu \delta_1'(\gamma) & a_1 \delta_3'(\gamma) & \nu \delta_3'(\gamma) & -\frac{a_1}{c^2} (m (\gamma) + \gamma m'(\gamma) ) & -\frac{\nu}{c^2} (m (\gamma) + \gamma m'(\gamma) ) \\
        \delta_2'(\gamma) & a_1 \delta_2'(\gamma) & -\frac{1}{c^2} (m (\gamma) + \gamma m'(\gamma) ) & -\frac{a_1}{c^2} (m (\gamma) + \gamma m'(\gamma) ) & \delta_4'(\gamma) & a_1 \delta_4'(\gamma) \\ 
        a_1 \delta_2'(\gamma) & \nu \delta_2'(\gamma) & -\frac{a_1}{c^2} (m (\gamma) + \gamma m'(\gamma) ) & -\frac{\nu}{c^2} (m (\gamma) + \gamma m'(\gamma) ) & a_1 \delta_4'(\gamma) & \nu \delta_4'(\gamma)
    \end{bmatrix}
\end{equation*}
Also, the derivatives $\delta_5'(\gamma),\delta_6'(\gamma), \delta_7'(\gamma)$ are given by 
\begin{align*}
    c \delta'_5(\gamma) &= m'(\gamma) \left( - \bv_2^\top \T(\gamma) \bv + 1 \right) - m(\gamma) \bv_2^\top \T'(\gamma) \bv \\
    c \delta'_6(\gamma) &= \bv_4^\top \T'(\gamma) \bv_2 + m'(\gamma) \left( \bv_2^\top \T(\gamma) \bv - 1 \right) \left( \bv_4^\top \T(\gamma) \bv - \frac{a_1}{c} \right) + m(\gamma) \bv_2^\top \T'(\gamma) \bv \left( \bv_4^\top \T(\gamma) \bv - \frac{a_1}{c} \right) \\
    & + m(\gamma) \bv_4^\top \T'(\gamma) \bv \left( \bv_2^\top \T(\gamma) \bv - 1 \right) \\
    c \delta'_7(\gamma) &=  \bv_4^\top \T'(\gamma) \bv_7 + m'(\gamma) \left( \bv_4^\top \T(\gamma) \bv - \frac{a_1}{c} \right) \left( \bv_7^\top \T(\gamma) \bv - \frac{a_1}{c} \left( 1 + \frac{1}{c} \right) \right) + m(\gamma) \bv_4^\top \T'(\gamma) \bv \left( \bv_7^\top \T(\gamma) \bv - \frac{a_1}{c} \left( 1 + \frac{1}{c} \right) \right) \\
    & + m(\gamma) \bv_7^\top \T'(\gamma) \bv \left( \bv_4^\top \T(\gamma) \bv - \frac{a_1}{c} \right) .
\end{align*}
\end{Lemma}

\begin{proof}[Proof of \Cref{lem:derivatives_deltas}]
By their definitions in \Cref{prop:DE_resovlent_noise} and \Cref{lem:further_approx}, we have
\begin{align*}
    m'(\gamma) &= - \frac{\frac{1}{c} - \bv^\top \T'(\gamma) \bv}{\left( \frac{\gamma}{c} + \frac{\nu}{c} + \frac{a_1^2}{c^2} - \bv^\top \T(\gamma) \bv \right)^2} =  \left( \bv^\top \T'(\gamma) \bv - \frac{1}{c} \right) m^2(\gamma) \\
    c \delta_1'(\gamma) &= -  \left( m'(\gamma) \bv^\top \T(\gamma) \bv_1 + m(\gamma) \bv^\top \T'(\gamma) \bv_1  \right) \\
    c \delta_2'(\gamma) &=  \bv_2^\top \T'(\gamma) \bv_1 +  c \delta'_1(\gamma) (1 - \bv_2^\top \T(\gamma) \bv) - c \delta_1(\gamma) \bv_2^\top \T'(\gamma) \bv = \bv_2^\top \T'(\gamma) (\bv_1 - c \delta_1(\gamma) \bv) +  c \delta'_1(\gamma) (1 - \bv_2^\top \T(\gamma) \bv) \\
    c\delta_3'(\gamma) &= \bv_1^\top \T'(\gamma) \bv_1 + \frac{ c^2 \delta_1(\gamma) \left( 2 \delta_1'(\gamma) m(\gamma) - \delta_1(\gamma) m'(\gamma) \right) }{m^2(\gamma)} \\
    c \delta_4'(\gamma) &= \bv_4^\top \T'(\gamma) \bv_4 + m'(\gamma) \left( \bv_4^\top \T(\gamma) \bv - \frac{a_1}c \right)^2 + 2 m(\gamma) \left( \bv_4^\top \T(\gamma) \bv - \frac{a_1}c \right) \bv_4^\top \T'(\gamma) \bv,
\end{align*}
with
\begin{align*}
   \mathbf{T}'(\gamma) &= \bDelta_0'(\gamma) (\I_6 + \bLambda_0 \bDelta_0(\gamma) )^{-1} - \bDelta_0(\gamma) (\I_6 + \bLambda_0 \bDelta_0(\gamma) )^{-1} \bLambda_0 \bDelta_0'(\gamma) (\I_6 + \bLambda_0 \bDelta_0(\gamma) )^{-1} \\ 
   &= (\I_6 + \bDelta_0(\gamma) \bLambda_0 )^{-1} \bDelta_0'(\gamma) (\I_6 + \bLambda_0 \bDelta_0(\gamma) )^{-1},
\end{align*}
and $\bDelta_0'(\gamma)$ as in the statement of \Cref{lem:derivatives_deltas}.

Similarly, by their definition in \Cref{lem:1}, we obtain the derivatives of $\delta_5,\delta_6$ and $\delta_7$ as
\begin{align*}
   \delta'_5(\gamma) &= \frac{1}{c} \left( m'(\gamma) \left( - \bv_2^\top \T(\gamma) \bv + 1 \right) - m(\gamma) \bv_2^\top \T'(\gamma) \bv \right) \\
    \delta'_6(\gamma) &= \frac{1}{c} \left( \bv_4^\top \T'(\gamma) \bv_2 + m'(\gamma) \bv_4^\top \T(\gamma) \bv \bv^\top \T(\gamma) \bv_2 + m(\gamma) \bv_4^\top \T'(\gamma) \bv \bv^\top \T(\gamma) \bv_2 \right.\\
    &\left.+ m(\gamma) \bv_4^\top \T(\gamma) \bv \bv^\top \T'(\gamma) \bv_2 - m'(\gamma) \left( \frac{a_1}{c} \bv_2 + \bv_4 \right)^\top \T(\gamma) \bv - m(\gamma) \left( \frac{a_1}{c} \bv_2 + \bv_4 \right)^\top \T'(\gamma) \bv + \frac{a_1}{c} m'(\gamma) \right) \\
    \delta'_7(\gamma) &= \frac{1}{c} \left( \bv_4^\top \T'(\gamma) \bv_7 + m'(\gamma) \bv_4^\top \T(\gamma) \bv \bv^\top \T(\gamma) \bv_7 + m(\gamma) \bv_4^\top \T'(\gamma) \bv \bv^\top \T(\gamma) \bv_7 + m(\gamma) \bv_4^\top \T(\gamma) \bv \bv^\top \T'(\gamma) \bv_7 \right.\\
    &\left.- \frac{a_1}{c} m'(\gamma) \left( \left( \frac{1}{c} + 2 \right) \bv_4 + \bv_7 \right)^\top \T(\gamma) \bv - \frac{a_1}{c} m(\gamma) \left( \left( \frac{1}{c} + 2 \right) \bv_4 + \bv_7 \right)^\top \T'(\gamma) \bv + \frac{a_1^2}{c^2} \left( \frac{1}{c} + 2 \right) m'(\gamma) \right).
\end{align*}
This concludes the proof of \Cref{lem:derivatives_deltas}.
\end{proof}

Putting these together, we conclude the proof of \Cref{theo:high_interpolation}.

\subsection{Proof of Proposition~\ref{prop:high_interpolation_error_RR}}
\label{subsec:proof_of_prop:ICM_RR}

Here, we provide the proof of \Cref{prop:high_interpolation_error_RR}. 
By the definition of linear regression interpolation error $E_{\rm LR}$ in \eqref{eq:def_E_RR} of \Cref{def:RR_interpolation_error}, it suffices to evaluate the following quadratic form 
\begin{equation}
   \frac1n \y^\top \left( \frac1n \X^\top \X  + \gamma \I_n \right)^{-1} \y,
\end{equation}
and its derivative with respect to $\gamma$, for $\X = \bmu \y^\top + \Z \in \RR^{p \times n}$ as in \Cref{theo:high_interpolation}.

By Woodbury identity, we have 
\begin{align*}
   &\frac1n \y^\top \left( \frac1n \X^\top \X  + \gamma \I_n \right)^{-1} \y = \ee_1^\top \U^\top \left( \frac1n \Z^\top \Z  + \gamma \I_n + \U \bLambda \U^\top \right)^{-1} \U \ee_1 \\ 
   & = \ee_1^\top \U^\top \Q_0(\gamma) \U \left( \I_2 + \bLambda \U^\top \Q_0(\gamma) \U \right)^{-1} \ee_1,
\end{align*}
where $\ee_1 = [1,~0]^\top$ and with a slight abuse of notations, we denote
\begin{equation}
   \U = [\y,~\Z^\top \bmu]/\sqrt n \in \RR^{n \times 2}, \quad \bLambda = \begin{bmatrix}
      \| \bmu \|^2 & 1 \\ 
      1 & 0
   \end{bmatrix} \in \RR^{2 \times 2}, 
   \quad \Q_0(\gamma) = \left( \frac1n \Z^\top \Z  + \gamma \I_n \right)^{-1}.
\end{equation}

Similar to \Cref{prop:DE_resovlent_noise}, we have the following Deterministic Equivalent result for the \emph{linear} resolvent $\Q_0(\gamma)$.
\begin{Lemma}[Deterministic Equivalent for $\Q_0$, {\citep[Theorem~2.4]{couillet2022RMT4ML}}]\label{lem:DE_linear}
Let $\Z \in \RR^{p \times n}$ have i.i.d.\@ sub-gaussian entries. 
Then, as $n,p \to \infty$ at the same pace with $p/n \to c \in (0,\infty)$ and $\gamma > 0$, the following Deterministic Equivalent (see~\Cref{def:DE}) holds
\begin{equation*}
   \left(\frac1n \Z \Z^\top   + \gamma \I_p\right)^{-1} \leftrightarrow m_{\rm LR}(\gamma) \cdot \I_p, \quad \left(\frac1n \Z^\top \Z  + \gamma \I_n \right)^{-1} \leftrightarrow \left( c m_{\rm LR}(\gamma) + \frac{1-c}\gamma \right) \I_n.
\end{equation*}
with $m(\gamma)$ is the unique Stieltjes transform solution to the following Mar\u{c}enko-Pastur equation~\citep{marvcenko1967distribution,couillet2022RMT4ML}
\begin{equation}\label{eq:MP_in_append}
   c\gamma m_{\rm LR}^2(\gamma) + \left( 1 - c + \gamma \right) m_{\rm LR}(\gamma) - 1 =0.
\end{equation}
\end{Lemma}

By \Cref{lem:DE_linear}, we have 
\begin{align*}
   &\frac1n \y^\top \left( \frac1n \X^\top \X  + \gamma \I_n \right)^{-1} \y = \ee_1^\top \U^\top \Q_0(\gamma) \U \left( \I_2 + \bLambda \U^\top \Q_0(\gamma) \U \right)^{-1} \ee_1 \\ 
   &= \left[ \begin{bmatrix} c m_{\rm LR}(\gamma) + \frac{1-c}\gamma & 0 \\ 0 & \| \bmu \|^2 ( 1 - \gamma m_{\rm LR}(\gamma) ) \end{bmatrix} \begin{bmatrix} 1 + \| \bmu \|^2 \left( c m_{\rm LR}(\gamma) + \frac{1-c}\gamma \right) & \| \bmu \|^2 ( 1 - \gamma m_{\rm LR}(\gamma)) \\ c m_{\rm LR}(\gamma) + \frac{1-c}\gamma & 1 \end{bmatrix}^{-1} \right]_{1,1} \\ 
   &= \frac{ c m_{\rm LR}(\gamma) + \frac{1-c}\gamma }{ 1 + \| \bmu \|^2 (1 - \gamma m_{\rm LR}(\gamma)) },
\end{align*}
so that by \eqref{eq:def_E_RR}, we obtain
\begin{equation}
   E_{\rm LR} - \bar E_{\rm LR} \to 0, \quad \bar E_{\rm LR} = - \frac{ c \gamma^2 m'(\gamma) + c - 1 + \| \bmu\|^2 \left(\gamma^2 m'(\gamma) + (1 - c - \gamma) (\gamma m(\gamma) - 1) \right) }{ \left( 1 +  \| \bmu \|^2 - \| \bmu \|^2 \gamma m_{\rm LR}(\gamma) \right)^2 },
\end{equation}
in probability as $n,p \to \infty$, with $m_{\rm LR}(\gamma)$ the Stieltjes transform solution to the Mar\u{c}enko-Pastur equation in \eqref{eq:MP_in_append}, and $m_{\rm LR}'(\gamma) = - \frac{ c m^2(\gamma) + m(\gamma)  }{ 2 c \gamma m(\gamma) + 1 - c + \gamma } $ its derivative with respect to $\gamma$.

This concludes the proof of \Cref{prop:high_interpolation_error_RR}.

\section{Additional Numerical Results and Discussions}
\label{sec:SM_additional_nums}

In this section, we present additional numerical results.

\begin{figure}[htb]
\centering
 \begin{subfigure}[t]{0.32\textwidth}
  \begin{tikzpicture}[font=\footnotesize]
    \renewcommand{\axisdefaulttryminticks}{4} 
    \pgfplotsset{every major grid/.append style={densely dashed}}          
    \tikzstyle{every axis y label}+=[yshift=-10pt] 
    \tikzstyle{every axis x label}+=[yshift=5pt]
    \pgfplotsset{every axis legend/.style={cells={anchor=west},fill=white,at={(0.98,0.9)}, anchor=north east, font=\footnotesize}}
    \begin{axis}[
      width=1\linewidth,
      height=.75\linewidth,
      xmin=0.1,xmax=10,
      xmode=log,
      ymax=1,
        grid=major,
        ymajorgrids=false,
        scaled ticks=false,
        xlabel={ SNR $\| \bmu \|^2$ },
        ylabel={ $E$ }
        ]
        \addplot[smooth,BLUE,line width=1pt] plot coordinates{ 
        (0.100000,0.743423)(0.117210,0.731978)(0.137382,0.718984)(0.161026,0.704304)(0.188739,0.687815)(0.221222,0.669410)(0.259294,0.649017)(0.303920,0.626599)(0.356225,0.602174)(0.417532,0.575818)(0.489390,0.547673)(0.573615,0.517954)(0.672336,0.486939)(0.788046,0.454968)(0.923671,0.422424)(1.082637,0.389722)(1.268961,0.357278)(1.487352,0.325499)(1.743329,0.294754)(2.043360,0.265365)(2.395027,0.237590)(2.807216,0.211623)(3.290345,0.187588)(3.856620,0.165548)(4.520354,0.145507)(5.298317,0.127426)(6.210169,0.111225)(7.278954,0.096799)(8.531679,0.084025)(10.000000,0.072770)
        };
        \addplot[densely dashed,GREEN,line width=1.5pt] plot coordinates{ 
        (0.100000,0.768956)(0.117210,0.760041)(0.137382,0.749821)(0.161026,0.738149)(0.188739,0.724871)(0.221222,0.709838)(0.259294,0.692909)(0.303920,0.673960)(0.356225,0.652894)(0.417532,0.629654)(0.489390,0.604235)(0.573615,0.576694)(0.672336,0.547163)(0.788046,0.515853)(0.923671,0.483061)(1.082637,0.449157)(1.268961,0.414579)(1.487352,0.379809)(1.743329,0.345350)(2.043360,0.311696)(2.395027,0.279304)(2.807216,0.248567)(3.290345,0.219798)(3.856620,0.193218)(4.520354,0.168954)(5.298317,0.147048)(6.210169,0.127464)(7.278954,0.110108)(8.531679,0.094842)(10.000000,0.081501)
        };
        \end{axis}
  \end{tikzpicture}
  \caption{{ $ p/n = 1/4 $ }}
  \label{subfig:E1}
  \end{subfigure}
  ~
  \begin{subfigure}[t]{0.32\textwidth}
  \begin{tikzpicture}[font=\footnotesize]
    \renewcommand{\axisdefaulttryminticks}{4} 
    \pgfplotsset{every major grid/.append style={densely dashed}}          
    \tikzstyle{every axis y label}+=[yshift=-10pt] 
    \tikzstyle{every axis x label}+=[yshift=5pt]
    \pgfplotsset{every axis legend/.style={cells={anchor=west},fill=white,at={(0.98,0.9)}, anchor=north east, font=\footnotesize}}
    \begin{axis}[
      width=1\linewidth,
      height=.75\linewidth,
      xmin=0.1,xmax=10,
      xmode=log,
      ymax=1,
        grid=major,
        ymajorgrids=false,
        scaled ticks=false,
        xlabel={ SNR $\| \bmu \|^2$ },
        ylabel={ }
        ]
        \addplot[smooth,BLUE,line width=1pt] plot coordinates{ 
        (0.100000,0.689194)(0.117210,0.680450)(0.137382,0.670378)(0.161026,0.658817)(0.188739,0.645601)(0.221222,0.630569)(0.259294,0.613577)(0.303920,0.594503)(0.356225,0.573272)(0.417532,0.549871)(0.489390,0.524362)(0.573615,0.496902)(0.672336,0.467748)(0.788046,0.437252)(0.923671,0.405848)(1.082637,0.374022)(1.268961,0.342279)(1.487352,0.311109)(1.743329,0.280950)(2.043360,0.252173)(2.395027,0.225064)(2.807216,0.199823)(3.290345,0.176568)(3.856620,0.155347)(4.520354,0.136147)(5.298317,0.118907)(6.210169,0.103531)(7.278954,0.089899)(8.531679,0.077876)(10.000000,0.067320)
        };
        \addplot[densely dashed,GREEN,line width=1.5pt] plot coordinates{ 
        (0.100000,0.420965)(0.117210,0.416713)(0.137382,0.411822)(0.161026,0.406214)(0.188739,0.399805)(0.221222,0.392511)(0.259294,0.384249)(0.303920,0.374940)(0.356225,0.364513)(0.417532,0.352913)(0.489390,0.340108)(0.573615,0.326092)(0.672336,0.310896)(0.788046,0.294592)(0.923671,0.277297)(1.082637,0.259177)(1.268961,0.240440)(1.487352,0.221334)(1.743329,0.202134)(2.043360,0.183126)(2.395027,0.164592)(2.807216,0.146795)(3.290345,0.129959)(3.856620,0.114262)(4.520354,0.099826)(5.298317,0.086722)(6.210169,0.074964)(7.278954,0.064528)(8.531679,0.055349)(10.000000,0.047341)
        };
        \end{axis}
  \end{tikzpicture}
  \caption{{ $ p/n = 1 $ }}
  \label{subfig:E2}
  \end{subfigure}
  ~
  \begin{subfigure}[t]{0.32\textwidth}
  \begin{tikzpicture}[font=\footnotesize]
    \renewcommand{\axisdefaulttryminticks}{4} 
    \pgfplotsset{every major grid/.append style={densely dashed}}          
    \tikzstyle{every axis y label}+=[yshift=-10pt] 
    \tikzstyle{every axis x label}+=[yshift=5pt]
    \pgfplotsset{every axis legend/.style={cells={anchor=west},fill=white,at={(0.98,0.9)}, anchor=north east, font=\footnotesize}}
    \begin{axis}[
      width=1\linewidth,
      height=.75\linewidth,
      xmin=0.1,xmax=10,
      xmode=log,
      ymax=1,
        grid=major,
        ymajorgrids=false,
        scaled ticks=false,
        xlabel={ SNR $\| \bmu \|^2$ },
        ylabel={ }
        ]
        \addplot[smooth,BLUE,line width=1pt] plot coordinates{ 
        (0.100000,0.642323)(0.117210,0.638146)(0.137382,0.633220)(0.161026,0.627410)(0.188739,0.620561)(0.221222,0.612489)(0.259294,0.602986)(0.303920,0.591819)(0.356225,0.578734)(0.417532,0.563462)(0.489390,0.545740)(0.573615,0.525334)(0.672336,0.502075)(0.788046,0.475907)(0.923671,0.446936)(1.082637,0.415476)(1.268961,0.382073)(1.487352,0.347487)(1.743329,0.312629)(2.043360,0.278448)(2.395027,0.245809)(2.807216,0.215389)(3.290345,0.187616)(3.856620,0.162679)(4.520354,0.140570)(5.298317,0.121146)(6.210169,0.104195)(7.278954,0.089473)(8.531679,0.076731)(10.000000,0.065736)
        };
        \addplot[densely dashed,GREEN,line width=1.5pt] plot coordinates{ 
        (0.100000,0.063054)(0.117210,0.062689)(0.137382,0.062264)(0.161026,0.061773)(0.188739,0.061205)(0.221222,0.060550)(0.259294,0.059796)(0.303920,0.058931)(0.356225,0.057943)(0.417532,0.056818)(0.489390,0.055544)(0.573615,0.054108)(0.672336,0.052500)(0.788046,0.050710)(0.923671,0.048735)(1.082637,0.046573)(1.268961,0.044231)(1.487352,0.041719)(1.743329,0.039058)(2.043360,0.036274)(2.395027,0.033402)(2.807216,0.030482)(3.290345,0.027559)(3.856620,0.024680)(4.520354,0.021891)(5.298317,0.019235)(6.210169,0.016746)(7.278954,0.014453)(8.531679,0.012374)(10.000000,0.010517)
        };
        \end{axis}
  \end{tikzpicture}
  \caption{{ $ p/n = 4 $ }}
  \label{subfig:E3}
  \end{subfigure}
  ~
  
\caption{ { 
Theoretical interpolation errors for $ f(t) =\tanh(t) $ (\textbf{\BLUE blue}) from \Cref{theo:high_interpolation} versus that of linear regression (\textbf{\GREEN green}) from \Cref{prop:high_interpolation_error_RR}, as a function of SNR, for different dimension ratio $p/n$, synthetic data drawn from the Gaussian signal-plus-noise model as in \Cref{def:signal_plus_noise} with $\w_K = \w_Q = \bmu_{\rm base} \sim \mathcal{N}(\zo,\one_p/p)$, $\bmu \propto \bmu_{\rm base}$, and $\gamma = 1$.
} }
\label{fig:capacity_versus_SNR}
\end{figure}

\Cref{fig:capacity_versus_SNR} compare the theoretical interpolation errors of nonlinear Attention (as characterized in \Cref{theo:high_interpolation}) with those of linear linear regression (from \Cref{prop:high_interpolation_error_RR}) on synthetic Gaussian signal-plus-noise data. 
We observe that, while linear regression generally achieves lower interpolation error than nonlinear Attention in the under-determined $p > n$ regime, this advantage is reversed in the over-determined setting with $p/n < 1$.
In such cases, nonlinear Attention yields \emph{lowers} error, for structured inputs and Attention weights aligned to the data signal.
Furthermore, compared to linear regression, the interpolation error of nonlinear Attention exhibits remarkably less sensitivity to the dimension ratio $p/n$, especially when the Attention weights are well aligned with the underlying signal in the input data.

\begin{figure}[htb]
\centering
\input{fig_capacity_of_SNR_n_bigger_than_p}
\caption{ { 
Theoretical interpolation errors for $ f(t) =\tanh(t) $ (\textbf{\BLUE blue}) versus $ f(t) = \max(-5,\min(5,t)) $ (\textbf{\PURPLE purple}) and that of linear regression (\textbf{\GREEN green}) in the over-determined regime, as a function of SNR, for different dimension ratio $p/n$ and regularization parameter $\gamma$, synthetic data drawn from the Gaussian signal-plus-noise model in \Cref{def:signal_plus_noise} with $\w_K = \w_Q = \bmu_{\rm base} \sim \mathcal{N}(\zo,\I_p/p)$, $\bmu \propto \bmu_{\rm base}$.
} }
\label{fig_capacity_of_SNR_n_bigger_than_p}
\end{figure}

\Cref{fig_capacity_of_SNR_n_bigger_than_p} further illustrates the impact of the Attention nonlinearity, the dimension ratio $p/n$, and the regularization parameter $\gamma$ on the interpolation errors of nonlinear/linear Attention and linear linear regression.
Reading the subfigures from left to right, we observe that the difference in interpolation error between different Attention (i.e., $\tanh$ nonlinear or truncated linear) and linear regression vanishes either as the regularization strength $\gamma$ decreases \emph{or} as the SNR increases.
Moreover, the advantage of nonlinear Attention over linear regression—in terms of reduced interpolation error—critically depends on both the dimension ratio $p/n$ (as already confirmed in \Cref{fig:capacity_versus_SNR}) \emph{and} the choice of regularization $\gamma$, see for example \Cref{subfig:E_2} versus \Cref{subfig:E_6}.
Reading the subfigures from top to bottom, we further observe that in the over-determined $p/n < 1$ regime, the interpolation error of nonlinear Attention is considerably less sensitive to the changes in the dimension ratio $p/n$ compared to linear regression.


\begin{figure}[htb]
\centering
\input{fig_real_weight_capacity_versus_gamma}
\caption{ { Theoretical interpolation errors of Softmax (\textcolor{cyan}{\textbf{cyan}}) and entry-wise tanh (\textbf{\BLUE blue}), truncated exponential ($ f(t) = \min(5, \exp(t)) $ in \textbf{\RED red}) Attention. 
Theoretical predictions under \Cref{ass:low-rank-weights} in solid lines and key/query weights using pretrained Attention weights in \underline{dotted} lines.
\textbf{\Cref{subfig:real_E1}}, \textbf{\Cref{subfig:real_E2}}, and \textbf{\Cref{subfig:real_E3}}: theoretical predictions obtained by assuming $ \w_K = \w_Q = \bmu = \zo$;  
\textbf{\Cref{subfig:real_E4}}, \textbf{\Cref{subfig:real_E5}}, and \textbf{\Cref{subfig:real_E6}}: theoretical predictions obtained by assuming $ \w_K = \w_Q = \bmu_{\rm base} \sim \NN(\zo, \I_p/p) $, $\bmu \propto \bmu_{\rm base}$ and $\gamma = 1$.} }
\label{real_weight_capacity_versus_gamma}
\end{figure}

\Cref{real_weight_capacity_versus_gamma} compares the interpolation error curves of nonlinear Attention using weights extracted from a pretrained GPT-2 model against our theoretical predictions from \Cref{theo:high_interpolation}, across varying regularization strengths, SNR levels, and activation nonlinearities. 
This numerical experiment serves to empirically validate the full-plus-low-rank decomposition of Attention weights posited in \Cref{ass:low-rank-weights}.

To extract the Attention weights $\W_Q$ and $\W_K$, we use the first Attention head from the 1st, 7th, and 12th Transformer layers of a pretrained GPT-2 model (accessed via HuggingFace).
Specifically, we extract the first and second $m$-sized column blocks from the projection matrix \texttt{model.transformer.h[l].attn.c\_attn.weight} (of shape $m \times 3m$ with $m=768$) as query and key weight matrices.
The weights for a single head are then obtained by selecting the first $ m_{\rm head} = m / n_{\text{heads}} = 64 $ columns from each matrix, consistent with the model's $ n_{\text{heads}} = 12 $-head configuration.

As shown in \Cref{real_weight_capacity_versus_gamma}, the empirical interpolation error curves obtained from pretrained Attention weights closely match the theoretical trends predicted by \Cref{theo:high_interpolation}, as a function of both regularization strength $\gamma$ and SNR. 
In particular, we observe that
\begin{enumerate}
   \item in the absence of input data signal ($\bmu = \zo$), pretrained Attention weights yield slightly lower errors than theory; and 
   \item in the presence of signal, pretrained Attentions perform marginally worse than theory from (manually) aligned weights.
\end{enumerate}
These discrepancies are modest in scale and consistent across both tanh and truncated exponential nonlinearities. We observe that Softmax Attention incurs higher interpolation error than entrywise exponential Attention, but only when meaningful input structure is present. This agreement in trends further supports that the simplified Attention model in Assumption 1 serves as a reasonable abstraction, even compared to real pretrained GPT-2 weights.


\end{document}